%% file: main.tex
\documentclass{amsart}

\input{preamble}

\begin{document}

\title{Bayesian learning for the stochastic shortest path problem}
\author{Chon Wai Ho$^1$, Sumeetpal S. Singh$^2$, Jiaqi Guo$^1$}
\email{cwh38@cam.ac.uk}
\email{sumeetpals@uow.edu.au}
\email{jg851@cam.ac.uk}

\address{$^1$Department of Engineering, University of Cambridge, UK}
\address{$^2$School of Mathematics and Physics, University
 of Wollongong, Wollongong, Australia}

\begin{abstract}
    \input{sections/abstract}
\end{abstract}

\maketitle

\input{sections/intro}

\input{sections/prelim}

\input{sections/learning}

\input{sections/theoretical_results}

\input{sections/numerical}

\input{sections/conclude}

\section*{Acknowledgements}
S.S.~Singh holds the Tibra Foundation professorial chair and gratefully acknowledges research funding as follows: ``This material is based upon work supported by the Air Force Office of Scientific Research under award number FA2386-23-1-4100''. C.W. Ho was supported by the UK Engineering and Physical Sciences Research Council (EPSRC) grant EP/T517847/1 for the University of Cambridge Doctoral Training Programme. J. Guo was supported by the China Scholarship Council for the PhD programme.

\bibliographystyle{apalike}
\bibliography{bib}

  \appendix
 \input{sections/appendix}

\end{document}

%% file: preamble.tex
\usepackage{graphicx}
\usepackage{microtype}
\usepackage{indentfirst}
\usepackage{csquotes}
\usepackage[dvipsnames]{xcolor}

\usepackage{amssymb,amsthm,amsmath}
\usepackage{amsfonts}
\usepackage{mathtools}
\usepackage{bbm}
\usepackage{thmtools, thm-restate}
\usepackage[mathscr]{euscript}
 \let\mathscr\relax
\usepackage[scr]{rsfso}

\usepackage{amsaddr}
\usepackage{enumitem}
\usepackage{booktabs}
\usepackage{caption} 
\usepackage{subcaption}
\usepackage{float}
\usepackage[plain,noend]{algorithm2e}
\usepackage{tikz}
\usetikzlibrary{arrows.meta, positioning}
\usepackage{accsupp}
\usepackage{needspace}

\usepackage{natbib}
\usepackage{bookmark}
\usepackage{hyperref,fancyhdr,etoolbox}
\hypersetup{ colorlinks=true, linkcolor=black, filecolor=black, urlcolor=black }
\usepackage[noabbrev]{cleveref}
\crefformat{equation}{#2#1#3}
\crefrangeformat{equation}{#3#1#4 to~#5#2#6}

\newtheorem{theorem}{Theorem}[]
\newtheorem{definition}{Definition}
\newtheorem{example}{Example}
\newtheorem{lemma}{Lemma}
\newtheorem{proposition}{Proposition}
\newtheorem{corollary}{Corollary}
\newtheorem{assumption}{Assumption}
\newtheorem{remark}{Remark}

\renewcommand{\cdot}[1][.5]{%
  \mathbin{\vcenter{\hbox{\scalebox{#1}{$\bullet$}}}}%
}

\DeclareMathOperator*{\argmax}{arg\,max}

\newcommand{\depth}{d}
\newcommand{\pabc}{\hat{p}_{\epsilon}}

%% file: sections/abstract.tex
 Sequential decision-making problems are often modelled as a Markov decision process (MDP). We focus on the stochastic shortest path (SSP) problem, which is an infinite-horizon undiscounted MDP with absorbing terminal states. We develop a Bayesian framework to learn the optimal decision strategy through interactions with the decision-making task. Specifically, we learn the optimal action-value function $Q^*$, but unlike many existing Bayesian approaches, we do not rely on unrealistic modelling assumptions and ad-hoc approximations. Our approach is to directly construct the posterior beliefs for $Q^*$ through  Bellman's optimality equations. For deterministic rewards, we characterise the posterior as a distribution with a manifold density. To facilitate simpler inference, we relax the likelihood so that a Lebesgue density exists. The flip side is to create unidentifiability issues. Specifically, the relaxed posterior can have
significant mass on improper decision rules, while the exact posterior will not. 
 We also calculate the exact posterior probabilities for optimal action selections for the tabular parametrisation of $Q^*$, a Gaussian likelihood relaxation and a Gaussian prior, which is useful in benchmarking studies. Numerical studies on variants of the Deep Sea benchmark verify our findings. We demonstrate that our framework faithfully quantifies uncertainty and, compared to other temporal-difference-based Bayesian methodologies, is more data efficient. We conclude with recommendations for future work.

%% file: sections/intro.tex
\section{Introduction}
In many sequential decision-making problems, the aim is to find and execute the best decisions based on some quantifiable objective, such as the rewards accrued over the series of decisions. However, not all the information necessary to find the best choices is available \textit{a priori}. Consequently, interaction with the problem's environment---by executing decisions and gathering more information---is needed. This should be done sparingly though, as frequent interactions may be computationally expensive or even impractical.

A suitable strategy for interacting with the environment 
is required. A purely {\it exploitative} strategy is one that executes the ``best guess" of the optimal decision based on the available information.  Alternatively, a strategy that incorporates \emph{exploration} is one that does not use this best guess all the time, often randomising the choice  in some way, as doing so could reveal new information that eventually leads to better decisions. Ideally, one should use a strategy  that is data-efficient, which strikes the right balance between exploration and exploitation. In particular, it should produce a sequence of decisions that rapidly converges to the optimal choices.

An example of a strategy that balances exploration and exploitation can be found in the context of multi-armed bandit (MAB) problems \citep{tstutorial,suttonbartobook}, which is a well-studied class of decision-making problems. Thompson sampling (TS)---also known as posterior sampling \citep{ts1933,tstutorial}---is a Bayesian strategy that makes decisions according to the posterior probability that the chosen action is optimal. Thompson sampling has been proven to be near-optimal compared to other exploration-exploitation strategies \citep{tsagrawal2013contextual,tsshidong2018contextualbayesregret}. 

Thompson sampling has been applied to a larger class of decision-making problems, namely Markov decision processes (MDPs) \citep{tsdeardenbayesq,tsstren,randomisedvaluefunction}. Unlike the MAB problem, in an MDP, the available decisions and rewards are determined by a time-varying internal state process that is Markovian \citep{putermanrl,suttonbartobook}; see Section \ref{sec:mdp} for more details. Our focus is on an infinite-horizon undiscounted MDP with an absorbing terminal state, which is a class of problems known specifically as \emph{stochastic shortest path} (SSP) problems \citep{bertseka1991,putermanrl}. The SSP formulation is more general than the usual discounted infinite-horizon MDP. The absence of a discount factor requires an absorbing terminal state for the objective to be well-defined. We focus on the finite state and action space setting, which can be a very challenging high-dimensional learning problem. 

In this work, to balance exploitation and exploration,  we adopt Thompson sampling. This raises two important challenges to be addressed. Firstly, how do we construct a Bayesian framework that meaningfully quantifies the uncertainty of an action being optimal? Secondly, how do we access the resulting posterior distribution? Both of these challenges are discussed further below as the modelling and inference challenges.

\textbf{Modelling.} Our aim is to find the optimal decision-making rule for an SSP problem, and formulating this as a Bayesian learning problem is not straightforward. The decision maker only observes the state and, upon choosing the action, the immediate reward, which constitutes the data set. In more ``standard" Bayesian inference problems, like the MAB problem, the quantity to be inferred is usually explicitly observed in noise. However, immediate rewards are not noisy observations of the long-term optimality of specific decision-making rules.

Our approach is to learn the optimal \emph{action-value function}, denoted $Q^*$, which is the expected cumulative rewards when following the state transition dynamics under the optimal policy. It characterises the optimality of actions and uniquely satisfies a set of simultaneous equations known as the Bellman optimality equations (BOEs). 

We introduce a likelihood function that constrains the parametric approximation of $Q^*$ to be consistent with the subset of the BOEs implied by the state-action pairs in the data set. We characterise the properties of the likelihood when the adopted parametrisation of $Q^*$ creates \textit{improper policies}, which generate recurrent non-goal states, and is difficult to avoid in SSP problems. For deterministic rewards, we relax the likelihood to ensure effective Monte Carlo sampling. This relaxation is treated in our methodology as a further layer of approximation, and we investigate, mathematically, its impact on learning the optimal decision-making rule.

In contrast to our approach, many Bayesian formulations (discussed in Section \ref{sec:relatedwork}) that learn $Q^*$ stem from the $Q$-learning algorithm, which is a stochastic approximation algorithm \citep{bertsekas2019} that incrementally updates the approximation of $Q^*$ using the BOEs \citep{qlearning1992}. Specifically, in these works, the chosen likelihood function is motivated by a stochastic approximation procedure. Additionally, some also rely on unjustified and/or implicit assumptions. Thus, the resulting Bayesian formulation is highly nuanced, lacks interpretability, and may not faithfully quantify the residual uncertainty after data assimilation. This could diminish the effectiveness of TS too---which we note in our numerical studies---as it relies on this posterior distribution to make the action choices.

\textbf{Inference.} With an appropriate Bayesian formulation in hand, under general assumptions, we give a formula for the posterior density of $Q^\ast$ for deterministic rewards; this is not a Lebesgue density but rather a ``manifold density'' with respect to the Hausdorff measure. Under further specific but mild assumptions on the SSP problem, we then show that our characterisation is valid for the {\it tabular} parametrisation of $Q^*$ and our prior choice. To avoid the challenge of sampling from the manifold density, we relax the likelihood so that a Lebesgue density exists. The flip side, as we show, is to create unidentifiability issues. Specifically, the relaxed posterior can have significant mass on improper decision rules, while the exact posterior will not. Furthermore, we find the exact posterior probabilities (up to Gaussian integrals) for selecting optimal actions for noisy Gaussian rewards, or deterministic rewards with a relaxed likelihood. Although this characterisation could serve as a useful benchmark, its computation does not scale well with the size of the state and action spaces, and so we also pursue Monte Carlo methods in the numerical studies. We present extensive numerical experiments on the Deep Sea benchmark problem \citep{randomisedvaluefunction}. We demonstrate that our framework does faithfully quantify uncertainty, verify our insights above, and highlight the exploration benefits of our posterior, particularly that it is more data-efficient than some competing Bayesian methods.

To our knowledge, these are all new insights, and results like ours have not been covered in the related literature by works that 
similarly focus on fundamental and methodological aspects of Bayesian reinforcement learning. For example, there are results on exploration strategies \citep{osband2013moreefficient,psrlinfhorizon,bayesoptiamlpolicy,psrl-ssp} and their learning rates \citep{whypsrlbetterthanoptimism,rlsvi_regret_russo,provablyefficientpossampling}; approximations for efficient inference \citep{bootstraposband,randomisedpriorwithensembleosband,randomisedvaluefunction,dropout,bdqn_azizz,tdsmc,qlearn_sgld}; and modelling assumptions for learning the action-value function for fixed policies \citep{gpsarsa_engel,kalmantd,bayesianbellmanoperator}.

The paper is structured as follows. Section \ref{sec:prelim} introduces the general SSP problem and gives sufficient conditions to guarantee the uniqueness of the solution to the BOEs. Our Bayesian framework to learn the optimal action-value function is presented in Section \ref{sec:bayesianlearning}. This is followed by a discussion of exploration via posterior sampling in Section \ref{sec:possamplingexploration}. Section \ref{sec:relatedwork} discusses in detail related works and other existing Bayesian approaches. Our theoretical results, including the posterior characterisation for noiseless rewards, derivation of posterior probabilities for tabular models, and mathematical analysis of the likelihood, are presented in Section \ref{sec:theory}. Experiments are presented in Section \ref{sec:experiment}. Finally, limitations, unresolved challenges, and future directions are discussed in Section \ref{sec:conclude}.

\subsection{Notations}

For a set $\mathcal{X}$, let $\mathscr{P}(\mathcal{X})$ denote the set of all probability distributions over $\mathcal{X}$. For any $x,x^\prime \in \mathcal{X}$, let $\delta_{x^\prime}(x)$ be the Dirac delta function centred at $x^\prime$. For a distribution $p \in \mathscr{P}(\mathcal{X})$ associated with a random variable $X$ taking values in $\mathcal{X}$,  and a real-valued function $f:\mathcal{X} \rightarrow \mathbb{R}$, let $\mathbb{E}_{X \sim p(\cdot)}[f(X)] = \int f(x) p(x) \mathrm{d}x$. If  $\mathcal{X}$ is discrete, use $\sum_{x \in \mathcal{X}}$ in place of $\int_{\mathcal{X}} \mathrm{d} x$ in this definition. The expected value is also denoted as $\mathbb{E}_p[f(X)]$, or $\mathbb{E}[f(X)]$ when unambiguous. Let $\mathcal{Y}$ be another set and $Y$ be a random variable taking values in $\mathcal{Y}$. The conditional probability density function (pdf), or probability mass function (pmf) if discrete, of $Y$ given $X=x$ is denoted by $p(\cdot|x) \in \mathscr{P}({\mathcal{Y}})$. Let $p(y|x)$ be its value for a specific $(x,y)$ pair. $\mathcal{N}(\cdot;\mu,\epsilon^2)$ denotes the univariate Gaussian density with mean $\mu \in \mathbb{R}$ and variance $\epsilon^2$. Similar notation is used for the multivariate Gaussian case.

For a real-valued function $x \mapsto f(x)$, its support is defined to be $\mathrm{supp}(f)=\{x \mid f(x)\neq 0\}$. If $\mathcal{X} \subseteq \mathbb{R}^n$ for some  $n>1$, the $i$-th component of $x \in \mathcal{X}$ is denoted by $x_i$. Similarly, for integers $i<j$, we denote the vector $(x_i,x_{i+1},\dots,x_j)^{\top}$ by $x_{i:j}$. Let $\mathbb{Z}_{\geq 0}$ denote the set of all non-negative integers. For a function $f:\mathbb{R}^m \rightarrow \mathbb{R}^n$, let $Df(x) \in \mathbb{R}^{n\times m}$ denote the matrix of partial derivatives at $x \in \mathbb{R}^m$. In particular, when $m > n$, denote $Jf(x)=\sqrt{\det{(Df(x)Df(x)^\top)}}$.

%% file: sections/prelim.tex
\section{Preliminaries}
\label{sec:prelim}
This section presents key definitions and results for MDPs and for the stochastic shortest path problem in particular.

\subsection{Introduction to MDPs}
\label{sec:mdp}
A discrete-time infinite-horizon MDP is denoted by 
$\mathcal{M}=(\mathcal{S}, \mathcal{A}, p^S, p^R, \rho)$ \citep{putermanrl} where $\mathcal{S}$ is the state space; $\mathcal{A}$ is the action space; $\mathcal{A}_s$ is the set of admissible actions for state $s \in \mathcal{S}$, and thus $\mathcal{A} = \cup_{s \in \mathcal{S}} \mathcal{A}_s$. For any admissible state-action pair $(s,a)$, specifically $a \in \mathcal{A}_s$, the distribution of the MDP's next state is $p^S(\cdot|s,a) \in \mathscr{P}(\mathcal{S})$. Similarly, the distribution of the rewards is $p^R(\cdot|s,a) \in \mathscr{P}(\mathbb{R})$. Finally, $\rho \in \mathscr{P}(\mathcal{S})$ is the initial state distribution.

Let $\Pi = \{\pi: \mathcal{S} \rightarrow \mathscr{P}(\mathcal{A}) \mid \text{supp}(\pi(\cdot|s)) \subseteq \mathcal{A}_s \forall s \in \mathcal{S}\}$ be the set of Markovian decision rules that make choices for the actions based solely on the present state. A {\it policy} is a collection of Markov decision rules, $(\pi_t)_{t \in \mathbb{Z}_{\geq 0}}$, and defines how actions are to be chosen at all times $t$. The policy is  {\it stationary} if it deploys the same  Markov decision rule $(\pi,\pi,\ldots)$ at all times. We denote the stationary policy by its decision rule $\pi$. Furthermore, the decision rule $\pi \in \Pi$ is deterministic if $\pi(\cdot|s)$ is supported on one action only for all $s \in \mathcal{S}$. In this case, the simplified notation $\pi(s)$ denotes the action $a$ to be applied, which lies in $\mathcal{A}_s$, for state $s$.

We now describe the dynamics of the MDP, assuming that the agent is using a stationary policy $\pi \in \Pi$. The MDP commences at time $t=0$ by sampling the MDP's initial state $s_0$ from $\rho$. At any time $t\geq0$, with  $s_t$ denoting the present state, the agent samples the stationary policy  $\pi(\cdot|s_t)$ to get the action $A_t$ to be applied. If $A_t=a_t$, then the reward $R_t$ is drawn from $p^R(\cdot|s_t,a_t)$, followed by the next state $S_{t+1}$ from $ p^S(\cdot|s_t,a_t)$. Continuing with policy $\pi$ until time $\tau$ gives rise to the sequence of random variables $(S_0,A_0,R_0,S_1,\dots,S_{\tau-1},A_{\tau-1},R_{\tau-1},S_\tau)$ with probability distribution
\begin{equation*}
p^\pi_{S_{0:\tau},A_{0:\tau-1},R_{0:r-1}}(s_{0:\tau},a_{0:\tau-1},r_{0:\tau-1}) = \rho(s_0)\prod_{t=0}^{\tau-1} [\pi(a_t|s_t)p^R(r_t|s_t,a_t)p^S(s_{t+1}|s_t,a_t)].
\end{equation*}
To simplify the notation, we drop the subscripts when it is unambiguous and denote this probability distribution by  $p^\pi(s_{0:\tau},a_{0:\tau-1},r_{0:\tau-1})$. The same simplified notation is used for its marginals and conditional probabilities, such as $p^\pi(s_{0:\tau},a_{0:\tau-1}) := p^\pi_{S_{0:\tau},A_{0:\tau-1}}(s_{0:\tau},a_{0:\tau-1})$. $\mathbb{E}^\pi$ will denote expectation over $p^\pi$. 

The goal is to find the optimal policy that maximises the expected cumulative reward function. The Markovian structure of the problem ensures that it is sufficient to consider Markov decision rules since history-dependent policies do not improve the optimality criterion  \cite[Sec. 5.5, 7.1]{putermanrl}. Furthermore, it is sufficient to focus only on Markov policies that are stationary \cite[Sec. 6.2, 7.1]{putermanrl}. For a stationary policy $\pi$, and discount factor $0\leq \gamma < 1$, let
\begin{equation*}
    Q^\pi(s,a) := \mathbb{E}^\pi\Bigg[\sum_{t=0}^{\infty} \gamma^t R_t \Bigg| S_0=s, A_0=a\Bigg],\quad s \in \mathcal{S}, a \in \mathcal{A}_s.
\end{equation*}
The function $Q^\pi:\mathcal{S} \otimes \mathcal{A} \rightarrow \mathbb{R}$, where $\mathcal{S} \otimes \mathcal{A} := \bigcup_{s \in \mathcal{S}} \{s\} \times \mathcal{A}_s$, is known as the {\it action-value} function of policy $\pi$.

Assume that $|\mathcal{A}|$ is finite. Let $(s,a) \mapsto Q^*(s,a)$ be the {\it optimal} action-value function,  $Q^*(s,a) := \sup_{\pi \in \Pi}Q^\pi(s,a)$, and define the operator $\mathcal{B}^*_q$ on $\{Q \mid Q:\mathcal{S} \otimes \mathcal{A} \rightarrow \mathbb{R}\}$ to be
\begin{equation}
    \mathcal{B}_q^*(Q)(s,a) := \mathbb{E} \Big[R_0+\gamma \max_{a_1 \in \mathcal{A}_{S_1}}Q(S_1, a_1) \Big|S_0=s, A_0=a\Big]. \label{eqn:bellmanoptim}
\end{equation} Then $Q^*$ satisfies the Bellman optimality equations (BOEs) $\mathcal{B}_q^*(Q^*) = Q^*$ \cite[Sec. 4.2]{bertsekas2019}. Furthermore, define the {\it value} functions $V^\pi$ and $V^*$ to be $V^\pi(s) := \mathbb{E}^{\pi}[Q^\pi(S_0,A_0)|S_0=s]$ and $V^*(s) = \max_{a \in \mathcal{A}_s}Q^*(s,a)$. A policy $\pi^* \in \Pi$ is optimal if its value function $V^\pi$ coincides with the optimal value function $V^\ast$, i.e., $V^{\pi^*} = V^*$ \cite[Sec. 5.4]{putermanrl}.

\subsection{Uniqueness of solutions to Bellman optimality equations}
In this work, we focus on the undiscounted infinite horizon criterion ($\gamma =1$) for the MDP. This is the more general case compared to the discounted setting $0 \leq \gamma<1$. Furthermore, we assume that $|\mathcal{S}|$ and $|\mathcal{A}|$ are finite. We now state sufficient conditions for the undiscounted infinite horizon criterion to have an optimal policy, which is also stationary: we denote its decision rule by $\pi^{\ast} \in \Pi$.

\begin{assumption}
    The discount factor $\gamma=1$ and there exists a unique, zero-reward and absorbing goal-state $s^g\in\mathcal{S}$. Specifically, absorbing implies  $\mathcal{A}_{s^g}=\{a^g\}$, $p^S(\cdot|s^g,a^g)=\delta_{s^g}(\cdot)$; and rewardless means $p^R(\cdot|s^g,a^g)=\delta_{0}(\cdot)$.\label{ass:boeunique1}
\end{assumption}
If there is more than one zero-reward and absorbing goal state, they can all be lumped together and treated as a single entity. An MDP that satisfies Assumption \ref{ass:boeunique1} is known as a stochastic shortest path model (SSP). A deterministic stationary policy $\mu \in \Pi$ is \textit{proper} if $\lim\limits_{t \rightarrow \infty} p^\mu(S_t=s^g|S_0=s_0) = 1$ for all $s_0 \in \mathcal{S}$, otherwise it is \textit{improper}.
\begin{assumption}
 A proper deterministic stationary policy exists. Furthermore, for any improper deterministic stationary policy $\mu \in \Pi$, there exists a state $s \in \mathcal{S}$ such that $V^\mu(s) = -\infty$.  \label{ass:boeunique2}
\end{assumption}

Results in this work are proved under Assumption \ref{ass:boeunique1} or both Assumption \ref{ass:boeunique1} and \ref{ass:boeunique2}. Under Assumptions \ref{ass:boeunique1} and \ref{ass:boeunique2}, we have the following characterisations for $V^*$ and $\pi^*$.

\begin{theorem}[\cite{bertseka1991}]
    Let Assumptions \ref{ass:boeunique1} and \ref{ass:boeunique2} hold. Let $\bar{r}(s,a):=\mathbb{E}[R_t|S_t=s,A_t=a]$, and let $\mathcal{B}_v^*$ be the {\it Bellman} operator on $\{V\mid V:\mathcal{S} \rightarrow \mathbb{R}\}$,
    \begin{equation*}
    \mathcal{B}_v^*(V)(s) := \max_{a \in \mathcal{A}_s} \bar{r}(s,a)+ \sum_{s^\prime \in \mathcal{S}} V(s^\prime) p^S(s^\prime|s,a),\quad s \in \mathcal{S}. 
    \end{equation*}Then, 
    \begin{enumerate}
        \item 
        $V^*$ is the unique fixed point of $\mathcal{B}_v^*$ in $\{V:\mathcal{S} \rightarrow \mathbb{R} \mid V(s^g) = 0\}$.

    \item There exists an optimal stationary policy $\pi^* \in \Pi$, which is deterministic and proper, and given b (with tie breaking if necessary)
    \begin{equation*}
        \pi^*(a|s) = \mathbbm{1}\bigg(a \in \argmax_{a^\prime \in \mathcal{A}_s} \bar{r}(s,a^\prime) + \sum_{s^\prime \in \mathcal{S}} V^*(s^\prime)p^S(s^\prime|s,a^\prime) \bigg) .
    \end{equation*}
    \end{enumerate}
\label{thmvunique}
\end{theorem}
The uniqueness result of Theorem \ref{thmvunique} can be translated to $Q^*$, and $Q^*$ similarly characterises $\pi^\ast$.
\begin{corollary}
     Under the same assumption as Theorem \ref{thmvunique}, $V^*$ is the unique fixed point of $\mathcal{B}_v^*$ in $\{V:\mathcal{S} \rightarrow \mathbb{R}\mid V(s^g)=0\}$ if and only if $Q^*$ is the unique fixed point of $\mathcal{B}^*_q$ in $\{Q:\mathcal{S} \otimes \mathcal{A} \rightarrow \mathbb{R}\mid Q(s^g,a^g)=0\}$. Furthermore, the optimal policy (with tiebreaker) is also given by $Q^*$, specifically, $\pi^*(a|s) = \mathbbm{1}(a \in \argmax_{a \in \mathcal{A}_s}Q^*(s,a)).$\label{cor:qvuniqueness}
\end{corollary}
\begin{proof}

 This can be proven by constructing an augmented MDP \citep[p. 186]{bertsekas2019} whose optimal value function coincides with the optimal action-value function of the original MDP $\mathcal{M}$. Then, showing that if $\mathcal{M}$ satisfies Assumptions \ref{ass:boeunique1} and \ref{ass:boeunique2}, so does the augmented MDP.
\end{proof}
Other sufficient conditions can be found, e.g., in \citet{bertseka1991,putermanrl,bertsekas2019,ssppolyhedral}. Under  Assumptions \ref{ass:boeunique1} and \ref{ass:boeunique2}, $Q^\ast$ can be found using a Picard-type iterative scheme, which is known as {\it value iteration}. Applications of the SSP problem include search problems \citep{ssp_eagle}. The SSP formulation, and some related convergence results for computing $Q^\ast$, has been extended to the partially observed setting where the state is observed indirectly \citep{ssp_patek, ssp_appl_singh}.

%% file: sections/learning.tex
\section{Bayesian modelling of the optimal action-value function} 
\label{sec:bayesianlearning}

To turn the problem of learning $Q^*$ into a Bayesian problem, we introduce a parametric approximation $Q_\theta:\mathcal{S} \otimes \mathcal{A} \rightarrow \mathbb{R}$, $\theta \in \Theta \subseteq \mathbb{R}^{d_\Theta}$,  a prior distribution  $p^\Theta$ on $\Theta$, and then find the posterior distribution of $\theta$ given the observed interactions with the environment. The difficulty is that $Q^*$ is not observed directly. However, as shown below, it can be inferred through the BOEs. With the posterior distribution of $\theta$ in hand, we then proceed in Section \ref{sec:hausdorff} to analyse its properties for the challenging situation where the rewards are noiseless.

The data set is a sequence of state-action-reward triplets, $(s_t,a_t,r_t)$, $t=0,\ldots,\tau$, where a smaller $t$ indicates older data. The data could be from a single episode, which means ending in the absorbing state, or terminating the collection before absorption. Or even multiple episodes, which means restarting the data collection from a new starting state. We only observe samples by interacting with the MDP.

To emphasise the dependence of the reward on the state-action pair, we write $R(s,a)$ to mean a sample from $p^R(\cdot|s,a)$, and $r(s,a)$ for its realisation. For brevity, we write the expected reward (introduced earlier in Theorem \ref{thmvunique}) with subscripts instead, $\bar{r}_{s,a} := \mathbb{E}[R_t|S_t=s,A_t=a]$, which is time-independent by the model's homogeneity. The scale parameter of the reward's additive noise term is to be inferred; that is
\begin{equation}
    R(s,a) = \bar{r}_{s,a} + \eta(\psi)W,\quad (s,a)\in\mathcal{S}\otimes\mathcal{A},
\end{equation} where $W \sim p^W(\cdot|s,a)$ is a zero-mean noise with known distribution $p^W$, and $\eta:\Psi \rightarrow \mathbb{R}$ is the scaling function parametrised by $\psi \in \Psi \subseteq \mathbb{R}^{d_\Psi}$. 

For $(s,a) \in \mathcal{S} \otimes \mathcal{A}$, we define  $g_{s,a}: \Theta \rightarrow \mathbb{R}$  to be
\begin{equation}\label{eqn:def_of_g}
g_{s,a}(\theta) := Q_\theta(s,a) - \mathbb{E}\Big[\max_{a^\prime \in \mathcal{A}_{S_{1}}}Q_\theta(S_{1}, a^\prime)\Big| S_0=s, A_0=a\Big].
\end{equation}
In the rest of the paper, we assume that the expectation in $g_{s,a}(\theta)$ can be evaluated or approximated whenever $(s,a)$ is encountered in the data; see Section \ref{sec:relatedwork} for further discussion. We may express the likelihood function $p$ of  $Q_\theta$ using the conditional density of $R_t=r_t$ given $S_t=s_t$, $A_t=a_t$ and the BOEs, which imposes $\bar{r}_{s_t,a_t}=g_{s_t,a_t}(\theta)$, as
\begin{equation}
p(r_t|\theta,\psi,s_t,a_t) := \eta(\psi)^{-1}p^W(\eta(\psi)^{-1}(r_t - g_{s_t,a_t}(\theta))|s_t,a_t). \label{eqn:llhpstarnoisy}
\end{equation}
Then, the overall likelihood for the data set $(s_{0:\tau}$, $a_{0:\tau}$, $r_{0:\tau})$ is

\begin{equation}
    p(r_{0:\tau}|\theta,\psi,s_{0:\tau},a_{0:\tau}) := \prod_{t=0}^\tau p(r_t|\theta,\psi,s_t,a_t),
    \label{eqn:llhpstarnoisyfull}
\end{equation}
which has a product-form by the conditional independence of the $R_t$'s given $S_t$'s and $A_t$'s. The posterior distribution is 
\begin{equation}
p^{\Theta,\Psi}(\theta,\psi|r_{0:\tau},s_{0:\tau},a_{0:\tau}) 
\propto p^\Theta(\theta) p^\Psi(\psi) \prod_{t=0}^\tau p(r_t|\theta,\psi,s_t,a_t),
\label{eqn:pstarnoisyfull}
\end{equation}
with $p^\Psi$ denoting the prior distribution on $\Psi$. Let $\mathcal{D}_{\tau} =\{(s_t,a_t,r_t)\}_{t=0}^{\tau}$ denote the data in Equation \ref{eqn:pstarnoisyfull}. Henceforth, $\mathcal{D}_{\tau}$ will be used instead of $r_{0:\tau},s_{0:\tau},a_{0:\tau}$ to indicate conditioning on this data: the posterior distribution of the $\theta$-component of the full distribution in Equation \ref{eqn:pstarnoisyfull} is thus denoted by  $p^\Theta(\cdot|\mathcal{D}_\tau)$.

\subsection{Posterior sampling for exploration}\label{sec:possamplingexploration}

The posterior in Equation \ref{eqn:pstarnoisyfull}, which is also a posterior over the optimal action-value function, can be used for exploration in MDPs \cite{tsstren}. Specifically, given a parametrisation $Q_\theta$ ($\theta\in\Theta$) of $Q^*$, we can define the posterior distribution over the set of admissible optimal deterministic, stationary, Markovian policies as
\begin{equation}\mathbb{P}(\mu \text{ is an optimal}| \mathcal{D}_\tau)
:=p^\Theta \bigg(
\bigg\{
\theta\in \Theta  \,\bigg|\,  \mu(s) \in \argmax_{a \in \mathcal{A}_s} Q_\theta(s,a), \; s \in \mathcal{S}
\bigg\} 
\bigg| \mathcal{D}_\tau 
\bigg),
\label{eq:ts_mdp}
\end{equation}
assuming that $\argmax_{a \in \mathcal{A}_s} Q_\theta(s,a)$ is $p^\Theta(\cdot|\mathcal{D}_\tau)$-almost-surely unique. Uniqueness prevents `double counting' of $\theta$, so that this \emph{posterior probability} over $\mu$ sums to one, and a tie breaking assignment can be used when it does not hold.

At time $\tau +1$, the strategy is to play a policy $\mu$ according to this posterior probability that has the largest expected cumulative reward. Evaluating Equation \ref{eq:ts_mdp} can be computationally expensive, if not intractable; see Sections \ref{sec:hausdorff} and \ref{sec:tabhausdorff} for theoretical characterisations of the posterior over $\theta$. In practice, sampling a policy from the posterior for deployment is equivalent to sampling  $\theta \sim p^\Theta(\cdot|\mathcal{D}_{\tau})$ and then acting according to its greedy policy, specifically $\mu(s) \in \argmax_{a \in \mathcal{A}_s} Q_\theta(s,a)$.
This sampling method is known as  \textit{posterior sampling for reinforcement learning} \citep{tsstren,osband2013moreefficient}, which is an extension of Thompson sampling (TS) \citep{ts1933} for multi-armed bandit problems \citep{tstutorial}, and is utilised as part of our numerical study in Section \ref{sec:experiment}.

\subsection{Related works}\label{sec:relatedwork}

We now discuss how our work is distinct from numerous other related contributions. We commence with a comparison of the Bayesian problem formulation itself. While our choice of likelihood follows naturally from the problem exposition in Section \ref{sec:bayesianlearning}, many works that adopt a Bayesian perspective to learn $Q^*$ are built upon classical non-Bayesian algorithms like $Q$-learning \citep{qlearning1992,tsdeardenbayesq}, which is an alternative to our formulation. This leads to different likelihood choices. As we discuss below, in some cases, the likelihood choice is itself questionable and leads to time-inconsistent posterior definitions. The advantages of using a Bayesian learning approach over an optimisation counterpart---such as posterior driven policy exploration in Equation \ref{eq:ts_mdp}---are thereby diminished.

Modern $Q$-learning employing deep neural networks \citep{dqn_humanlevel,qlearningtheory, td_optimperspective} to approximate $Q^*$ can be viewed as stochastic gradient descent step for the following sequence of mean-squared temporal difference (TD) error minimisation objectives,
$$ L^{\mathrm{TD}}(\theta;\theta^t)  := \mathbb{E}_{\substack{S,A \sim d(\cdot) \\ S^\prime \sim p^S(\cdot|S,A)}} \bigg[\bigg(R(S,A) + \max_{a^\prime \in \mathcal{A}_{S^\prime}} Q_{\theta^t} (S^\prime,a^\prime) - Q_\theta(S,A)\bigg)^2\bigg],$$
where $Q_{\theta^t}$ is the present approximation to $Q^*$, and $d(\cdot)$ is a user-defined state-action sampling distribution; see also related pre-neural-network works, for example \citet{residualgradient_baird1995,fittedqiter}. Let $\theta^{t+1}$ be the minimising argument, or an approximately minimising one, of $L^{\mathrm{TD}}(\theta;\theta^t)$. In addition to the absence of an overall optimisation objective, in practice, the $Q$-learning update is performed with various additional ``tricks'' and safeguards for better stability: by breaking the temporal correlations via data-subsampling from a data buffer (also known as experience replay); and/or freezing $\theta_t$ in the optimisation criterion for multiple optimisation steps \citep{neuralfittedqiter, dqn_humanlevel, doubledqn}.

To learn $Q^*$ in a Bayesian way, a common approach is to recast the sequence of minimisation problems, together with the associated stability tricks, as a sequence of  Bayesian regression problems with the following sequence of  Gaussian likelihoods: $L(\theta;\theta^t) = \prod_{i=0}^t \mathcal{N}(r_i+\max_{a^\prime \in \mathcal{A}_{s_{i+1}}}Q_{\theta^t}(s_{i+1},a^\prime);Q_{\theta}(s_i,a_i),\epsilon^2)$. Here  $\epsilon^2$ is the variance and  $\theta^t$ is a point estimate of the parameter (for the true $Q^*$), which is chosen from the posterior of the previous Bayesian regression problem; for example, see \citet{dropout,bootstraposband,randomisedpriorwithensembleosband,randomisedvaluefunction, tdsmc}. In addition to the validity of replacing $\theta$ with $\theta^t$ being unclear, there are other concerns. A major one is the randomness of $\theta^t$, which is not accounted for in the posterior distribution. Also, $R_i + \max_{a^\prime \in \mathcal{A}_{S_{i+1}}} Q_{\theta^t} (S_{i+1},a^\prime)$ is a non-linear transformation of the reward $R_i\sim p^R(\cdot|s_t,a_t)$ and the state $S_{i+1}\sim p^S(\cdot \vert s_i,a_i)$ random variables. Thus, its distribution is not Gaussian as assumed by the temporal difference likelihood above. The combination of these factors results in an uninterpretable posterior, therefore its uncertainty quantification as well.

 An alternative to the temporal difference criterion in the optimisation literature is the mean-squared error of the Bellman residual (BR), 
$$ L^{\mathrm{BR}}(\theta) := \mathbb{E}_{S,A \sim d(\cdot)}\bigg[\bigg(R(S,A)+ \mathbb{E}_{S^\prime \sim p^S(\cdot|S,A)}\bigg[\max_{a^\prime \in \mathcal{A}_{S^\prime}} Q_\theta(S^\prime,a^\prime)\bigg]-Q_\theta(S,A)\bigg)^2\bigg].$$
Unlike the $L^{\mathrm{TD}}$, this objective's global minimum is the MDP's optimal action-value function $Q^*$ when the latter is the unique solution of the BOEs \citep{residualgradient_baird1995,lstdlinear,  policyevaluationsurvey}. Analogous to TD-based methods, Bellman residuals can be incorporated within the likelihood function in a Bayesian framework. In fact, Bellman residuals can be found within our likelihood for $Q^*$ as defined in Equation \ref{eqn:llhpstarnoisy}. Unlike TD-based approaches, using a BR-based likelihood does not change the definition of the posterior by using point estimates.

While we focus on learning the optimal action-value function $Q^\ast$, many BR-based works have focused on learning the posterior distribution of the action-value function $Q^\pi$ for a specific policy $\pi$; for example, \citet{gprl-rasmussen, gpsarsa_engel, kalmantd, singh_whitney_chopin_acm, bayesianbellmanoperator, sellandsingh}. The Bellman equations that characterise $Q^\pi$ have no maximisation inside the expectation, which makes the problem more amenable to approximations. For example, \citet{gprl-rasmussen} learn a Gaussian process approximation to $Q^\pi$: combining a Gaussian process prior with a Gaussian model for the dynamics $p^S$ results in a tractable Gaussian process solution to the BOEs for $Q^\pi$; also see related studies \citep{gpsarsa_engel, kalmantd}. Techniques for sampling the posterior distribution of the chosen problem formulation include conjugate tractable approaches \citep{tsdeardenbayesq,gprl-rasmussen,gpsarsa_engel,psrl-ssp}; variational inference methods \citep{dropout,bdqn_azizz}; Kalman filtering \citep{kalmantd}; ensemble-based methods \citep{bootstraposband,randomisedpriorwithensembleosband,randomisedvaluefunction,  tdsmc}; MCMC \citep{sellandsingh,qlearn_sgld}. Finally, we note that if the inner expectation within the Bellman residual (and therefore our likelihood) is not tractable, which is a common problem shared by all such BR-based methods, it is possible to estimate the inner expectation using Monte Carlo.

With a posterior definition and sampling technique in hand, many works have focused on characterising the {\it learning rate} of the resulting Thompson sampling (or posterior sampling) algorithm. The performance of posterior sampling with a TD-based likelihood is studied in \citet{rlsvi_regret_russo,provablyefficientpossampling}. The performance for posterior sampling is also studied in \citet{psrl_modelbased_regret_agrawal,whypsrlbetterthanoptimism, psrl-ssp}, but 
these works adopt a different definition of a posterior than ours. They maintain a tractable posterior over the transition dynamics (instead of $Q^*$) for finite state and action spaces and solve for the optimal policy using an iterative scheme like dynamic programming.

%% file: sections/theoretical_results.tex
\section{Theoretical results}
\label{sec:theory}

We now present theoretical results, beginning with a characterisation of the posterior for noiseless rewards---deterministic rewards are a pillar of the MDP framework, with many potential applications \citep{putermanrl,bertsekas2012}. We show that the posterior of the optimal action-value function lies in a low-dimensional manifold of the original (high-dimensional) space (Theorems \ref{thm:hausdorff_maintext} and \ref{thm:hausdorff_maintext_tabular}). As sampling a manifold posterior is challenging, we relax the likelihood so that the resulting \textit{approximate} posterior admits a Lebesgue density on the original space, which is more suitable for standard Monte Carlo sampling (see the numerical work in Section \ref{sec:experiment}). Furthermore, when combined with a Gaussian prior, the posterior is a mixture of truncated-Gaussians (Proposition \ref{prop:psrl_new}). However, we show how the relaxed likelihood incurs unidentifiability (Lemma \ref{lem:loopthm}), which has negative implications for inference (Lemma \ref{lem:propernou} and Theorem \ref{thm:unboundedlh}). The section ends with some suggestions for mitigation.

\begin{remark}
The state-action components of the data $\mathcal{D}_{\tau} =\{(s_t,a_t,r_t)\}_{t=0}^{\tau}$ (see  Equation \ref{eqn:pstarnoisyfull}), $\{(s_t,a_t)\}_{t=0}^{\tau}$, may contain multiple visits to a state and the same action selected for that state more than once. Our results in this section will be framed in the context of unique state-action pairs, excluding the goal state. Let
$\mathcal{D}^{\mathcal{S},\mathcal{A}}_{\tau} =\{(s_{t_i},a_{t_i})\}_{i=1}^n$ such that $0=t_1<\ldots<t_n\leq\tau$;  
$(s_{t_j},a_{t_j})\neq (s_{t_i},a_{t_i})$, for $i<j$; and $(s_{t_j},a_{t_j}) \neq (s^g,a^g)$. The unique state-action pairs of $\mathcal{D}_{\tau}$ is the largest such set $\mathcal{D}^{\mathcal{S},\mathcal{A}}_{\tau}$ that can be found.
Repeated state-action pairs are unnecessary for deterministic rewards. Furthermore, the goal state is excluded from $\mathcal{D}^{\mathcal{S},\mathcal{A}}_{\tau}$ since, by Assumption \ref{ass:boeunique1}, $g_{s^g,a^g}(\theta)=0$ and we will select parametrisations such that $Q_\theta(s^g,a^g)=0$ for all $\theta \in \Theta$. For simplicity, we re-index the subsequence so that
$\mathcal{D}^{\mathcal{S},\mathcal{A}}_{\tau} =\{(s_{i},a_{i})\}_{i=1}^n$.
Finally, for noiseless rewards, there is no $\psi$-component to the inferred, only $\theta$ is to be learnt.

\label{rem:dataset}
\end{remark}

\subsection{Hausdorff characterisation of posterior for deterministic rewards and relaxation} \label{sec:hausdorff}
When the rewards are noiseless, each state-action-reward triplet imposes an equality constraint (in contrast with Equation \ref{eqn:llhpstarnoisy}) for $\theta$ via the BOEs. In this case, $\theta$'s region of uncertainty contracts from the prior belief to the subset of $\Theta$ where the corresponding set of observed BOEs is satisfied. In other words, the support of the posterior distribution of $\theta$ is a low-dimensional manifold of $\Theta$, and no longer admits a density with respect to the Lebesgue measure on  $\Theta$ as the prior did. Our first result in Theorem \ref{thm:hausdorff_maintext} characterises the posterior through its density, expressed with respect to the Hausdorff measure.

 For noiseless rewards, we cannot use the posterior characterisation in Equation \ref{eqn:pstarnoisyfull}. Instead, the posterior is defined rigorously using the notion of a {\it regular conditional probability }\citep{rcp_def_parthasarathy}. We then find the specific expression for this regular conditional probability, which is expressed as a density, when conditioning is restricted to the possible set of data we can observe, which is expressed as specific values of the rewards in Equation \ref{eqn:hausdorff} below. The expression for the density in Equation \ref{eqn:hausdorff} uses the co-area formula technique proposed in \citet{coarea-diaconis} for finding manifold posteriors.

The statement of Theorem \ref{thm:hausdorff_maintext} concerns the posterior of $\theta$ conditioned on the data in $\mathcal{D}_\tau$ defined in Remark \ref{rem:dataset}. Let the corresponding reward vector be $\bar{r}:=(\bar{r}_{s_1,a_1},\dots,\bar{r}_{s_n,a_n})^\top$. For a differentiable function $G:\mathbb{R}^{d_\Theta}\rightarrow \mathbb{R}^n$, let $DG(\theta) \in \mathbb{R}^{n \times d_\Theta}$ denote its Jacobian matrix evaluated at $\theta$. Finally, $\mathcal{H}^{d}$ is the $d$-dimensional Hausdorff measure.
\begin{theorem}\label{thm:hausdorff_maintext}
Let $\mathcal{M}$ be an MDP that satisfies Assumption \ref{ass:boeunique1}, and has deterministic rewards $\bar{r}_{s,a}$. Let $\Theta =\mathbb{R}^{d_\Theta}$, and let---the constraint function---$G_\tau:\Theta\rightarrow \mathbb{R}^n$ be $G_\tau(\theta)_i=g_{s_i,a_i}(\theta)$. 

Assume the preimage $\mathcal{O}^{\bar{r}}_\tau = \{\theta \in \Theta \mid G_\tau(\theta)=\bar{r}\} \subseteq \text{supp}(p^\Theta)$, and is not empty. Furthermore, 
assume $G_\tau$ is Lipschitz continuous, and assume $JG_\tau(\theta)=\sqrt{\det(DG_\tau(\theta) DG_\tau(\theta)^\top)} > 0$ wherever it is differentiable on $\mathcal{O}_\tau^{\bar{r}}$. Let $N^c=\{\theta \in \Theta \mid  DG_\tau(\theta) \text{ exists}\}$. Then, the posterior density of $\theta$ conditional on $G_\tau(\theta)=\bar{r}$ is the integrand below:
\begin{equation}\label{eqn:hausdorff}
     P^\prime(E|G_\tau(\theta)=\bar{r}) :=  \int_{E} \frac{1}{m(\bar{r})} \frac{p^\Theta(\theta) \mathbbm{1}(\theta \in \mathcal{O}_\tau^{\bar{r}}\cap N^c)}{ JG_\tau(\theta)} \mathcal{H}^{d_\Theta - n}(\mathrm{d}\theta),\quad E \subseteq \mathbb{R}^{d_\Theta}, 
\end{equation}
where $m(\bar{r}):=\int_{\mathcal{O}_\tau^{\bar{r}}\cap N^c} p^\Theta(\theta)(JG_\tau (\theta))^{-1} \mathcal{H}^{d_\Theta-n}(\mathrm{d}\theta)$,
provided $0<m(\bar{r}) < \infty$.
\end{theorem}
\begin{proof}
    See Appendix \ref{apd:hausdorffmaintexthmproof}.
\end{proof}

 Theorem \ref{thm:hausdorff_maintext} gives a general formula for the posterior under verifiable conditions. Thus, in order to use this result to characterise the SSP's problem posterior distribution, we verify the assumptions for our specific choices. In particular, we show that the {\it tabular} parametrisation satisfies the Lipschitz condition. For this parametrisation, we show that if the MDP also satisfies Assumption \ref{ass:boeunique2}, then $JG_\tau(\theta)>0$ wherever $G_\tau$ is differentiable on the pre-image $\mathcal{O}_\tau^{\bar{r}}$, for all possible data sets $\bar{r}$ that may arise. These findings culminate in Theorem \ref{thm:hausdorff_maintext_tabular}.

Thus, the \textit{posterior distribution} in Equation \ref{eqn:hausdorff}, which is essentially a regular conditional distribution given a sigma-algebra generated by $G_\tau$, possesses a density with respect to the  $\mathcal{H}^{d_\Theta-n}(\mathrm{d} \theta)$. Sampling the Hausdorff density in Equation \ref{eqn:hausdorff} directly, however, is challenging. For example, it is difficult to design an MCMC proposal kernel that proposes samples on the manifold \citep{manifoldlifting}. Thus, one implication of this result is that it provides the theoretical foundation for tractable approximations that we can utilise; rather than directly working on the Hausdorff density, we can relax the degenerate likelihood (by convolving it) to ``thicken'' the manifold posterior so that it has a usual Lebesgue density in $\mathbb{R}^{d_\Theta}$. The resulting approximate posterior, denoted $\hat{p}_\epsilon$, is
\begin{equation}\label{eqn:abcpos}
    \pabc(\theta|\mathcal{D}_\tau) \propto  p^\Theta(\theta) \prod_{(s,a) \in \mathcal{D}_\tau^{\mathcal{S},\mathcal{A}}} K_\epsilon(g_{s,a}(\theta),\bar{r}_{s,a}) =: p^\Theta(\theta) \hat{L}_\epsilon(\theta|\mathcal{D}_\tau).
\end{equation}
(Specific notation for the approximate likelihood, $\hat{L}_\epsilon$, is introduced for later use.) Common kernels include the uniform kernel, $K_\epsilon(x,y) = (2\epsilon)^{-1} \mathbbm{1}(|x-y|<\epsilon)$, and the Gaussian kernel, $K_\epsilon(x,y) = \mathcal{N}(y;x,\epsilon^2)$. This posterior is suitable for standard Monte Carlo sampling. For example, rejection sampling; see related work in approximate Bayesian computation \citep{abcwikinson,abcmarin}. However, rejection sampling is unlikely to do well when $\theta$ is high-dimensional, and we implement some MCMC alternatives in Section \ref{sec:experiment}.

With an appropriate kernel, as $\epsilon$ tends to zero, expectations (of functions from a suitable class) computed with the approximate posterior distribution converge to the same expectation computed with the target posterior; see, e.g., \citet{manifoldsampling_phdthesis}. When $Q^*$ does not lie within the parametric class, it can be shown that the posterior collapses to a maximiser of $\hat{L}_\epsilon(\theta|\mathcal{D}_\tau)$ for appropriate kernels as $\epsilon$ tends to zero under regularity conditions.

\subsection{Posterior characterisation for tabular models}\label{sec:tabhausdorff}

The focus of this section is on the following tabular parametrisation.

\begin{definition}
   $Q_\theta(s,a):\mathcal{S} \otimes \mathcal{A} \rightarrow \mathbb{R}$ has a tabular form if there exists a bijection  $\nu:\mathcal{S} \otimes \mathcal{A} \rightarrow \{1,\dots,d_{\Theta}+1\}$---called the index function---where $d_{\Theta}=|(\mathcal{S} \setminus \{s^g\}) \otimes \mathcal{A}|$ and  $\nu((s^g,a^g)) :=d_\Theta+1$, such that $Q_{\theta}(s,a) \equiv \theta_{\nu(s,a)}$ for all $s \in \mathcal{S}\setminus \{s^g\}$, $a \in \mathcal{A}_s$. In addition, with an abuse of notation, define $\theta_{d_\Theta+1}= Q_\theta(s^g,a^g) :=0$.
    \label{def:tabular}
\end{definition}
Here, only $\theta \in \mathbb{R}^{d_\Theta}$ is inferred from the data since, by Assumption \ref{ass:boeunique1}, $Q^\pi(s^g,a^g) =0$ for all stationary policies $\pi$.  Furthermore, the goal state is removed from the data set when computing the posterior.

The following result shows that the tabular parameterisation of $Q^*$ satisfies the Lipschitz condition, and when combined with Assumptions \ref{ass:boeunique1} and \ref{ass:boeunique2}, the positivity of $JG_\tau$, as required by Theorem \ref{thm:hausdorff_maintext}, is assured.
\begin{theorem}\label{thm:hausdorff_maintext_tabular} Let $\mathcal{M}$ be an MDP that satisfies Assumption \ref{ass:boeunique1}. Let $G_\tau(\theta)$ be defined as in Theorem \ref{thm:hausdorff_maintext}. If  $Q_\theta$ is tabular, as in Definition \ref{def:tabular}, then $G_\tau(\cdot)$ is Lipschitz and thus differentiable almost everywhere.\\
In addition to $Q_{\theta}$ being tabular, let the MDP $\mathcal{M}$ satisfy both Assumptions \ref{ass:boeunique1} and \ref{ass:boeunique2}. Then, for $\theta \in \mathcal{O}^{\bar{r}}_\tau $, $JG_\tau(\theta)>0$ when $G_\tau(\cdot)$ is differentiable at $\theta$.

\end{theorem}
\begin{proof}
    See Appendix \ref{apd:hausdorff_assumptions_proof}.
\end{proof}

\begin{remark}\label{rem:hausdorff_corr}
    Notice that the posterior density for $\theta$ conditioned on $G_\tau(\theta)=\bar{r}$ in Equation \ref{eqn:hausdorff} is defined through the differentiable points of $G_\tau(\cdot)$. Corollary \ref{cor:tabularhausdorff_m} in Appendix \ref{apd:hausdorff_assumptions_proof} completes the verification of this characterisation by showing the normalising constant $m(\bar{r}) >0$.
\end{remark}

It follows from Theorem \ref{thm:hausdorff_maintext_tabular} that for any data set of any length---and notably for any reward values $\bar{r}$---that originates from an MDP that satisfies Assumptions \ref{ass:boeunique1} and \ref{ass:boeunique2}, the posterior of the optimal action-value function is indeed given by Equation \ref{eqn:hausdorff}. Theorem \ref{thm:hausdorff_maintext_tabular} (and Theorem \ref{thm:hausdorff_maintext}) also describes how the posterior uncertainty contracts. In particular, when $Q_\theta$ is given by the tabular parametrisation, Theorem \ref{thm:hausdorff_maintext_tabular} proves that the Hausdorff dimension of the posterior decreases with each unique (or new) state-action pair observed, provided the data comes from an MDP satisfying Assumptions \ref{ass:boeunique1} and \ref{ass:boeunique2}. This holds even though the prior $p^\Theta$ may permit improper policies.

Once again, this Hausdorff density is difficult to target computationally, even for a tabular parametrisation. Therefore, in Proposition \ref{prop:psrl_new}  below, we find the density of the relaxed tabular posterior. This posterior can be expressed as a mixture of densities, where each mixture component is the posterior for $\theta$ conditioned on a particular set of action choices for some (data-dependent fixed) set of states being optimal. When the prior $p^\Theta$ is Gaussian, and $K_{\epsilon}$ is a Gaussian kernel, the approximate posterior, denoted $\hat{p}_\epsilon(\cdot|\mathcal{D}_\tau)$, is shown in Proposition \ref{prop:psrl_new} to be a mixture of truncated-Gaussians.  Furthermore, this posterior is tractable up to multivariate Gaussian integrals.

 The expression for the likelihood given $\mathcal{D}_\tau$ involves the expected value of {$Q_\theta$} computed over the next states in $\mathcal{S}^{\mathcal{D}_\tau} := \bigcup_{i=1}^n \text{supp}(p^S(\cdot|s_i,a_i))$  (see Equation \ref{eqn:def_of_g}), $\mathcal{S}^{\mathcal{D}_\tau} \subseteq \mathcal{S}$. For this reason, the posterior of $\theta$ can be expressed in terms of the chosen actions of deterministic stationary policies $l$ in  $\mathcal{S}^{\mathcal{D}_\tau}$.

\begin{proposition}
\label{prop:psrl_new}
Assume $Q_\theta$  is tabular, as in Definition \ref{def:tabular}, and the prior for $\theta$ is $\prod_{j=1}^{d_\Theta}\mathcal{N}(\theta_j ; 0, \sigma^2)$.

Given a mapping $s \in \mathcal{S}^{\mathcal{D}_{\tau}}\rightarrow \ell(s)\in \mathcal{A}_s$---the action choices in $\mathcal{S}^{\mathcal{D}_{\tau}}$ of some deterministic stationary policy---let 
$$ E^\ell := \{\theta \in \Theta \mid \theta_{\nu(s,\ell(s))}\geq \theta_{\nu(s,a)}, s\in \mathcal{S}^{\mathcal{D}_{\tau}}, a \in \mathcal{A}_s\}.$$ 
Let $\mathcal{L}^{\mathcal{D}_{\tau}}$ be the set of all such mappings $\ell$. Then,
the approximate posterior defined in Equation \ref{eqn:abcpos} is
\begin{equation}\label{eqn:gaussianintegralformula}
\hat{p}_\epsilon ( E |\mathcal{D}_{\tau}) \propto
\sum\limits_{\ell \in \mathcal{L}^{\mathcal{D}_{\tau}}} \mathcal{N}(\bar{r}_{1:n};0,(\Gamma^{\ell})^{-1}) \int_{E\cap E^\ell} \mathcal{N}(\theta;\mu^\ell_{\theta|r},\Sigma^\ell_{\theta|r})\mathrm{d}\theta, \quad E \subseteq \Theta,\end{equation}
where $\bar{r}_{i}:= \bar{r}_{s_i,a_i}$, $\Gamma^{\ell} := (\sigma^2 B^\ell {B^\ell}^\top + \epsilon^2 I_n)^{-1}$; $\mu^\ell_{\theta|r} := \sigma^2 {B^\ell}^\top\Gamma^{\ell} r_{1:n}$; $\Sigma^\ell_{\theta|r}:= \sigma^2 I_{d_{\Theta}} - \sigma^4 {B^\ell}^\top \Gamma^{\ell} B^\ell$; and $B^\ell \in \mathbb{R}^{n\times d_{\Theta}}$ with $B^\ell_{i,j} = \mathbbm{1}(j=\nu(s_i,a_i)) - \sum_{s^\prime \in \mathcal{S}^{\mathcal{D}_{\tau}}} p^S(s^\prime|s_i,a_i) \mathbbm{1}(j=\nu(s^\prime,\ell(s^\prime)))$, for $1\leq i \leq n$, $1\leq j \leq d_{\Theta}$.
\end{proposition}
\begin{proof}
See Appendix \ref{apd:proofpsrl}.
\end{proof}

\begin{remark}\label{rmk:psrl_extension}
Naturally, the truncated-Gaussian mixture result of Proposition \ref{prop:psrl_new} extends to the true posterior $p^\Theta(\cdot |\mathcal{D}_{\tau})$ in Equation \ref{eqn:llhpstarnoisyfull} for noisy rewards, with additive Gaussian noise with constant variance. This true posterior may be found by replacing $\hat{L}_\epsilon(\theta|\mathcal{D}_\tau)$ with $\prod_{t=0}^\tau\mathcal{N}(r_t;g_{s_t,a_t}(\theta),\eta^2)$, where $\eta^2$ denotes the variance, and keeping $\mathcal{S}^{\mathcal{D}_\tau}$ unchanged.
 \end{remark}

While this posterior can be used to construct the exploration policy (see Section \ref{sec:possamplingexploration}), this approach does not scale well with $|\mathcal{S}|$, $|\mathcal{A}|$, and $d_{\Theta}$ because the Gaussian integrals are on a potentially high-dimensional space $\Theta$, and the summation is over a very large subset $\mathcal{L}^{\mathcal{D}_\tau}$ of deterministic stationary policies. As discussed further in Section \ref{sec:possamplingexploration}, MCMC samples from the posterior in Equation \ref{eqn:abcpos} can be used to construct a greedy policy for exploration (Algorithm \ref{alg:onlinedegen1}), which avoids evaluating the optimality probabilities of every policy in $\mathcal{L}^{\mathcal{D}_\tau}$. Nevertheless, these exact probabilities, which may be useful for further theoretical investigations, will be used in numerical studies in Section \ref{sec:experiment} as a gold standard for benchmarking.

\subsection{Model unidentifiability when improper policies exist}\label{sec:unidentifiability_improper}
We now articulate some properties of the likelihood and posterior distribution of $\theta$ when either $K_\epsilon$ is a Gaussian kernel or the reward noise is Gaussian with fixed variance. The results are presented for a tabular parametrisation of $Q^*$ as defined in Definition \ref{def:tabular}.

The following deterministic MDP will help illustrate some key points in the discussion to follow.
\begin{example}
    Consider the 2-state deterministic MDP shown in Figure \ref{fig:2dmdp}, where $\mathcal{S}=\{s^1,s^2\}$, $\mathcal{A}_{s^1}=\{a^1,a^2\}$, and $s^2$ is the goal state with absorbing action $a^g$. At $s^1$, taking action $a^1$ leads back to $s^1$ with reward $r^1$, while taking action $a^2$ leads to $s^2$ and yields reward $r^2$. Let $r^1=r^2=-1$, which  implies $Q^*(s^1,a^1)=-2$, $Q^*(s^1,a^2)=-1$ and $Q^*(s^2,a^g)=0$. A tabular model for $Q^*$ therefore requires two scalar parameters, $\theta=(\theta_1,\theta_2)^\top$, where $\theta_1^\ast = Q^*(s^1,a^1)$ and  $\theta_2^\ast = Q^*(s^1,a^2)$. Let $\mathcal{D}_1 = \{(s^1,a^1,-1)\}$ denote the partial data set,  and $\mathcal{D}_2 = \{(s^1,a^1,-1),(s^1,a^2,-1)\}$ the {\it complete} data set---every possible state-action-reward triplet is in $\mathcal{D}_2 $.
    \label{ex:2dmdp}
\end{example}

\begin{figure}[ht!]
\captionsetup{justification=raggedright}  
\centering
\BeginAccSupp{ActualText={A diagram visualising the state transitions and the associated rewards for each state-action pair of the 2-state MDP example.}}
\begin{tikzpicture}[
    >={Stealth[round]}, 
    state/.style={draw, circle, minimum size=0.5cm,font=\small}, 
    edge label/.style={midway, font=\small}, 
]

\node[state] (s1) at (0, 0) {$s^1$};
\node[state] (s2) at (2, 0) {$s^2$};

\draw[->] (s1) to[bend left] node[edge label, above] {$a^2,r^2$} (s2);
\draw[->] (s1) .. controls (-1.5, 1) and (-1.5, -1) .. node[edge label, left] {$a^1,r^1$} (s1);
\draw[->] (s2) .. controls (3.5, 1) and (3.5, -1) .. node[edge label, right] {$a^g,0$} (s2);

\end{tikzpicture}
\EndAccSupp{}
    \caption{A 2-state MDP with goal-state $s^g=s^2$, and a recurrent non-goal state $s^1$.}
    \label{fig:2dmdp}
\end{figure}
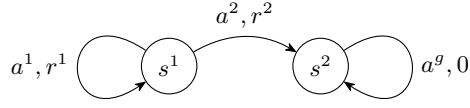
The model can be unidentifiable, under the likelihood, for an MDP that admits improper policies, even with a complete data set. Specifically, Lemma \ref{lem:loopthm} shows that for such MDPs there exists a subset of $\Theta$ which corresponds to improper policies, that is unbounded if $\Theta$ is unbounded, and the likelihood remains constant along half-lines in this subset.

A state $s^r \in \mathcal{S}$ is a recurrent non-goal state under a deterministic policy $\mu:\mathcal{S} \rightarrow \mathcal{A}$ if $s^r \neq s^g$ and $p^\mu(S_t=s^r \text{ for some } t \in \mathbb{Z}_{\geq 1}|S_0=s^r)=1$ . The transition matrix of improper policies creates recurrent non-goal states. Let $u \in [0,1]^{d_\Theta}$ be the vector whose elements $u_{\nu(s,a)}$ are equal to the maximum probability of ever reaching $s^r$ from all state-action pairs $(s,a)$, where the maximum is over non-stationary deterministic policies $\tilde{\mu}=( \mu_t )_{t\geq 0}$, that is,  
\begin{equation}\label{eqn:maxprob_u}
u_i=\max_{\tilde{\mu}} p^{\tilde{\mu}}(S_t=s^r \text{ for some } t \in \mathbb{Z}_{\geq 1}|(S_0,A_0) = \nu^{-1}(i)).
\end{equation}

\begin{lemma}
Let $\mathcal{M}$ be an MDP that satisfies Assumption \ref{ass:boeunique1}, and assume $Q_\theta$  is tabular as in Definition \ref{def:tabular}.
Furthermore, let $s^r$ denote a recurrent non-goal state of some improper deterministic policy, and let $\phi:\mathcal{S} \rightarrow \mathcal{A}$ be the maximising actions of $u$, defined in Equation \ref{eqn:maxprob_u}, that is $\phi(s) \in \argmax_{a \in \mathcal{A}_{s}} u_{\nu(s,a)}$ for any $s \in \mathcal{S}$. 

Then, for any $\theta \in \mathcal{O}^\phi  := \{\theta \in \Theta\mid \theta_{\nu(s,\phi(s))} \geq \theta_{\nu(s,a)},\,s \in \mathcal{S}, \,  a\in\mathcal{A}_s\}$, $$\hat{L}_\epsilon(\theta|\mathcal{D}_\tau) = \hat{L}_\epsilon(\theta+cu|\mathcal{D}_\tau),\quad \textrm{for all } \epsilon>0, c > 0 \; \textrm{and data set } \mathcal{D}_\tau,$$
where $\hat{L}_\epsilon$ is defined in Equation \ref{eqn:abcpos} with a Gaussian kernel $K_\epsilon$.
\label{lem:loopthm}
\end{lemma}

\begin{proof}
    See Appendix \ref{apd:nondecaylkhproof}.
\end{proof}
The invariance of $\hat{L}_\epsilon$ for all data sets emphasises that the unidentifiability of the approximate likelihood persists even when $\mathcal{D}_\tau$ includes a visit to every state-action pair.

\begin{remark}
Furthermore, similar to Remark \ref{rmk:psrl_extension}, the result exactly as stated also applies when rewards are observed with additive Gaussian noise and known variance. In which case, $\mathcal{D}_\tau$ should contain all the observed data, including repeated state-action pairs with their noisy rewards. An analogous extension will also hold in Lemma \ref{lem:propernou} and Theorem \ref{thm:unboundedlh} to follow.
\end{remark}

The policy $\phi$, which was defined in Lemma \ref{lem:loopthm} using $u$ in Equation \ref{eqn:maxprob_u}, is improper. In contrast, for some deterministic stationary proper policy $\mu$, let $\mathcal{O}^\mu := \{\theta \in \Theta\mid \theta_{\nu(s,\mu(s))} \geq \theta_{\nu(s,a)}, s \in \mathcal{S},\, a \in \mathcal{A}_s\}$. Define a \textit{complete} data set to be one that contains at least one instance of each (non-goal) state-action pair, along with their expected reward: for example, $\mathcal{D}^{\mathrm{full}}=\{(s,a,\bar{r}_{s,a}) \mid s \in \mathcal{S} \setminus \{s^g\}, a \in \mathcal{A}_s\}$. The following result verifies that the likelihood of a deterministic stationary proper policy is not constant along half-lines in the sense of Lemma \ref{lem:loopthm}.

\Needspace{7\baselineskip} 
\begin{lemma}\label{lem:propernou} 
Let $\mathcal{M}$ be an MDP that satisfies Assumption \ref{ass:boeunique1}, and assume $Q_\theta$  is tabular as in Definition \ref{def:tabular}.

\begin{sloppypar}
Let $\theta \in \Theta$ be such that all of its corresponding deterministic policies, $\mu(s)\in\argmax_{a\in \mathcal{A}_s} Q_\theta(s,a)$, $s\in \mathcal{S}$, are proper. Let $\epsilon>0$ and $ u \in \mathbb{R}^{d_\Theta}$. Then,
\[\hat{L}_\epsilon(\theta+cu| \mathcal{D}^{\mathrm{full}}) = \hat{L}_\epsilon(\theta| \mathcal{D}^{\mathrm{full}}), \mathrm{~for~all~} c > 0 
\quad \implies \quad u=0,\]
where $\hat{L}_\epsilon$ is defined with the Gaussian kernel $K_\epsilon$.
\end{sloppypar}
\end{lemma}

\begin{proof}
    See Appendix \ref{apd:lem:propernou}.
\end{proof}

Finally, the result below frames the impact of the invariance of the likelihood for improper policies by contrasting the total mass the likelihood assigns to the sets $\mathcal{O}^\phi$ and $\mathcal{O}^\mu$.

\begin{theorem}\label{thm:unboundedlh}
Under the assumptions of Lemma \ref{lem:loopthm},
$$(i)\,\,\int_{\mathcal{O}^\phi} \hat{L}_\epsilon(\theta|\mathcal{D}_\tau) \mathrm{d} \theta = \infty  \quad \text{for any data set $\mathcal{D}_\tau$};\qquad (ii)\,\,\int_{\mathcal{O}^\mu} \hat{L}_\epsilon(\theta| \mathcal{D}^{\mathrm{full}}) \mathrm{d} \theta < \infty.$$
\end{theorem} 

\begin{proof}
    See Appendix \ref{apd:unboundedlhproof}.
\end{proof}

Theorem \ref{thm:unboundedlh} implies that the set of improper policies can dominate the posterior, even with a complete data set, if the prior assigns mass to $\mathcal{O}^\phi$. In contrast, for a proper policy, the total mass of the likelihood is finite when every state-action-reward triplet has been observed. The posterior dominance of improper policies is more pronounced when a prior with a large variance is chosen. Such priors may be favoured when there is a lack of knowledge of the scale of $Q^*$, or to introduce optimism to facilitate exploration \citep{randomisedvaluefunction,provablyefficientpossampling}. As a consequence of Lemma \ref{lem:loopthm}, any prior mass assigned to $\mathcal{O}^\phi$ causes the posterior density to become elongated along half-lines in $\mathcal{O}^\phi$, and more so when the prior's variance is larger.
Example \ref{ex:2dmdp} is a specific example of an MDP covered by Lemma \ref{lem:loopthm} and Theorem \ref{thm:unboundedlh}. State $s^0$ of Example \ref{ex:2dmdp} is the non-goal recurrent state in Lemma \ref{lem:loopthm}, vector $u$ is $(1,0)^\top$ and $\mathcal{O}^\phi=\{\theta \in \mathbb{R}^2\mid\theta_1 \geq \theta_2\}$ for improper policy $\phi(s^1)=a^1$. While in Theorem \ref{thm:unboundedlh}, $\mathcal{O}^\mu = \{\theta \in \mathbb{R}^2 \mid \theta_1<\theta_2 \}$ for proper policy $\mu(s^1)=a^2$. As shown in Figure \ref{fig:2dcontour} (left), the posterior density remains stretched towards large positive $\theta_1$, instead of contracting uniformly towards $Q^*$. Note that since $Q_\theta(s^2,a^g):=0$, likelihood translation invariance does not hold in $\mathcal{O}^\mu$.

\begin{figure}[t!]
    \centering
    \BeginAccSupp{ActualText={Two subplots comparing data settings: the left shows the elongation of the probability density along one dimension under a complete data set setting, while the right shows the posterior density collapsed into a diagonal line in the incomplete data set setting with a small tolerance.}}
    \includegraphics[width=1\linewidth]{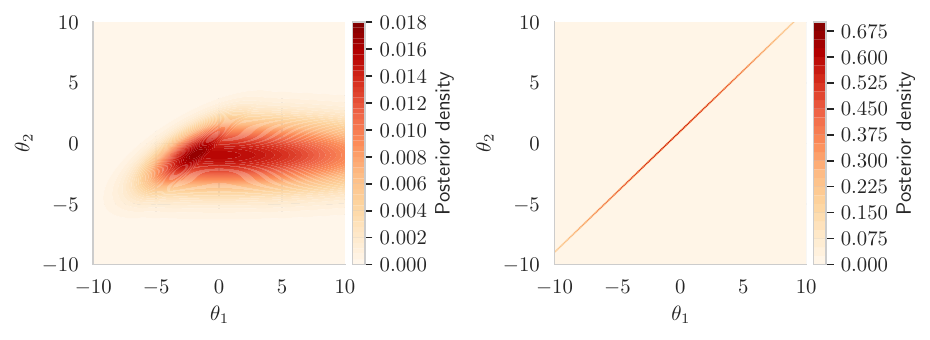}
    \EndAccSupp{}
    \caption{Contour plots of the posterior of Example \ref{ex:2dmdp}. Left: For the complete data set $\mathcal{D}_2$, zero-mean Gaussian prior with $\sigma=10$, and tolerance $\epsilon=2$. Right: For the partial data set $\mathcal{D}_1$, a zero-mean Gaussian prior with $\sigma = 10$, and tolerance $\epsilon = 0.01$.}
    \label{fig:2dcontour}
\end{figure}

To summarise, a Gaussian prior for the tabular parametrisation of $Q^*$ fails to encode the prior knowledge that the policy which corresponds to $Q^*$ must be proper. The Gaussian likelihood, which is a popular choice (see e.g., \citet{dropout,randomisedvaluefunction}), assigns infinite total mass to improper policies (Theorem \ref{thm:unboundedlh}). Although the resulting posterior is well-defined, the posterior's mass is shifted away from $Q^*$ towards parameters associated with recurrent non-goal states (Figure \ref{fig:2dcontour}). This may result in a significant bias in the estimation of the optimal policy probabilities in Equation \ref{eq:ts_mdp} under the true model. However, as more state-action-reward data are gathered for the noisy rewards case, or as $\epsilon \rightarrow 0$ for noiseless rewards, the likelihood of $\theta$ values that do not exactly satisfy BOEs becomes small, and this pathological effect diminishes; see Figure \ref{fig:illustrative_plot}.

\begin{figure}[t!]
    \centering
    \BeginAccSupp{ActualText={Line plots showing the contraction of the marginal posterior of each dimension towards its true value as the tolerance decreases in the complete data set setting.}}
    \includegraphics[width=1\linewidth]{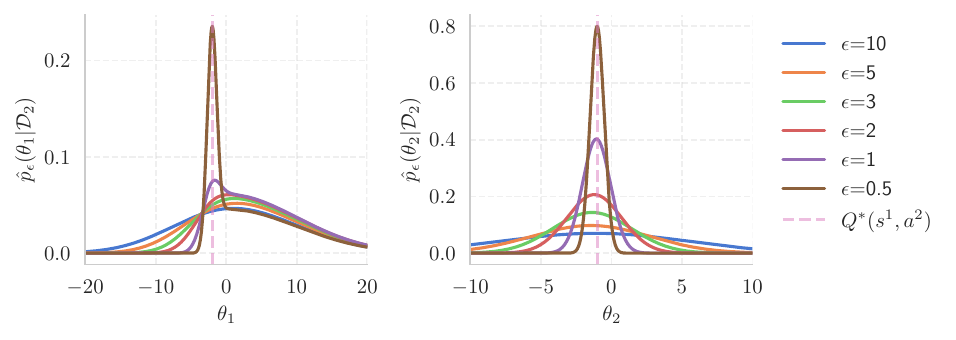}
    \EndAccSupp{}
    \caption{The posterior marginals for  $\theta_1$ and $\theta_2$ of Example \ref{ex:2dmdp} with the complete data set $\mathcal{D}_2$, zero-mean Gaussian prior with $\sigma=10$, and various tolerances.}
    \label{fig:illustrative_plot}
\end{figure}

A potential mitigation is to leverage knowledge of $p^S$ and $p^R$ to elicit a prior that excludes, or penalises, improper policies. The development of such priors is left for future work.

For other choices of the prior (beyond a Gaussian), the infinite integral of the likelihood for improper policies in Theorem \ref{thm:unboundedlh} may not result in a well-defined posterior. As a mitigation, a discounted reward criterion can be used. A discounted MDP, with discount factor $\gamma \in [0,1)$, can be reformulated as an equivalent SSP by introducing a {\it fictitious} zero-reward goal state so that each state transition has a $1-\gamma$ probability of moving to this goal state; conversely, with probability $\gamma$, it moves according to the original state transition dynamics \citep[Sec. 4.3]{bertsekas2019}. Consequently, discounting yields an approximation of the original problem formulation in which all policies are proper. In this new SSP, $\int_{\mathcal{O}^\mu} \hat{L}_\epsilon(\theta| \mathcal{D}^{\mathrm{full}}) \mathrm{d} \theta < \infty$ for all deterministic policies $\mu$---thus compromising exactness to lessen posterior `elongation' towards improper policies.

%% file: sections/numerical.tex
\section{Numerical illustrations}\label{sec:experiment}

 In this section, we demonstrate our methodology on a benchmark MDP problem and compare our results with competing Bayesian paradigms.

Algorithm \ref{alg:onlinedegen1} gives the posterior sampling exploration methodology (c.f. Section \ref{sec:possamplingexploration}) that is used in all the numerical studies. At the beginning of each episode, a policy is sampled using either method A or B (which are probabilistically equivalent when implemented exactly). Method A computes the posterior probability mass function of the policies, as defined in Equation \ref{eq:ts_mdp}, and samples from it directly. Method B samples a $\theta$ from its posterior distribution and constructs the corresponding greedy policy. Once sampled, the policy is used to interact with the MDP until either the goal state $s^g$ is reached or the episode {\it termination criterion} is met, which is only satisfied by improper policies. Before a new episode is initiated, the newly collected data is appended to update the posterior.\footnote{The code for these experiments is available at \url{https://github.com/hcw1026/ABRLExact}.}

\begin{algorithm}
\caption{Pseudocode for learning $Q^\ast$ for $I$ episodes.}
Set $t \gets 0$; set $\mathcal{D}_{-1} \gets \emptyset$.\\
    \For{$i \gets 0$ \KwTo $I-1$}{
    Set $t_i \gets t-1$.\\
    (A) Sample $\mu_i \sim \mathbb{P}(\cdot \text{ is optimal} |\mathcal{D}_{t_i})$ in Equation \ref{eq:ts_mdp}; or\\
    (B) Sample $\theta \sim p^\Theta(\cdot|\mathcal{D}_{t_i})$ and set $\mu_i(s) \in \argmax_{a \in \mathcal{A}_s} Q_\theta(s,a)$ for all $s \in \mathcal{S}$. \\
    Sample (initial state) $s_t \sim \rho$.\\
    \While{$s_t \neq s^g$ and termination criterion is not reached}{
    Select $a_t \gets \mu_i(s_t)$.\\
    Observe $r_t \sim p^R(\cdot|s_t,a_t)$, $s_{t+1} \sim p^S(\cdot|s_t,a_t)$.\\
    Append new data $\mathcal{D}_t \gets \mathcal{D}_{t-1} \cup \{(s_t,a_t,r_t)\}$.\\
    Increment $t \gets t + 1$.\\
    }}
\label{alg:onlinedegen1}
\end{algorithm}

All MDP examples in this section have deterministic state transition dynamics; an example with stochastic transition dynamics is provided in Appendix \ref{apd:sec:numerical}. With deterministic state transitions, an episode can be terminated when the same state is encountered again, which indicates an improper policy is deployed; other, more sophisticated termination heuristics are available, e.g., see \citet{ucrl2regret,psrlinfhorizon,bayesoptiamlpolicy,psrl-ssp}.

Following \citet{tsstren}, Algorithm \ref{alg:onlinedegen1} deploys the sampled policy for multiple time steps, which is a widely adopted strategy. This is because it has been shown to be an effective learning strategy, both theoretically and empirically \citep{osband2013moreefficient,bootstraposband,randomisedpriorwithensembleosband,randomisedvaluefunction,psrlinfhorizon}, and also more effective than refreshing the policy at every time step. Also, as an added benefit, it is potentially computationally more efficient to delay updating the posterior until more informative data points have been collected. Such data points, for example, are those generated by episodes that end in the goal state.

\begin{figure}[t!]
    \centering
    \BeginAccSupp{ActualText={Diagrams depicting the state transition dynamics of the three variants of deep sea environment.}}
    \begin{subfigure}[b]{0.32\textwidth}
        \centering
        \begin{minipage}[c][4cm][c]{\linewidth}
            \centering
            \includegraphics[width=\linewidth, height=4cm, keepaspectratio]{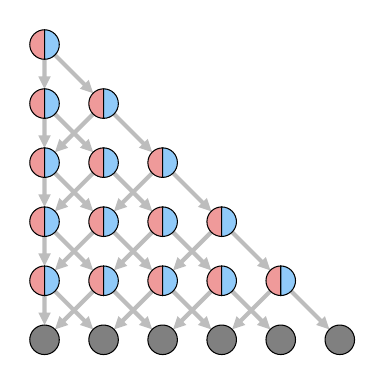}
        \end{minipage}
        \caption{\texttt{Deep sea}}
        \label{fig:demo_deepsea}
    \end{subfigure}
    \hfill
    \begin{subfigure}[b]{0.32\textwidth}
        \centering
        \begin{minipage}[c][4cm][c]{\linewidth}
            \centering
            \includegraphics[width=\linewidth, height=4cm, keepaspectratio]{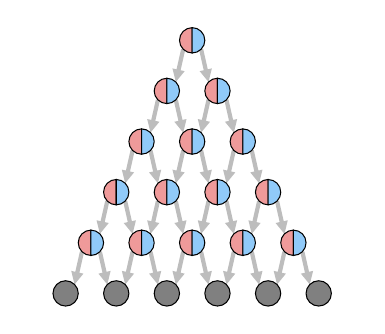}
        \end{minipage}
        \caption{\texttt{Deep sea pyramid}}
        \label{fig:demo_deepseapyramid}
    \end{subfigure}
    \hfill
    \begin{subfigure}[b]{0.32\textwidth}
        \centering
        \begin{minipage}[c][4cm][c]{\linewidth}
            \centering
            \includegraphics[width=\linewidth, height=4cm, keepaspectratio]{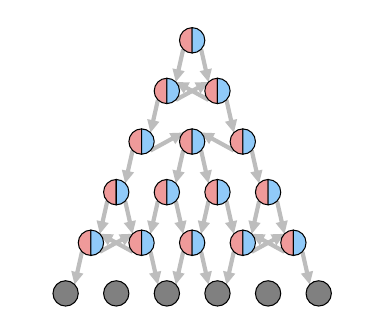}
        \end{minipage}
        \caption{\texttt{Deep sea swirl}}
        \label{fig:demo_deepseaswirl}
    \end{subfigure}
    \EndAccSupp{}
    \caption{ The {\it deep sea} problem (left) \citep{randomisedvaluefunction, bsuite} with depth 5, and two variants (middle and right). Each node is a state, which is also displayed as a pie chart to show the marginal posterior probability that the left action (pink) and the right action (blue) are optimal. Arrows show the state transition stemming from the actions. The set of grey states on the bottom row represents the goal states, which collectively form one single goal state. An episode terminates when the goal state is reached.}
    \label{fig:deepsea_init_demo}
\end{figure}

All the numerical studies use the relaxation of the deterministic rewards via the Bellman residual (BR) likelihood $\hat{L}_{\epsilon}$ in Equation \ref{eqn:abcpos}, with a Gaussian kernel and various tolerance values $\epsilon$. We use a tabular $Q_\theta$ (Definition \ref{def:tabular}) and the prior distribution $p^\Theta$ is taken as $\mathcal{N}(0,\sigma^2 I)$, with $\sigma=10$ unless otherwise specified. In this setting, policies can indeed be sampled via method A as discussed above by applying Proposition \ref{prop:psrl_new} to Equation \ref{eq:ts_mdp} to compute the probability mass function over the policies without Monte Carlo, denoted {\it Bayes-BR} in the text and just {\it BR} in the graphs; see Appendix \ref{apd:sec:numerical} for a more detailed explanation. To explore the impact of numerical approximations, we also compute the posterior over $\theta$ with Hamiltonian Monte Carlo \citep{hmcneal} and sample the policies using method B, referred to as {\it Bayes-BR-HMC} in the text and HMC in the graphs.

 We compare our posterior with various temporal difference (TD) alternatives listed below, which are also sampled using Method A.

\begin{enumerate}[label=(\alph*)]
    \item \textit{Bayesian temporal difference} (\textit{Bayes-TD}), denoted as TD in the graphs, targets the sequence of episode-specific posteriors
$$p^\Theta(\theta) \prod_{(s,a,\bar{r}_{s,a}) \in \mathcal{D}_{t_i}} \mathcal{N}\Big(\bar{r}_{s,a};\theta_{\nu(s,a)} - \sum_{s^\prime \in \mathcal{S}} p^S(s^\prime|s,a) \max_{a^\prime \in \mathcal{A}_{s^\prime}} \hat{\theta}^{i-1}_{\nu(s^\prime,a^\prime)}, \epsilon^2\Big),$$
where $p^\Theta(\theta)\equiv\mathcal{N}(\theta;0,10^2 I)$, and $\hat{\theta}^{i-1}$ is the expectation of the posterior obtained in episode $i-1$ conditioning on $\mathcal{D}_{t_i}$. This acts as a baseline model for simple Bayesian TD methods, such as \citet{dropout}.
    \item  \textit{Bayesian ensemble-based temporal difference} (\textit{Bayes-TD-En}), denoted as \text{TD-En} in the graphs, aims to more closely mimic ensemble-based methods, such as \citet{bootstraposband,tdsmc}. Specifically, the sequence of episode-specific posteriors is
    $$p^\Theta(\theta)\Bigg[ \frac{1}{N}\sum_{n=1}^N \prod_{(s,a,\bar{r}_{s,a}) \in \mathcal{D}_{t_i}} \mathcal{N}\Big(\bar{r}_{s,a};\theta_{\nu(s,a)} - \sum_{s^\prime \in \mathcal{S}} p^S(s^\prime|s,a) \max_{a^\prime \in \mathcal{A}_{s^\prime}} \hat{\theta}^{i-1, (n)}_{\nu(s^\prime,a^\prime)}, \epsilon^2\Big) \Bigg],$$
    where $\{\hat{\theta}^{i-1, (n)}\}_{n=1}^N$ is a set of samples from the (truncated-Gaussian mixture) posterior from episode $i-1$. 
    \item A variant of \textit{Bayesian temporal difference} targeting the sequence of posteriors
     $$p^\Theta(\theta) \prod_{(s,a,\bar{r}_{s,a}) \in \mathcal{D}_{t_i}} \mathcal{N}\Big(\bar{r}_{s,a};\theta_{\nu(s,a)} - \sum_{s^\prime \in \mathcal{S}} p^S(s^\prime|s,a)  m^{i-1}(s^\prime), \epsilon^2\Big),$$
    where $m^{i-1}(s^\prime)$ denotes the posterior expectation over $\Theta$ of $\max_{a^\prime \in \mathcal{A}_{s^\prime}} \theta_{\nu(s^\prime,a^\prime)}$ computed from episode $i-1$. This variant, called {\it Bayes-TD-Max} and denoted as TD-Max in the graphs, replaces the max term with its expectation instead of taking the expectation inside the max as in (a).
\end{enumerate}

(See Appendix \ref{apd:sec:numerical} for all technical details.)

We evaluate the algorithms on variants of the \texttt{deep sea} environment. \texttt{Deep sea} \citep{randomisedvaluefunction, bsuite}, with depth $\depth \geq 1$, is an MDP whose state space can be depicted as a grid of states arranged in a lower triangular layout with side lengths $\depth$, as illustrated in Figure \ref{fig:demo_deepsea}. Each state has two associated actions, one to move left and one to move right, and both moves also descend one row down. In the deterministic version, the left action moves to the adjacent cell on the left one row down---similarly for the right move---except for states on the edge where the move is straight down. A depth-$\depth$ \texttt{deep sea} instance requires executing a sequence of $\depth$ actions to reach the goal. Every left action yields a positive reward of $0.15/\depth$, while every right action incurs the negative of that amount. Also, at the penultimate row's rightmost state, taking the right action yields the usual reward of $-0.15/\depth$ and an additional reward of $1$. Thus, the optimal policy is to consistently move right until the goal state is reached, receiving a reward of $0.85$ in total.

\begin{figure}[t!]
    \centering
    \BeginAccSupp{ActualText={A sequence of diagrams showing the exploration-exploitation progress in an instance of depth-3 deep sea from episodes 0 to 9. The visualisation shows the posterior contracting around the optimal route and actions, with exploration completed at episode 8.}}
    \includegraphics[width=1\linewidth]{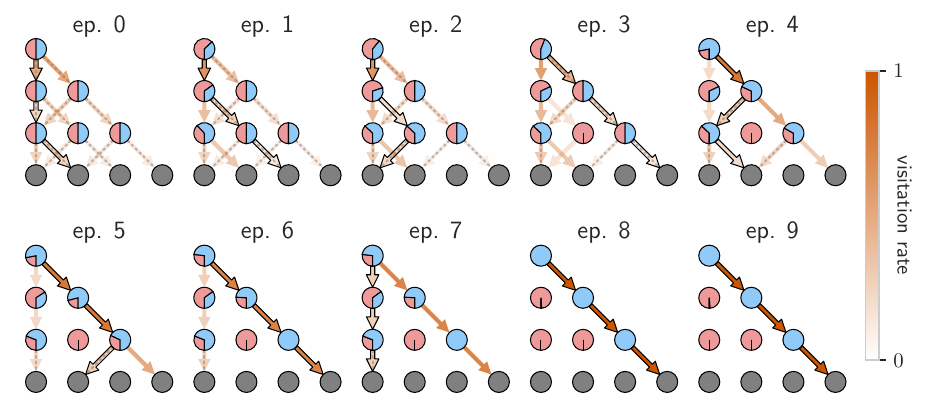}
    \EndAccSupp{}
    \caption{An illustration of Algorithm \ref{alg:onlinedegen1}'s output on depth-$3$ \texttt{deep sea} with \textit{Bayes-BR} with $\epsilon=0.01$. (See Figure \ref{fig:deepsea_init_demo} for an explanation of the node colours.) The state-action trajectory of the episode is shown with the sequence of arrows with borders. The dotted lines indicate the state-action pairs for which there is no data yet. For each displayed episode, $10000$ Monte Carlo samples from the episode's posterior are generated to illustrate the distribution over optimal policies. This is shown as orange arrows where the transparency encodes the estimated probabilities.}
    \label{fig:deepsea_demo_3}
\end{figure}

Figure \ref{fig:deepsea_demo_3} illustrates how the policy-selection probabilities evolve after each episode with our posterior and method A (denoted BR in the figure). In this instance, it took $8$ episodes before the action-selection probabilities led to consistently taking the right actions to travel along the diagonal edge and receive the maximum cumulative reward.

\begin{figure}[ht!]
    \centering
    \BeginAccSupp{ActualText={Line graphs over 100 episodes comparing BR and HMC for each of the tolerances 0.01, 0.1, 0.2, 0.5, 1 and 5 in the depth-5 deep sea environment. The plot shows the convergence of the posterior probability for the optimal policy as well as the corresponding decay of suboptimal policies.}}
    \includegraphics[width=1\linewidth]{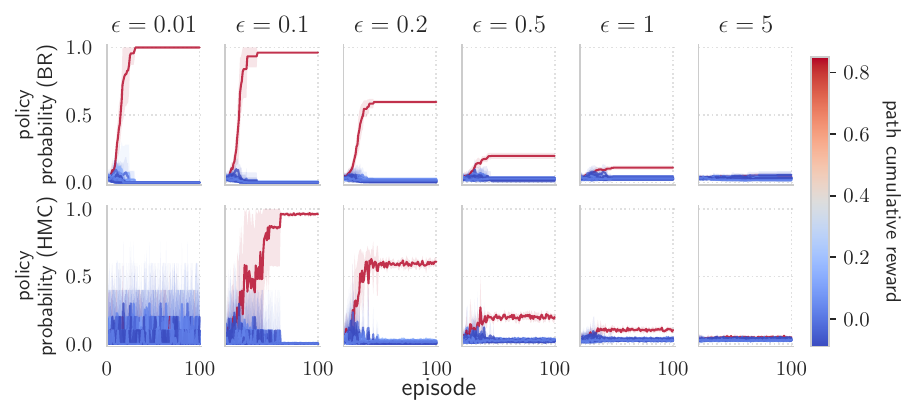}
    \EndAccSupp{}
    \caption{Comparing exact (BR) and Monte Carlo sampling (HMC) from our posterior for the depth-$5$ {\tt deep sea} environment. The evolution of the posterior over policies (Equation \ref{eq:ts_mdp}) with episode number is shown for different tolerance values $\epsilon$, after averaging over $10$ independent runs. The colour shows the cumulative reward achieved by the polices. Shaded areas show one standard deviation computed with the $10$ independent runs.}
    \label{fig:deepsea_policy_cvg}
\end{figure}
Figures \ref{fig:deepsea_policy_cvg} and \ref{fig:deepsea_cumreg_and_acc} reveal several important insights. Firstly, with \textit{Bayes-BR}, a sufficiently small tolerance $\epsilon$ leads to a rapid concentration of the posterior distribution over the optimal policies, thereby revealing them. Unsurprisingly, as the tolerance increases, in this case when $\epsilon \geq 0.1$, suboptimal policies continue to be selected when following the posterior sampling exploration strategy, even after the data set has become complete; see Figure \ref{fig:deepsea_cumreg_and_acc}.

Secondly, we sample our posterior with a tuned `modern' Monte Carlo sampler---$2000$ iterations of a tuned Hamiltonian MCMC algorithm are executed for each episode. Specifically, a fixed computational budget is assigned to each episode by fixing the factors that affect the cost. The initial $1000$ iterations are used to identify hyperparameters that yield an acceptance probability of approximately $0.7$. Furthermore, the step-size of the sampler employed in a given episode is passed down to the subsequent episode and used as an initial estimate. As shown, labelled HMC in the figures, sampling struggles to reliably identify the optimal policy when the tolerance $\epsilon$ is $0.01$, while its convergence remains comparatively slow at $\epsilon=0.1$.
Crucially, this coincides with the same region where the optimal policy has a probability close to $1$ of being selected after the data set has become complete. Furthermore, as illustrated in Figure \ref{fig:deepsea_hmc_acc}, the acceptance probabilities exhibit high variance for the two cases, $\epsilon=0.01$ and $\epsilon=0.1$. For these two cases, while keeping the same computational budget, we attempted to improve the tuning outcome by allowing larger changes in the HMC step-sizes between adjacent episodes, since the best step-size for an episode can differ greatly from that of the next. However, despite our efforts, the resulting plots remain similar; see Appendix \ref{apd:sec:numerical} for technical details. This demonstrates that the relaxation of the manifold  posterior is still challenging to sample using conventional MCMC methods. The development of more sophisticated, potentially hybrid, sampling methodologies is left for future work.
\begin{figure}[ht!]
    \centering
    \BeginAccSupp{ActualText={Cumulative regret curves over 100 episodes in a depth-5 deep sea environment, comparing BR against HMC, TD, TD-En and TD-Max for each of the tolerances 0.01, 0.1, 0.2, 0.5, 1 and 5.}}
    \includegraphics[width=1\linewidth]{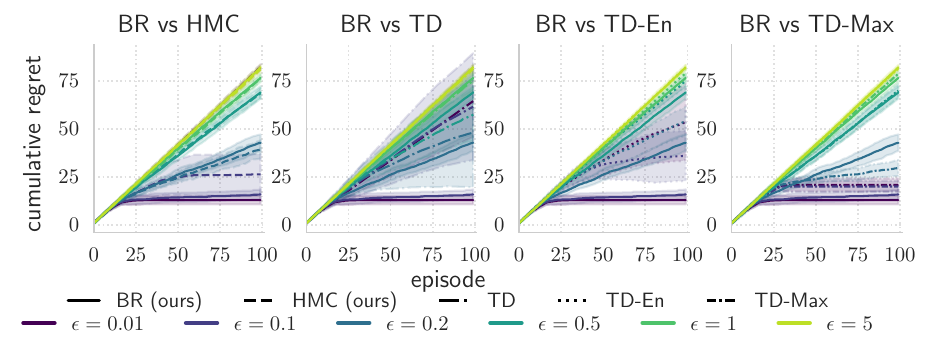}
    \EndAccSupp{}
    \caption{The progress of the regret as it accumulates with each new episode, for the depth-$5$ \texttt{deep sea} problem, for different tolerance values $\epsilon$. Results are averaged over $10$ independent runs. The regret of the deployed policy for an episode is the difference between the total (expected) reward achieved by an optimal policy ($0.85$) and the (empirical) total reward achieved by the deployed policy for that episode. Shaded areas indicate one standard deviation computed with the $10$ independent runs.}
    \label{fig:deepsea_cumreg}
\end{figure}

\begin{figure}[t!]
    \centering
    \BeginAccSupp{ActualText={Two subplots for the depth-5 deep sea environment over 100 episodes: the left tracks the exploration progress of BR, HMC, TD, TD-En and TD-Max for tolerances 0.01, 0.1, 0.2, 0.5, 1 and 5; the right plots the corresponding HMC acceptance probabilities for each tolerance level.}}
    \begin{subfigure}[b]{0.7\textwidth}
        \centering
        \includegraphics[width=\linewidth]{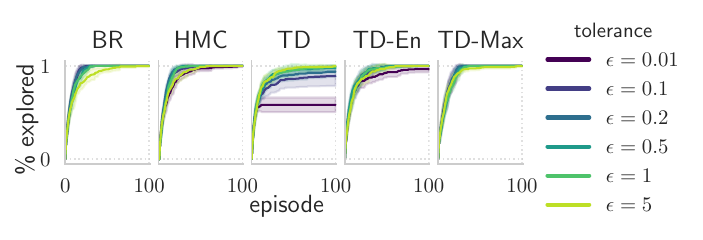}
        \caption{Percentage of the environment's state-action pairs visited at least once.}
        \label{fig:deepsea_expl_prog}
    \end{subfigure}
    \hfill
    \begin{subfigure}[b]{0.29\textwidth}
        \centering
        \includegraphics[width=\linewidth]{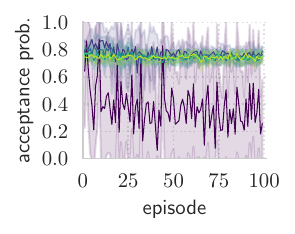}
        \caption{HMC acceptance probability}
        \label{fig:deepsea_hmc_acc}
    \end{subfigure}
    \EndAccSupp{}
    \caption{Results for the depth-$5$ \texttt{deep sea} environment. All results are averaged over $10$ independent runs. Shaded areas show one standard deviation computed using these runs.}
    \label{fig:deepsea_cumreg_and_acc}
\end{figure}

Thirdly, Figure \ref{fig:deepsea_cumreg} shows our posterior (labelled BR specifically) performing better than the \textit{Bayes-TD}-based methods when $\epsilon=0.01$ and $\epsilon=0.1$, while all fail for tolerances greater than or equal to $0.5$, due to their large scale parameter compared to the magnitude of the rewards. Additionally, although \textit{Bayes-TD-Max} has a lower cumulative regret curve at $\epsilon=0.02$, it indicates that \textit{Bayes-BR} and \textit{Bayes-TD}-based  posteriors, which are themselves different, lead to different learning outcomes as the tolerance shifts. We see from Figure \ref{fig:deepsea_expl_prog} that the well-performing variants have fully explored the environment much before episode $100$, indicating that any increase in the cumulative regret (defined in Figure \ref{fig:deepsea_cumreg}) curves near the end of the $100$ episodes cannot be attributed to incomplete data collection. While it is difficult to interpret the impact of using a point estimate $\hat{\theta}^i$ in the likelihood of the TD-based Bayesian methods, it is reassuring that our Bayesian formulation, which can be said to have been derived in a more justified manner, is performing reliably.

\begin{figure}[ht!]
    \centering
    \BeginAccSupp{ActualText={A sequence of diagrams showing the exploration-exploitation progress in an instance of depth-5 deep sea pyramid at episodes 0, 5, 10, 15, 21. The visualisation shows the posterior contracting evenly around each optimal route and actions, with exploration completed at episode 21.}}
    \includegraphics[width=1\linewidth]{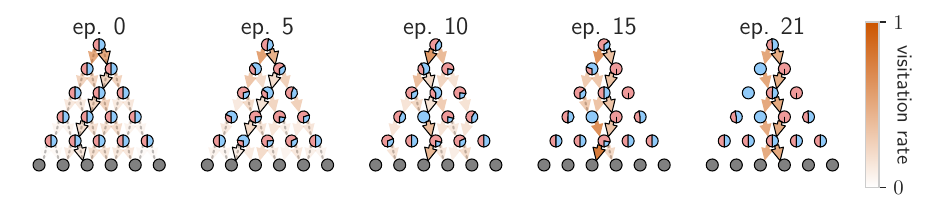}
    \EndAccSupp{}
    \caption{An illustration of Algorithm \ref{alg:onlinedegen1}'s output on depth-$5$ \texttt{deep sea pyramid} with \textit{Bayes-BR} with $\epsilon=0.01$. (See Figure \ref{fig:deepsea_init_demo} for an explanation of the node colours.)}
    \label{fig:deepseapyramid_demo_5}
\end{figure}

\begin{figure}[ht!]
    \centering
    \BeginAccSupp{ActualText={A sequence of diagrams showing the exploration-exploitation progress in an instance of depth-5 deep sea swirl at episodes 0, 10, 15, 20, 36. The visualisation shows the posterior contracting around the optimal route and actions, with exploration completed at episode 36.}}
    \includegraphics[width=1\linewidth]{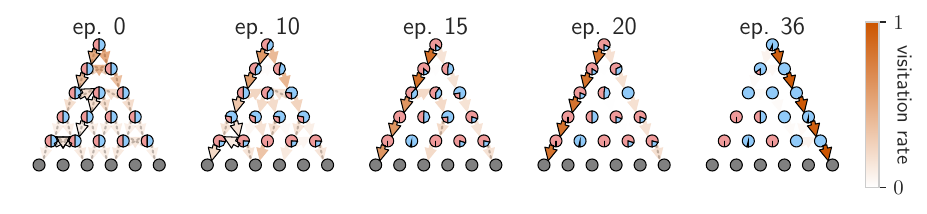}
    \EndAccSupp{}
    \caption{An illustration of Algorithm \ref{alg:onlinedegen1}'s output on depth-$5$ \texttt{deep sea swirl} with \textit{Bayes-BR} with $\epsilon=0.01$. (See Figure \ref{fig:deepsea_init_demo} for an explanation of the node colours.)}
    \label{fig:deepseaswirl_demo_5}
\end{figure}

Next, we construct a variant of the \texttt{deep sea} environment, termed \texttt{deep sea pyramid}. As shown in Figure \ref{fig:demo_deepseapyramid}, the main difference between these environments is their state connectivity. We selected the reward distribution $p^R$ for \texttt{deep sea pyramid}  so that there are multiple optimal policies to check if the methods are able to discover them. The policies that result in transitions that are confined to routes along the central vertical axis are optimal. Specifically, transitioning to and from any states on the central vertical axis yields a reward of $1/\depth$, whereas all other state-action pairs earn zero reward, meaning that the optimal policies receive a cumulative reward of $1$. As demonstrated in the depth-$5$ \texttt{deep sea pyramid} in Figure \ref{fig:deepseapyramid_demo_5}, all optimal policies are recovered (when $\epsilon=0.01$ and $\sigma=10$) and have an almost equal probability of being optimal after $21$ episodes. Thus, this demonstrates that our Bayesian formulation results in a posterior whose support eventually concentrates on the optimal policies with equally distributed mass.

Another variant of \texttt{deep sea}, which has improper policies, called \texttt{deep sea swirl}, is presented. This new environment's state connectivity, as illustrated in Figure \ref{fig:demo_deepseaswirl}, is similar in structure to \texttt{deep sea pyramid}, except that some down arrows have been redirected to connect to neighbouring nodes on the same row instead. By design, improper policies exist, and proper policies take $\depth$ or more actions to reach the goal. The reward structure in this example is the same as that of the \texttt{deep sea}, except that transitioning horizontally via the augmented actions earns a reward of $-0.45/\depth$ ($3$ times taking an ordinary right action). Therefore, the optimal policy remains the one that consistently takes the move-right actions. As illustrated in Figure \ref{fig:deepseaswirl_demo_5}, after $36$ episodes, the optimal class of policies is recovered, even though the posterior landscape is challenging (as described in Section \ref{sec:unidentifiability_improper}) due to the presence of improper policies.

%% file: sections/conclude.tex
\section{Discussion and conclusion}
\label{sec:conclude}

We presented a Bayesian framework to construct posterior beliefs for the optimal action-value function $Q^*$. Compared to other TD-based Bayesian approaches, our framework does not rely on unrealistic modelling assumptions or ad-hoc simplifications of the desired posterior distribution. The benefits include that posterior beliefs can be better relied upon for correctness, as well as the emergence of new insights. In particular, for deterministic rewards, we showed how the updated belief progressively constrains the prior to the manifold on which $Q^*$ lies, and we found this posterior density on the manifold. We then investigated the implications of using a Gaussian relaxation of the likelihood. For a tabular parametrisation of $Q^{\ast}$, we showed how the Gaussian likelihood is translation invariant for improper policies, even when the data set is complete. This leads to improper policies dominating the posterior belief. Thus, the variance of the Gaussian relaxation has to be carefully chosen to balance this ``loss'' against the benefits of simpler inference, since we are no longer computing a manifold posterior.

We confirmed these insights in numerical studies, and also compared the accuracy of our framework to TD-based Bayesian methodologies. Using our posterior to both actuate the MDP and guide further data gathering exhibits higher data-efficiency for low relaxation variance $\epsilon$. We also showed that small $\epsilon$ values are required to obtain a good approximation of the posterior beliefs over the optimal policies.

Our work opens interesting directions for follow-up studies, such as designing priors that exclude improper policies. Even with a relaxed likelihood, our results show that more efficient MCMC samplers are needed. We have not investigated how the inexact evaluation of the expectation over the state transition dynamics in the Gaussian likelihood relaxation further diminishes the accuracy of the posterior. Finally, the implementation ``tricks'' and modelling simplifications in various Bayesian reinforcement learning studies could be assessed against how much they distort the posterior towards improper policies.

%% file: sections/appendix.tex
\section{Proofs}

\subsection{Posterior distribution for deterministic rewards via Hausdorff measure}\label{apd:hausdorff_proof}

\subsubsection{Theorem \ref{thm:hausdorff_maintext}}\label{apd:hausdorffmaintexthmproof}

We begin by defining a regular conditional probability. Let the measurable spaces be $(X,\mathcal{B})$ and $(Y,\mathcal{C})$,
which are assumed to be \emph{separable standard Borel spaces.} In addition, let $\pi:X\rightarrow  Y$
be a measurable mapping.
\begin{definition}\label{def:rcp}\citep{rcp_def_parthasarathy}
Let $P$ be a probability measure on $(X,\mathcal{B})$. Let the
measure $C\in\mathcal{C}\rightarrow Q(C)$ on $(Y,\mathcal{C})$ be
$Q(C)=P(\pi^{-1}(C))$. A \emph{regular conditional probability given
$(\pi,P)$} is a mapping $y\in Y\rightarrow P_{y}(\cdot)$ such
that:
\end{definition}

\begin{enumerate}
\item For each $y\in Y$, $P_{y}$ is a probability measure on $(X,\mathcal{B})$.
\item There exists a ($Q$-null) set $N\in\mathcal{C}$, namely $Q(N)=0$,
and for all $y\notin N,$ the set $X_{y}=\left\{ x\in X\mid\pi(x)=y\right\} $
has full measure under $P_{y}$, specifically $P_{y}(X_{y})=1$.
\item For any set $A\in\mathcal{B}$, the map $y \mapsto P_{y}(A)$ is
$\mathcal{C}$-measurable and $P(A)=\int_{Y}P_{y}(A)Q(\mathrm{d}y)$.
\end{enumerate}

The following supporting lemma will be used in the proof of the main result.

\begin{lemma}\label{lem:complete_diaconis}
Given a measurable set $\Gamma\in\mathcal{C}$, define the probability measure
$A\in\mathcal{B}\rightarrow\overline{P}(A)=P(\pi^{-1}(\Gamma)\cap A)/P(\pi^{-1}(\Gamma))$.
Assume we are given a regular conditional probability $\overline{P}_{y}$ for $(\pi,\overline{P})$. Then, $y\mapsto P_{y}^\prime(\cdot)$ defined
as 
\[
P_{y}^\prime=\begin{cases}
\overline{P}_{y} & y\in\Gamma,\\
P_{y} & \mathrm{otherwise,}
\end{cases}
\]
is a regular conditional probability for $(\pi,P)$ as defined by
conditions 1 to 3. 
\end{lemma}

\begin{proof}
It is clear that $P_{y}^\prime$ satisfies condition 1. 

For condition 2, let $N,\overline{N}\in\mathcal{C}$ such that $P(\pi^{-1}(N))=0$
and $\overline{P}(\pi^{-1}(\overline{N}))=0$. Take $N^\prime=(\overline{N}\cap\Gamma)\cup(N\cap\Gamma^{c})\in\mathcal{C}$,
then $P(\pi^{-1}(N^\prime))= P(\pi^{-1}(\overline{N})\cap\pi^{-1}(\Gamma))+P(\pi^{-1}(N)\cap\pi^{-1}(\Gamma^{c}))=0$
and clearly, $\forall\,y\notin N^\prime$, $P_{y}^\prime(X_{y})=1$.

For condition 3, for any set $A\in\mathcal{B}$, and any set $R\in \mathcal{B}(\mathbb{R}_{\geq 0})$,
\begin{align*}
\{y\in Y \mid P_{y}^\prime(A)\in R\}= & \{y\in Y \cap \Gamma \mid P_{y}^\prime(A)\in R\}\cup \{y\in Y \cap \Gamma^c \mid P_{y}^\prime(A)\in R\}\\
= & \{y\in Y\cap\Gamma \mid \overline{P}_{y}(A)\in R\} \cup \{y\in Y\cap\Gamma^{c} \mid P_{y}(A)\in R\} \in\mathcal{C}.
\end{align*}

So $y\mapsto P_{y}^\prime(A)$ is $\mathcal{C}$-measurable. Next,
\begin{align*}
\int_{Y}P_{y}^\prime(A)Q({\rm d}y)= & \int_{\Gamma}P_{y}^\prime(A)Q({\rm d}y)+\int_{\Gamma^{c}}P_{y}^\prime(A)Q({\rm d}y)\\
= & \int_{\Gamma}\overline{P}_{y}(A)Q({\rm d}y)+\int_{\Gamma^{c}}P_{y}(A)Q({\rm d}y).
\end{align*}

Since the measure $\overline{Q}(\cdot) := \overline{P}(\pi^{-1}(\cdot))= Q(\cdot \cap \Gamma)/P(\pi^{-1}(\Gamma))$, it follows that
$\int_{\Gamma}\overline{P}_{y}(A)Q({\rm d}y)= P(\pi^{-1}(\Gamma))\int_{Y}\overline{P}_{y}(A)\overline{Q}({\rm d}y)$.  As $\overline{P}_{y}$ is a regular conditional probability for $\overline{P}$, the first term of the displayed equation is also equal to $P(\pi^{-1}(\Gamma))\overline{P}(A)=P(A\cap\pi^{-1}(\Gamma))$. For the second term of the displayed equation,
\begin{align*}
\int_{\Gamma^{c}}P_{y}(A)Q({\rm d}y)&=\int_{\Gamma^{c}}P_{y}(A\cap\pi^{-1}(\Gamma^{c}))Q({\rm d}y)\\
&=\int_{Y}P_{y}(A\cap\pi^{-1}(\Gamma^{c}))Q({\rm d}y)=P(A\cap\pi^{-1}(\Gamma^{c})),
\end{align*}
where the first and second equalities follow from condition 2.
Thus, we have
\[
\int_{Y}P_{y}^\prime(A)Q({\rm d}y)=P(A\cap\pi^{-1}(\Gamma))+P(A\cap\pi^{-1}(\Gamma^{c}))=P(A).
\]
This completes the proof.
\end{proof}

We now state the co-area formula, which is fundamental in the derivation of the regular conditional probability in \citet{coarea-diaconis}.

\begin{definition}[Co-area formula](\citealp[Ch. 3.2]{geomeasuretheory_federer};\citealp{coarea-diaconis}; \citealp[Ch. 3]{hausdorff_evans_2015})
    Let $\pi:\mathbb{R}^d \rightarrow \mathbb{R}^n$ be Lipschitz continuous, where $d \geq n$. For a Lebesgue integrable function $h: \mathbb{R}^d \rightarrow \mathbb{R}$,
$$\int_{\mathbb{R}^d} h(x) J\pi(x)\mathrm{d}x=\int_{\mathbb{R}^n} \Bigg(\int_{\pi^{-1}(y)} h(x) \mathcal{H}^{d-n}(\mathrm{d}x) \Bigg)\mathrm{d}y,$$
where $J\pi$ is defined as in Theorem \ref{thm:hausdorff_maintext}.
\label{def:coareaformula}
\end{definition}

\begin{remark}\label{rem:Gamma_0_Gamma_1}
Assume $P$ admits a density $p$ with respect to Lebesgue on $\mathbb{R}^{d_{\Theta}}$. Let $\Gamma_0:=\{x \in X \mid p(x) > 0\}$,  $\Gamma_1=\{x \in X \mid \pi \text{ is differentiable at }x \text{ and } J\pi(x)>0\}$. \citet{coarea-diaconis} assume $p(x)>0 \Longrightarrow J\pi(x)>0$, or $\Gamma_0\subset \Gamma_1$. Furthermore, they give a construction of $P_y$ for any $y\in Y$ such that $X_y \cap \Gamma_0$ is non-empty. Their construction involves a density that has $J\pi(x)$ appearing in the denominator (see Equation \ref{eqn:bar_P_y} below for the expression) and thus cannot be extrapolated to any part of $X_y$ that lies outside $\Gamma_1$. Hence, they assume  $\Gamma_0 \subset \Gamma_1$, which makes such an extrapolation unnecessary as $P$ has no mass outside $\Gamma_1$.

We cannot assume $p(x)>0 \Longrightarrow J\pi(x)>0$. Firstly, improper policies,  which correspond to a specific subset of $x$ values in $\Gamma_0$, may have $J\pi(x)=0$. To exemplify this point, in Example \ref{ex:2dmdp}, the mapping $\pi$ is $(\theta_1,\theta_2)\mapsto \theta_1 - \max\{\theta_1,\theta_2\}$. Its partial derivatives are $(0,0)$ for $\theta_1>\theta_2$ and thus $J\pi(\theta_1,\theta_2)=0$. However, if the prior for $(\theta_1,\theta_2)$ are independent Gaussians, it has mass everywhere.

Secondly, due to the maximum operator, the constraint function ($\pi$ in this discussion) is only almost-everywhere differentiable, and not differentiable at every point in $\Gamma_0$ as assumed in \citep{coarea-diaconis}.  For example, $(\theta_1,\theta_2)\mapsto \theta_1 - \max\{\theta_1,\theta_2\}$ is not differentiable on the line $\theta_1=\theta_2$.

As a conditional probability needs to completely ``reconstruct'' $P$ (through condition 3), in the proof of Theorem \ref{thm:hausdorff_maintext} below, we extrapolate the result of \citet{coarea-diaconis} to almost-everywhere differentiable mappings $\pi$, and use their construction for $P$ restricted to where $J\pi$ is positive.
\end{remark}

Now, we proceed to prove Theorem \ref{thm:hausdorff_maintext}.

\begin{proof}[of Theorem \ref{thm:hausdorff_maintext}]

Consider the setting where $(\mathbb{R}^{d_{\Theta}},\mathbb{\mathcal{B}}(\mathbb{R}^{d_{\Theta}}),p^{\Theta})$
is a probability space, and $(\mathbb{R}^{n},\mathcal{B}(\mathbb{R}^{n}))$
is a measurable space, where $\mathcal{B}(\mathbb{R}^{k})$ denotes the Borel $\sigma$-algebra
of $\mathbb{R}^{k}$. Let the measure $Q(C)=P(G_{\tau}^{-1}(C))$, $C\in\mathcal{B}(\mathbb{R}^{n})$. 

Let $\Gamma_1=\{\theta \in \mathbb{R}^{d_{\Theta}} \mid D G_\tau(\theta) \text{ exists } \text{ and } JG_\tau(\theta)>0 \}$. For generality, we assume a measurable set $N$ in $\mathbb{R}^{d_{\Theta}}$, which has zero measure under $p^{\Theta}$, and define 
$$\Gamma = \{y \in G_\tau(\mathbb{R}^{d_\Theta}) \mid \mathcal{O}_\tau^y \subseteq \Gamma_1 \cup N\}.$$In the context of the statement of Theorem \ref{thm:hausdorff_maintext},  $N^c=\{\theta \in \mathbb{R}^{d_\Theta} \mid D G_\tau(\theta) \text{ exists}\}$. Since $G_\tau(\theta)$ is Lipschitz, therefore it is differentiable almost everywhere by Rademacher's Theorem \citep{coarea-diaconis}, and $N$ has zero Lebesgue measure.

Let \begin{equation} \overline{p}^{\Theta}(\theta)= \frac{p^{\Theta}(\theta)\mathbbm{1}_{N^c \cap G_{\tau}^{-1}(\Gamma)}(\theta)}{\int_{N^c\cap G_{\tau}^{-1}(\Gamma)}p^{\Theta}(\theta)\mathrm{d}\theta} \label{eq:bar_lowercase_p}\end{equation} and $\overline{P}(\mathcal{O})=\int_{\mathcal{O}}\overline{p}^{\Theta}(\theta)\mathrm{d}\theta$, $\mathcal{O}\in\mathcal{B}(\mathbb{R}^{d_{\Theta}})$. We now construct a regular conditional probability, as defined by conditions 1-3 of Definition \ref{def:rcp}, for $(G_{\tau},\overline{P})$. For any $\mathcal{O}\in\mathcal{B}(\mathbb{R}^{d_{\Theta}})$ and
any $y\in\mathbb{R}^{n}$, let 
\begin{equation}
\overline{P}_{y}(\mathcal{O}):=\begin{cases}
\frac{1}{m(y)}\int_{\mathcal{O}\cap\mathcal{O}_{\tau}^{y}} \frac{\overline{p}^{\Theta}(\theta)}{JG_{\tau}(\theta)}\mathcal{H}^{d_{\Theta}-n}(\mathrm{d}\theta) & 0<m(y)<\infty\\
\delta_{\theta^{*}}(\mathcal{O}) & {\rm otherwise}
\end{cases}
\label{eqn:bar_P_y}
\end{equation}
where $\theta^{*}\in\mathbb{R}^{d_{\Theta}}$ is defined arbitrarily and
$$
m(y):=\int_{\mathcal{O}_{\tau}^{y}  }\frac{\overline{p}^{\Theta}(\theta)}{JG_{\tau}(\theta)}\mathcal{H}^{d_{\Theta}-n}(\mathrm{d}\theta).
$$
Note that the definition of $\Gamma$ ensures $\overline{p}^{\Theta}(\theta)>0 \Longrightarrow JG_\tau(\theta)>0$, so $\overline{P}_y$ is well-defined.  By \cite[Proposition 2]{coarea-diaconis},   \eqref{eqn:bar_P_y} is a regular conditional probability, as defined by conditions 1-3 of Definition \ref{def:rcp}, for $(G_{\tau},\overline{P})$. Specifically, the push-forward measure $\overline{Q}(C)=\overline{P}(G_{\tau}^{-1}(C))$ in Definition \ref{def:rcp} has density $m(\cdot)$ with respect to the Lebesgue measure on $\mathbb{R}^n$ \cite[Propostion 2a]{coarea-diaconis}. Condition 3 of Definition \ref{def:rcp} is verified, using the co-area formula, in the proof of Proposition 2b in \cite{coarea-diaconis}. (For the reader's convenience, we reproduce these two verifications in the final paragraph below.) We can therefore, by Lemma \ref{lem:complete_diaconis}, adopt the regular probability $\overline{P}_{y}$ in \eqref{eqn:bar_P_y} for $y\in \Gamma$,
and use an alternate construction for $y \notin \Gamma$ (which always exists \cite{rcp_def_parthasarathy}) as the regular conditional probability for $(G_{\tau},p^{\Theta})$.

For $\mathcal{O} \in \mathcal{B}(\mathbb{R}^{d_\Theta})$, apply the co-area formula in Defintion \ref{def:coareaformula}  with $h(\theta)=\overline{p}^{\Theta}(\theta) \mathbbm{1}_{\mathcal{O}}(\theta)/J G_{\tau}(\theta)$ to get
$$\overline{P}(\mathcal{O})=\int_{\mathbb{R}^n}  \int_{\mathcal{O} \cap \mathcal{O}^y_\tau} \frac{\overline{p}^\Theta(\theta)}{JG_\tau(\theta)}\mathcal{H}^{d_\Theta-n}(\mathrm{d}\theta) \mathrm{d}y=\int_{\mathbb{R}^n} m(y) \overline{P}_y(\mathcal{O}) \mathrm{d}y.$$
Furthermore, for any $C \in \mathcal{B}(\mathbb{R}^n)$, 
$$\overline{Q}(C):=\overline{P}(G_\tau^{-1}(C))=\int_C \int_{\mathcal{O}^y_\tau } \frac{\overline{p}^\Theta(\theta)}{JG_\tau(\theta)}\mathcal{H}^{d_\Theta-n}(\mathrm{d}\theta)\mathrm{d}y= \int_C m(y)\mathrm{d}y.$$ This implies that
the measure $\overline{Q}(\cdot)$ in Definition \ref{def:rcp} has density $m(\cdot)$ with respect to the Lebesgue measure on $\mathbb{R}^n$. Therefore, also verifying condition 3.

\end{proof}

\subsubsection{Gradient of \texorpdfstring{$\theta \mapsto g_{s,a}(\theta)$}{g} for tabular \texorpdfstring{$Q_\theta$}{Qθ}}
\label{apd:gradq}

Recall the definition of $g_{s,a}$, 
     $$g_{s,a}(\theta) = Q_\theta(s,a)
     -\sum_{s^\prime \in \mathcal{S}} p^S(s^\prime|s,a)
     \max_{a^\prime \in \mathcal{A}_{s^\prime}}Q_\theta(s^\prime, a^\prime).$$

Then, for any $\theta \in \{\vartheta \in \Theta \mid \forall s \in \mathcal{S} \setminus \{s^g\}, \,|\argmax_{a \in \mathcal{A}_s} \vartheta_{\nu(s,a)}|=1\}$, the set of differentiable $\theta \in \Theta$, 
 $$
 \nabla_\theta g_{s,a}(\theta) =  \nabla_\theta Q_\theta(s,a) - \sum_{s^\prime \in \mathcal{S}} \sum_{a^\prime \in \mathcal{A}_{s^\prime}}  \nabla_\theta Q_\theta(s^\prime, a^\prime) p^S(s^\prime|s,a) \mathbbm{1}\Big(a^\prime \in \argmax_{a^{\prime\prime} \in \mathcal{A}_{s^\prime}}Q_\theta(s^\prime,a^{\prime\prime})\Big).
 $$

 Note that each of the $\argmax$ only contains one element because of the set of differentiable $\theta$ and the fact that $\mathcal{A}_{s^g}=\{a^g\}$.

As $Q_\theta(s,a) = \theta_{\nu(s,a)} = \sum_{j=1}^{d_{\Theta}} \theta_{j}\mathbbm{1}(j=\nu(s,a))$ for $s\in \mathcal{S}, a \in \mathcal{A}_s$,
 $$
 \frac{\partial g_{s,a}(\theta)}{\partial \theta_k} = \mathbbm{1}(k=\nu(s,a)) - p^S(s^k|s,a) \mathbbm{1}\Big(a^k \in \argmax_{a^{\prime} \in \mathcal{A}_{s^k}}\theta_{\nu(s^k,a^{\prime})}\Big)
 $$
 for $k \in \{1,\dots,d_{\Theta}\}$, where $(s^k,a^k):= \nu^{-1}(k)$.

Finally, as a simple check for the special case of $s^g$, as $\nu(s^g,a^g) > d_{\Theta}$ and $s^k \neq s^g$ for any $k \in \{1,\dots,d_\Theta\}$, $\mathbbm{1}(k=\nu(s^g,a^g))=0$ and $p^S(s^k|s^g,a^g)=0$, and hence, $\frac{\partial g_{s^g,a^g}(\theta)}{\partial \theta_k}=0$.

\subsubsection{Theorem \ref{thm:hausdorff_maintext_tabular} (Satisfaction of assumptions of Theorem \ref{thm:hausdorff_maintext}) and the corollary in Remark \ref{rem:hausdorff_corr}} \label{apd:hausdorff_assumptions_proof}

\begin{proof}[of Theorem \ref{thm:hausdorff_maintext_tabular}]
Firstly, we show that $G_\tau$ is Lipschitz continuous if $\theta \mapsto Q_\theta(s,a)$ is Lipschitz continuous.

\begin{lemma}\label{lem:lipschitz} 
Let $\{(s_i,a_i)\}_{i=1}^n$ be a data set of state-action pairs that excludes the goal state: $s_i \in \mathcal{S} \setminus \{s^g\}$ and $a_i \in \mathcal{A}_{s_i}$. Furthermore, let the function $G:\Theta \rightarrow \mathbb{R}^n$ be $G(\theta)_i=g_{s_i,a_i}(\theta)$.

If $\theta \mapsto Q_\theta(s,a)$ is $L$-Lipschitz, with respect to the Euclidean distance $\lVert\cdot\rVert_2$, for every $s \in \mathcal{S}$ and $a \in \mathcal{A}_s$, then $G(\cdot)$ is also Lipschitz.  In particular, $G(\cdot)$ is Lipschitz when $Q_\theta$ is tabular as defined in Definition \ref{def:tabular}.
\end{lemma}

\begin{proof}
Consider an arbitrary $s \in \mathcal{S}\setminus\{s^g\}$ and $a \in \mathcal{A}_s$,
    \begin{align*}
    |g_{s,a}(\theta)-g_{s,a}(\theta^\prime)| \leq &|Q_\theta(s,a)-Q_{\theta^\prime}(s,a)| \\
    &+ \Bigg|\mathbb{E}\Big[\max_{a^\prime \in \mathcal{A}_{S_1}} Q_\theta(S_1,a^\prime)-\max_{a^\prime \in \mathcal{A}_{S_1}} Q_{\theta^\prime}(S_1,a^\prime) \Big|S_0=s,A_0=a \Big]\Bigg|.
    \end{align*}
Since $Q_\theta$ is $L$-Lipschitz for all $s\in \mathcal{S}$, $a \in \mathcal{A}_s$, 
$$|Q_\theta(s,a)-Q_{\theta^\prime}(s,a)| \leq L \lVert\theta-\theta^\prime\rVert_2.$$
Furthermore,
\begin{align*}
    &\Bigg|\mathbb{E}\Big[\max_{a^\prime \in \mathcal{A}_{S_1}} Q_\theta(S_1,a^\prime)-\max_{a^\prime \in \mathcal{A}_{S_1}} Q_{\theta^\prime}(S_1,a^\prime) \Big|S_0=s,A_0=a \Big]\Bigg|\\
    \leq& \mathbb{E} \Big[ \max_{a^\prime \in \mathcal{A}_s} | Q_\theta(S_1,a^\prime)-Q_{\theta^\prime}(S_1,a^\prime) | \Big| S_0=s,A_0=a\Big] \\
    \leq & L \lVert\theta-\theta^\prime\rVert_2,
\end{align*}
where the first inequality follows from Jensen's inequality and the fact that the mapping  $x \in \mathbb{R}^m \mapsto \lVert x\rVert_\infty$ is 1-Lipschitz with respect to $\lVert\cdot\rVert_\infty$ (for any positive integer $m$). For any $\theta,\theta^\prime \in \Theta$,
    $$\lVert G(\theta)-G(\theta^\prime)\rVert_2^2=\sum_{i=1}^n (g_{s_i,a_i}(\theta)-g_{s_i,a_i}(\theta^\prime))^2.$$
Combining the inequalities above, we have
$$\lVert G(\theta)-G(\theta^\prime)\rVert_2 \leq 2\sqrt{n}L \lVert\theta-\theta^\prime\rVert_2.$$
Finally, if $Q_\theta$ is tabular, as $|Q_\theta(s,a)-Q_{\theta^\prime}(s,a)|=|\theta_{\nu(s,a)}-\theta^\prime_{\nu(s,a)}|\leq \lVert\theta-\theta^\prime\rVert_2$, implying that $G(\theta)$ is $2\sqrt{n}$-Lipschitz.
\end{proof}

 We now present the following supporting lemma for proving Lemma \ref{lem:cor:hausdorff}.

\begin{lemma}\label{lem:bellmanfirstpart}
    Let Assumption \ref{ass:boeunique1} hold.
    Then, for any proper (deterministic) policy $\mu:\mathcal{S} \rightarrow \mathcal{A}$, the fixed-point equation
    $$Q(s,a)=\sum_{s^\prime \in \mathcal{S}\setminus \{s^g\}} p^S(s^\prime|s,a) Q(s^\prime,\mu(s^\prime)), \qquad \text{with domain }s \in \mathcal{S} \setminus \{s^g\},\,a \in \mathcal{A}_s$$
    has a unique solution, which is $Q \equiv 0$.
\end{lemma}
When we only consider state-action pairs $(s,\mu(s))$, these equations can be interpreted as the Bellman equations for an SSP problem where $\mathcal{A}_s=\{\mu(s)\}$, for all $s$, and the reward function is zero for all state-action pairs. Since $\mu$ is proper, it may seem intuitive that $Q \equiv 0$ is the unique solution. However, we could not find an explicit statement in the literature to verify this. Thus, a proof is presented below that follows the techniques in Proposition 3.2.1 of \citet{bertsekas2012}.
\begin{proof}
Let $s \in \mathcal{S} \setminus \{s^g\},\,a \in \mathcal{A}_s$. Define the Bellman operator $f_\mu$ for $\mu$ as
$$(f_\mu(Q))(s,a)=\sum_{s^\prime \in \mathcal{S} \setminus \{s^g\}} Q(s^\prime,\mu(s^\prime)) p^S(s^\prime|s,a).$$
Then, it is clear that $Q \equiv 0$ is a solution to $f_\mu(Q)=Q$. Suppose $Q^\prime$ is also a solution, i.e., $f_\mu(Q^\prime)\equiv Q^\prime$. Then, $f_\mu^k(Q^\prime)=Q^\prime$ for any positive integer $k$. We now rewrite $f^k_\mu$ as follows:
    \begin{equation} (f_\mu^k(Q))(s,a) = \sum_{s_k \in \mathcal{S}\setminus \{s^g\}} Q(s_k,\mu(s_k)) p^\mu(S_k=s_k|S_0=s,A_0=a). \label{eqn:lemma_fpi}
    \end{equation}
This follows from the following inductive argument---Firstly, the base case $k=1$ follows from the definition of $f_\mu$. Suppose Equation \ref{eqn:lemma_fpi} holds for $k=\ell$. For $k=\ell+1$,
\begin{align*}
(f^{\ell+1}_\mu(Q))(s,a) &= \sum_{s^\prime \in \mathcal{S} \setminus \{s^g\}} (f_\mu^{\ell}(Q))(s^\prime,\mu(s^\prime))p^S(s^\prime|s,a) \\
&=  \sum_{s^\prime \in \mathcal{S} \setminus \{s^g\}} \sum_{s^{\prime\prime} \in \mathcal{S} \setminus \{s^g\}} Q(s^{\prime\prime},\mu(s^{\prime\prime})) p^\mu(S_\ell=s^{\prime\prime}|S_0=s^\prime,A_0=\mu(s^\prime)) p^S(s^\prime|s,a)\\
&= \sum_{s^{\prime\prime} \in \mathcal{S} \setminus \{s^g\}} Q(s^{\prime\prime},\mu(s^{\prime\prime})) p^\mu (S_{\ell+1}=s^{\prime\prime}|S_0=s,A_0=a),
\end{align*}
where the final equality uses that $p^\mu(S_\ell=s^{\prime\prime}|S_0=s^g, A_0=\mu(s^g)) p^S(s^g|s,a)=0$ for all $s^{\prime \prime} \in \mathcal{S} \setminus \{s^g\}$. This finishes the inductive argument.

Since $\mu$ is proper, $\lim_{k \rightarrow \infty }p^\mu(S_k=s|S_0=s_0) = \mathbbm{1}(s =s^g)$ for any $s_0,s \in \mathcal{S}$. Therefore, $Q^\prime \equiv \lim_{k\rightarrow \infty} f^k_\mu(Q^\prime)\equiv0$. In other words, $Q \equiv 0$ is the unique solution to $f_\mu(Q)\equiv Q$.
\end{proof}

Now, we show in Lemma \ref{lem:cor:hausdorff} that a tabular $Q_\theta$ satisfies the differentiability assumption of Theorem \ref{thm:hausdorff_maintext}. The aim of this lemma is to establish the positive definiteness of $JG_\tau$ whenever $G_\tau$ is differentiable as required in Theorem \ref{thm:hausdorff_maintext_tabular}. For the reader's convenience, we repeat the preamble of Theorem \ref{thm:hausdorff_maintext_tabular}.

The MDP $\mathcal{M}$ satisfies Assumption
\ref{ass:boeunique1}, and $\mathcal{D}_\tau^{\mathcal{S},\mathcal{A}}=\{(s_i,a_i)\}_{i=1}^n$ is the re-indexed unique non-goal state-action components of the data $\mathcal{D}_\tau=\{(s_t,a_t,r_t)\}_{t=0}^\tau$. The corresponding (deterministic) reward vector is $\bar{r}:=(\bar{r}_{s_1,a_1},\dots,\bar{r}_{s_n,a_n})^\top$. The mapping $G_\tau:\Theta\rightarrow \mathbb{R}^n$ is   $G_\tau(\theta)_i=g_{s_i,a_i}(\theta)$, for $i=1,\ldots,n$; and $JG_\tau(\theta):= \sqrt{\det{(DG_\tau(\theta) DG_\tau(\theta)^\top)}}.$

\begin{lemma}\label{lem:cor:hausdorff} Assume $Q_\theta$ is tabular as defined in Definition \ref{def:tabular}. If $\mathcal{M}$  satisfies Assumptions \ref{ass:boeunique1} and \ref{ass:boeunique2}, then $JG_\tau(\theta) > 0$ for any $\theta \in \{\vartheta \in \Theta \mid G_\tau(\vartheta)=\bar{r}\}$ at which $G_\tau$ is differentiable.
\end{lemma}

The proof strategy involves constructing a related MDP $\widetilde{\mathcal{M}}$ such that it generates the same subset of BOEs as those forming the constraint $G_\tau$, and thus shares the same preimage set $\mathcal{O}_\tau^{\bar{r}}$. Furthermore, for any $\theta \in \mathcal{O}_\tau^{\bar{r}}$, the corresponding greedy policy is proper for the MDP $\widetilde{\mathcal{M}}$.
By properness, we can then show that the partial derivative vectors of MDP $\widetilde{\mathcal{M}}$'s constraint function are linearly independent, which in turn implies the same for the derivative vectors of $G_\tau$. Linear independence will then verify $JG_\tau(\theta) > 0$.

\begin{proof}

Without loss of generality, let  $\nu(s_i,a_i)=i$ for all $i \in \{1,\dots,n\}$, and extend the indices to label all of $(\mathcal{S} \setminus \{s^g\}) \otimes \mathcal{A}$ so that  $\nu(s_i,a_i)=i$ for all $i \in \{1,\dots,d_\Theta\}$.

We first verify the differentiability of $G_\tau$ and give its gradient. By Lemma \ref{lem:lipschitz}, $\theta \mapsto G_\tau(\theta)$ is Lipschitz and is therefore differentiable almost everywhere by Rademacher's Theorem. Furthermore, it is easy to see that $G_\tau(\theta)$ is differentiable at $\theta$ if and only if  $\argmax_{a\in \mathcal{A}_s} \theta_{\nu(s,a)}$ is a singleton for all $s \in \mathcal{S}$. 

For any $\theta$ at which $G_\tau$ is differentiable, let $\mu_{\theta}:\mathcal{S} \rightarrow \mathcal{A}$  be the greedy policy of $\theta$, that is,  $\mu_{\theta}(s) = \argmax_{a \in \mathcal{A}_s} \theta_{\nu(s,a)}$, where the equality follows from that the sets $\argmax$ are singletons for all states in $\mathcal{S}$. From here on, the subscript is omitted in $\mu$ to declutter the mathematical expressions. By Appendix \ref{apd:gradq},
$$\frac{\partial G_\tau(\theta)_i}{\partial \theta_k}=\mathbbm{1}(k=\nu(s_i,a_i)) - p^S(s_k|s_i,a_i)  \mathbbm{1}(a_k =\mu(s_k) )
$$
 for $k \in \{1,\dots,d_{\Theta}\}$, where $(s_k,a_k)\equiv\nu^{-1}(k)$.\newline

 We now define a new MDP $\widetilde{\mathcal{M}}$ that has the same state space, action space, reward distribution and initial state distribution as $\mathcal{M}$, but is equipped with a new transition dynamics $\tilde{p}^S$ such that the available data set $\mathcal{D}_\tau^{\mathcal{S},\mathcal{A}}$ effectively contains a visit to every non-goal state-action pair of $\widetilde{\mathcal{M}}$. Specifically, for any $s^\prime \in \mathcal{S}$,
$$\tilde{p}^S(s^\prime|s,a) = \begin{cases} p^S(s^\prime|s,a) & \text{if }(s,a) \in \mathcal{D}_\tau^{\mathcal{S},\mathcal{A}} \\
\mathbbm{1}(s^\prime=s^g) & \text{otherwise} \end{cases}.$$ 
Furthermore, define $\widetilde{G}_\tau:\Theta \rightarrow \mathbb{R}^n$, where for any $i \in \{1,\dots,n\}$,
$$(\widetilde{G}_\tau(\theta))_i := Q_{\theta}(s_i,a_i) - \sum_{s^\prime \in \mathcal{S}} \max_{a^\prime \in \mathcal{A}_{s^\prime}} Q_\theta(s^\prime,a^\prime) \tilde{p}^S(s^\prime|s_i,a_i)=(G_\tau(\theta))_i.$$
Therefore, $\mathcal{M}$ and $\widetilde{\mathcal{M}}$ share the same subset of BOEs over $\mathcal{D}_\tau^{\mathcal{S},\mathcal{A}}$ and $DG_\tau(\theta) = D\widetilde{G}_\tau(\theta)$ for all $\theta \in \Theta$. Finally, define the extended constraint $\widetilde{G}:\Theta \rightarrow \mathbb{R}^{d_\Theta}$, 
$$(\widetilde{G}(\theta))_i:= Q_{\theta}(s_i,a_i) -\sum_{s^\prime \in \mathcal{S}} \max_{a^\prime \in \mathcal{A}_{s^\prime}} Q_\theta(s^\prime,a^\prime) \tilde{p}^S(s^\prime|s_i,a_i),\quad i=1,\ldots, d_\Theta.$$
Using the definition of $\tilde{p}^S$, we see that $(\widetilde{G}(\theta))_i = Q_{\theta}(s_i,a_i) = \theta_i$ for $i=n+1,\ldots, d_\Theta$.

\begin{sloppypar}
 The main body of the proof is now presented, commencing with the claim that $\{( \frac{\partial G_\tau(\theta)_i}{\partial \theta_1},\dots,\frac{\partial G_\tau(\theta)_i}{\partial \theta_{d_\Theta}})^\top\}_{i=1}^n$ are linearly independent. However, this is equivalent to claiming $\{( \frac{\partial \widetilde{G}_\tau(\theta)_i}{\partial \theta_1},\dots,\frac{\partial \widetilde{G}_\tau(\theta)_i}{\partial \theta_{d_\Theta}})^\top\}_{i=1}^n$ are linearly independent, which in turn is true if $\{( \frac{\partial \widetilde{G}(\theta)_i}{\partial \theta_1},\dots,\frac{\partial \widetilde{G}(\theta)_i}{\partial \theta_{d_\Theta}})^\top\}_{i=1}^{d_\Theta}$ are linearly independent. By the fact that the column rank and the row rank of a matrix are always equal,
\begin{align*}
    &\bigg\{\bigg( \frac{\partial \widetilde{G}(\theta)_1}{\partial \theta_k},\dots,\frac{\partial \widetilde{G}(\theta)_{d_\Theta}}{\partial \theta_{k}}\bigg)^\top\bigg\}_{k=1}^{d_\Theta} \text{ are linearly independent} \\
    &\qquad\iff \bigg\{\bigg( \frac{\partial \widetilde{G}(\theta)_i}{\partial \theta_1},\dots,\frac{\partial \widetilde{G}(\theta)_i}{\partial \theta_{d_\Theta}}\bigg)^\top\bigg\}_{i=1}^{d_\Theta} \text{ are linearly independent}.
\end{align*}
\end{sloppypar}
 We verify the linear independence of the vectors in the first part of the displayed assertion. To show it, for  $\alpha\in \mathbb{R}^{d_\Theta}$,
 \begin{align*} &\sum_{k=1}^{d_\Theta} \alpha_k\frac{\partial \widetilde{G}(\theta)_i}{\partial \theta_k} = 0 \quad \forall \, i \in \{1,\dots,d_\Theta\}\\
 \iff&  \sum_{k=1}^{d_\Theta} \alpha_k \Big[\mathbbm{1}(k=\nu(s_i,a_i)) - \tilde{p}^S(s_k|s_i,a_i) \mathbbm{1}\Big(a_k \in \argmax_{a^{\prime} \in \mathcal{A}_{s_k}}\theta_{\nu(s_k,a^{\prime})}\Big)\Big]=0 \quad \forall \, i\\
 \iff & \alpha_i -\sum_{k=1}^{d_\Theta} \alpha_k 
 \tilde{p}^S(s_k|s_i,a_i)   \mathbbm{1}(a_k=\mu(s_k))=0 \quad \forall \, i \\
 \iff& \alpha_i -\sum_{s^\prime \in \mathcal{S} \setminus \{s^g\}} \tilde{p}^S(s^\prime|s_i,a_i) \alpha_{\nu(s^\prime,\mu(s^\prime))}=0  \quad \forall \, i .
 \end{align*}
Given that $\mu$ is proper in $\widetilde{\mathcal{M}}$---a proof is provided at the end---it follows from Lemma \ref{lem:bellmanfirstpart} that $\alpha=0$ is the unique solution, thus verifying the claim that $\{( \frac{\partial G_\tau(\theta)_i}{\partial \theta_1},\dots,\frac{\partial G_\tau(\theta)_i}{\partial \theta_{d_\Theta}})^\top\}_{i=1}^n$ are linearly independent.

Therefore, for any $v \in \mathbb{R}^n$, $DG_\tau(\theta)^\top v=0$ if and only if $v=0$ by the rank-nullity theorem. It follows that $DG_\tau(\theta) DG_\tau(\theta)^\top$ is positive definite. This implies that $\det(DG_\tau(\theta)DG_\tau(\theta)^\top) > 0$.

The final step in the proof is to show that $\mu$ is proper for $\widetilde{\mathcal{M}}$.
 Suppose instead $\mu$ is improper for $\widetilde{\mathcal{M}}$. Then, there exists $s \in \mathcal{S}$ such that $\lim_{t \rightarrow \infty} \tilde{p}^\mu(S_t=s^g|S_0=s)<1$. Consider the Markov chain with transition kernel 
 $$\tilde{p}^{\mu}(S_{t+1}=s^\prime|S_t=s)=\tilde{p}^S(S_{t+1}=s^\prime| s,\mu(s) )$$ for $s,s^\prime \in \mathcal{S}$. Then,  by standard arguments for finite state-space Markov chains \citep{markovchain}, there exists a closed set of recurrent states $\mathcal{C} \subseteq \mathcal{S}$ under $\mu$, $s^g \notin \mathcal{C}$, and for any $t \in \mathbb{Z}_{\geq 1}$,
 $$\tilde{p}^\mu(S_t \in \mathcal{C}|S_0 \in \mathcal{C})=1.$$
 Since $p^S(s^g|s,a)=1$ for any $(s,a) \notin \mathcal{D}^{\mathcal{S},\mathcal{A}}_\tau$, it follows that $\mathcal{C} \subseteq \{s_1,\dots,s_n\}$ and $(s_i,\,\mu(s_i)) \in \mathcal{D}^{\mathcal{S},\mathcal{A}}_\tau$ for $i=1,\ldots,n$; otherwise  there will be an immediate transition to $s^g$ and the set is not closed under $\tilde{p}^\mu$. Since $\mathcal{M}$ and $\widetilde{\mathcal{M}}$ have the same transition distributions from state-action pairs in $\mathcal{D}_\tau^{\mathcal{S},\mathcal{A}}$, we can conclude that $\mathcal{C}$ is also a closed set of recurrent states in $\mathcal{M}$ under $\mu$. Thus $\mu$ is also improper for $\mathcal{M}$.

 Recall that $\theta$ satisfies the following subset of BOEs of $\mathcal{M}$,
 $$Q_\theta(s,a) = \bar{r}_{s,a} + \mathbb{E}[Q_\theta(S_1,\mu(S_1))|S_0=s,A_0=a] \qquad \forall \, (s,a) \in \mathcal{D}_\tau^{\mathcal{S},\mathcal{A}},$$ where the inner {\it $\max_{a^\prime} Q_\theta(S_1,a^\prime)$} of the BOE is found by  $\mu(S_1)$.

Pick $s \in \mathcal{C}$. By expanding the $Q_{\theta}$ term in the expectation of the BOEs, using the fact that $\mathcal{C}$ is closed under $\mu$ and $\theta$ satisfies the BOEs of $\{(s_i,\mu(s_i))\}_{i=1}^n \subseteq \mathcal{D}_\tau^{\mathcal{S},\mathcal{A}}$,
\begin{equation*}
Q_\theta(s,\mu(s))=\mathbb{E}\Bigg[\sum_{t=0}^{\tau-1} \bar{r}_{S_t,\mu(S_t)}\Bigg|S_0=s\Bigg]+\mathbb{E}[Q_\theta(S_\tau,\mu(S_\tau))|S_0=s],
\end{equation*}
for all $\tau \in \mathbb{Z}_{\geq 1}$.

Now, as $\tau \rightarrow \infty$, the sum diverges because $\mu$ is improper for $\mathcal{M}$ and $\mathcal{M}$ satisfies Assumption \ref{ass:boeunique2}: that is, any improper policy contains a state that yields negative infinite expected cumulative rewards. Thus, this results in a contradiction since $Q_\theta(s,a)$ is bounded.
\end{proof}
This completes the proof.
\end{proof}

\begin{corollary}\label{cor:tabularhausdorff_m}
Assume $Q_\theta$ is tabular, as defined in Definition \ref{def:tabular},  and $\mathcal{M}$  satisfies Assumptions \ref{ass:boeunique1} and \ref{ass:boeunique2}. Furthermore, assume there exists a $\bar{\theta} \in \mathcal{O}_\tau^{\bar{r}}$ at which $G_\tau$ is differentiable. Then, a sufficient condition for $m(\bar{r}) > 0$ is that $p^\Theta(\theta)>0$ for all $\theta \in \mathcal{O}_\tau^{\bar{r}}$.
\end{corollary}

\begin{proof}The outline of the proof is as follows. We first show $\mathcal{H}^{d_\Theta-n}(\mathcal{O}_\tau^{\bar{r}} \cap N^c) >0$. Thus, the integral that defines $m(\bar{r})$ will also be positive since the measure of the set $\mathcal{O}_\tau^{\bar{r}} \cap N^c$ and integrand are both positive, that is, $p^\Theta(\theta)>0$ and $JG_\tau(\theta) > 0$ (Lemma \ref{lem:cor:hausdorff}) for all $\theta \in \mathcal{O}_\tau^{\bar{r}} \cap N^c$. 

     Since $N^c$ is open, there exists $\epsilon > 0$ such that $\bar{\theta} \in B_\epsilon(\bar{\theta}) \subset N^c$, where $B_\epsilon(\bar{\theta}):= \{\theta \in \Theta \mid \lVert \theta-\bar{\theta}\rVert_2 < \epsilon\}$ denotes the $\epsilon$-ball centred at $\bar{\theta}$.
     It is sufficient to show that $\mathcal{H}^{d_\Theta-n}(\mathcal{O}_\tau^{\bar{r}} \cap B_\epsilon(\bar{\theta})) > 0$. 
     
    In the proof of Lemma \ref{lem:cor:hausdorff}, it was shown that $DG_\tau(\bar{\theta})$ has full rank. Without loss of generality, assume that $\{\frac{\partial G_\tau}{\partial \theta_i}\}_{i=1}^n$ are linearly independent; otherwise, simply permute the coordinates of $\Theta$. Partition $\bar{\theta}=((\bar{\theta}^1)^\top,(\bar{\theta}^2)^\top)^\top$ such that $\bar{\theta}^1 \in \mathbb{R}^n$ and $\bar{\theta}^2 \in \mathbb{R}^{d_\Theta-n}$. For the tabular parametrisation, $G_\tau$ is continuously differentiable on the open ball $B_\epsilon(\bar{\theta})$. The implicit function theorem \citep{analysis_reference_rudin} guarantees the existence of a $\delta$-ball $B_\delta(\bar{\theta}^2) \subset \mathbb{R}^{d_\Theta-n}$ and a unique continuously differentiable function $h:B_\delta(\bar{\theta}^2) \rightarrow \mathbb{R}^n$ satisfying $h(\bar{\theta}^2)=\bar{\theta}^1$ and $G_\tau((h(y)^\top,y^\top)^\top)=\bar{r}$ for all $y \in B_\delta(\bar{\theta}^2)  \subseteq \mathbb{R}^{d_\Theta-n}$. Consequently, by the continuity of $h$, the set
     $$U:= \{ (h(y)^\top,y^\top)^\top \mid y \in B_\delta(\bar{\theta}^2)\} \subseteq \mathcal{O}_\tau^{\bar{r}} \cap B_\epsilon(\bar{\theta})$$ 
     for a sufficiently small $\delta > 0$. 
     Furthermore, by \citet[Theorem 2.8 (ii)]{hausdorff_evans_2015}, 
     $\mathcal{H}^{d_\Theta-n}(U)\geq \lambda(B_\delta(\bar{\theta}^2)) > 0$, where $\lambda$ denotes the Lebesgue measure on $\mathbb{R}^{d_\Theta-n}$. Therefore, $\mathcal{H}^{d_\Theta-n}(\mathcal{O}_\tau^{\bar{r}} \cap N^c) \geq \mathcal{H}^{d_\Theta-n}(\mathcal{O}_\tau^{\bar{r}} \cap B_\epsilon(\bar{\theta})) \geq \mathcal{H}^{d_\Theta-n}(U) > 0$.
\end{proof}

 \subsection{Theoretical form of posterior under tabular \texorpdfstring{$Q_\theta$}{Qθ} and Gaussian likelihood}
 \label{apd:proofpsrl}

\begin{proof}[of Proposition \ref{prop:psrl_new}]
Firstly, rewrite $$\hat{p}_\epsilon(\theta \in E|\mathcal{D}_\tau) = \frac{\int_{E}p(\theta,\bar{r}_{1:n}|s_{1:n},a_{1:n})\mathrm{d}\theta}{p(\bar{r}_{1:n}|s_{1:n},a_{1:n})}$$
where
$$p(\theta,\bar{r}_{1:n}|s_{1:n},a_{1:n}) :=\hat{L}_\epsilon(\theta|\mathcal{D}_\tau) p^\Theta(\theta),$$
and
$$p(\bar{r}_{1:n}|s_{1:n},a_{1:n}) = \int_\Theta p(\theta,\bar{r}_{1:n}|s_{1:n},a_{1:n})\mathrm{d} \theta.$$
Recall that
$$\mathcal{L}^{\mathcal{D}_\tau} := \{\ell : \mathcal{S}^{\mathcal{D}_\tau} \rightarrow \mathcal{A} \mid \ell(s) \in \mathcal{A}_s \,\, \forall \, s \in \mathcal{S}^{\mathcal{D}_\tau}\}.$$
Since
$$E^\ell = \{\theta \in \Theta \mid \theta_{\nu(s,\ell(s))} = \max_{a \in \mathcal{A}_s} \theta_{\nu(s,a)} \,\,\forall \,s\in \mathcal{S}^{\mathcal{D}_{\tau}}\}, \quad \ell \in \mathcal{L}^{\mathcal{D}_\tau},$$
it is easy to see that $\bigcup_{\ell \in \mathcal{L}^{\mathcal{D}_\tau}} E^\ell = \Theta$; and for $\ell,\ell^\prime \in \mathcal{L}^{\mathcal{D}_\tau}$ such that $\ell \neq \ell^\prime$, we have $E^\ell \cap E^{\ell^\prime} = \emptyset$ $p^\Theta$-a.s. To see this, let 
$$E^\Theta = \{\theta \in \Theta \mid \theta_i \neq \theta_j \,\,\forall \, i,j \in \{1,\dots,d_{\Theta}\}, i \neq j\}.$$
It is clear that $p^\Theta(E^\Theta)=1$.

With the partition $\{E^\ell\}_{\ell \in \mathcal{L}^{\mathcal{D}_\tau}}$ of $\Theta$, we can now express $p(\theta,\bar{r}_{1:n},\theta \in E^\Theta|s_{1:n},a_{1:n})$ as:
\begin{align*}
&p(\theta,\bar{r}_{1:n},\theta \in E^\Theta|s_{1:n},a_{1:n}) \\
=& \Bigg[ \prod_{i=1}^n \sum_{\ell \in \mathcal{L}^{\mathcal{D}_\tau}} \mathcal{N}\Big(\bar{r}_i;\theta_{\nu(s_i,a_i)}-\sum_{s_i^\prime \in \mathcal{S}} p^S(s_i^\prime|s_i,a_i)\max_{a_i^\prime \in \mathcal{A}_{s_i^\prime}}\theta_{\nu(s_i^\prime,a_i^\prime)},\epsilon^2\Big) \mathbbm{1}(\theta \in E^{\ell})\Bigg] p^\Theta(\theta) \mathbbm{1}(\theta \in E^\Theta)\\
=& \sum_{\ell \in \mathcal{L}^{\mathcal{D}_\tau}} \Big[ \prod_{i=1}^n \mathcal{N} \Big(\bar{r}_i;\theta_{\nu(s_i,a_i)}- \sum_{s_i^\prime \in \mathcal{S}} p^S(s_i^\prime|s_i,a_i)\max_{a_i^\prime \in \mathcal{A}_{s_i^\prime}}\theta_{\nu(s_i^\prime,a_i^\prime)}, \epsilon^2 \Big) \mathbbm{1}(\theta \in E^{\ell}) \Big] p^\Theta(\theta)\mathbbm{1}(\theta \in E^\Theta) \\
=& \sum_{\ell \in \mathcal{L}^{\mathcal{D}_\tau}} [p(\theta,\bar{r}_{1:n},\theta \in E^{\ell}|s_{1:n},a_{1:n})]\mathbbm{1}(\theta \in E^\Theta).
\end{align*}

Next, rewrite the numerator of $\hat{p}_\epsilon(\theta \in E|\mathcal{D}_\tau)$ using the partition:
\begin{align*}
   &  \int_{E} p(\theta,\bar{r}_{1:n}|s_{1:n},a_{1:n})\mathrm{d} \theta \\
    =& \int_{E} p(\theta,\bar{r}_{1:n}, \theta \in E^\Theta|s_{1:n},a_{1:n}) + p(\theta,\bar{r}_{1:n}, \theta \in (E^\Theta)^c|s_{1:n},a_{1:n})\mathrm{d} \theta \\
    =& \int_{E} \sum_{\ell \in \mathcal{L}^{\mathcal{D}_\tau}} [p(\theta,\bar{r}_{1:n},\theta \in E^{\ell}|s_{1:n}, a_{1:n})]\mathbbm{1}(\theta \in E^\Theta) \mathrm{d}\theta \\
    =& \sum_{\ell \in \mathcal{L}^{\mathcal{D}_\tau}} \int_{E \cap E^\ell \cap E^\Theta} p(\theta,\bar{r}_{1:n}|s_{1:n},a_{1:n}) \mathrm{d} \theta \\
    =& \sum_{\ell \in \mathcal{L}^{\mathcal{D}_\tau}} \int_{E \cap E^\ell} p(\theta,\bar{r}_{1:n}|s_{1:n},a_{1:n}) \mathrm{d} \theta,
\end{align*}
and similarly, for the denominator,
\begin{equation*}
   p(\bar{r}_{1:n}|s_{1:n},a_{1:n})   = \sum_{\ell \in \mathcal{L}^{\mathcal{D}_\tau}} \int_{E^\ell} p(\theta,\bar{r}_{1:n}|s_{1:n},a_{1:n}) \mathrm{d} \theta.
\end{equation*}

We now define auxiliary distributions that can help us to evaluate the integrals. 

Consider an auxiliary joint distribution $p^{\ell}$ as follows:
$$p^{\ell}(\theta,\bar{r}_{1:n}) := \prod_{i=1}^n \mathcal{N} \Big(\bar{r}_i;\theta_{\nu(s_i,a_i)} - 
\sum_{s_i^\prime \in \mathcal{S}} p^S(s_i^\prime|s_i,a_i) \theta_{\nu(s_i^\prime,\ell(s_i^\prime))},\epsilon^2 \Big) p^\Theta(\theta),$$
with the conditional distribution $\bar{r}_{1:n}|\theta;\ell \sim \mathcal{N}(\bar{r}_{1:n}; B^\ell \theta, \epsilon^2 I)$.

Thus, $(\theta^\top,\bar{r}_{1:n}^\top)^\top$ are jointly Gaussian under $p^\ell$, i.e.
\begin{equation}
\begin{pmatrix} \theta \\ \bar{r}_{1:n} \end{pmatrix} \sim \mathcal{N}\Bigg(0, \begin{pmatrix} \sigma^2I_{d_{\Theta}} & \sigma^2{B^{\ell}}^\top \\ \sigma^2 B^\ell & \sigma^2 B^\ell {B^\ell}^\top + \epsilon^2I_n\end{pmatrix}\Bigg),
\label{eqn:jointgaussian:apdproofpsrl}
\end{equation}
and by standard multivariate Gaussian conjugacy result, the posterior is of the form:
$$\theta|\bar{r}_{1:n} \sim \mathcal{N}(\sigma^2 {B^\ell}^\top(\sigma^2 B^\ell {B^\ell}^\top + \epsilon^2 I_n)^{-1} \bar{r}_{1:n}, \sigma^2 I_{d_{\Theta}} - \sigma^4 {B^\ell}^\top(\sigma^2 B^\ell {B^\ell}^\top + \epsilon^2 I_n)^{-1} B^\ell).$$

By construction, $$\int_{E \cap E^\ell} p(\theta,\bar{r}_{1:n}|s_{1:n},a_{1:n}) \mathrm{d}\theta = \int_{E \cap E^\ell} p^\ell(\theta,\bar{r}_{1:n}) \mathrm{d} \theta = p^\ell(\bar{r}_{1:n}) p^\ell(\theta \in E \cap E^\ell|\bar{r}_{1:n}),$$
and likewise,
$$\int_{E^\ell} p(\theta,\bar{r}_{1:n}|s_{1:n},a_{1:n}) \mathrm{d}\theta = p^\ell(\bar{r}_{1:n}) p^\ell(\theta \in E^\ell|\bar{r}_{1:n}).$$

While the marginal $p^\ell(\bar{r}_{1:n})$ can be read off from the joint Gaussian distribution in Equation \ref{eqn:jointgaussian:apdproofpsrl},   $p^\ell(\theta \in E^\ell|\bar{r}_{1:n})$ can be evaluated by observing that it is simply a multivariate Gaussian cumulative distribution function, which can be approximated by suitable Monte-Carlo methods. The same argument holds for $p^\ell(\theta \in E  \cap E^\ell|\bar{r}_{1:n})$ for simple $E$ where $E \cap E^\ell \neq \emptyset$, e.g.,
$E=\ \{\theta \in \Theta \mid\theta_{\nu(s,a)} \leq\theta_{\nu(s,\mu(s))}  \,\, \forall \, s \in \mathcal{S}, a \in \mathcal{A}_s\}$ for computing the probability that a policy $\mu:\mathcal{S} \rightarrow \mathcal{A}$ is optimal. Hence, we can now evaluate $\int_{E} p(\theta,\bar{r}_{1:n}|s_{1:n},a_{1:n})\mathrm{d} \theta$ and $ p(\bar{r}_{1:n}|s_{1:n},a_{1:n})$ , and hence, $\hat{p}_\epsilon(\theta \in E|\mathcal{D}_\tau)$. 

Thus, the overall form of the posterior of interest is:
\begin{equation}
\hat{p}_\epsilon(\theta \in E|\mathcal{D}_\tau) = \frac{\sum_{\ell \in \mathcal{L}^{\mathcal{D}_\tau}} p^\ell(\bar{r}_{1:n}) p^\ell(\theta \in E \cap E^\ell|\bar{r}_{1:n})}{\sum_{\ell \in \mathcal{L}^{\mathcal{D}_\tau}} p^\ell(\bar{r}_{1:n}) p^\ell(\theta \in E^\ell|\bar{r}_{1:n})}.
\end{equation}
This completes the proof.
 \end{proof}

 \subsection{Likelihood Unidentifiability for MDPs which contain non-goal recurrent states}
 \subsubsection{Proof of Lemma \ref{lem:loopthm}}\label{apd:nondecaylkhproof}
We first present some supporting results and definitions.
 \begin{lemma}
 Suppose Assumption \ref{ass:boeunique1} holds. Non-goal recurrent states exist if and only if improper policies exist.
\label{lem:nongoalrecurstate}
 \end{lemma}
 \begin{proof}
{\it (Forward implication.)} Suppose a non-goal recurrent state exists under a deterministic policy $\mu$. In other words, $p^{\mu}(S_t=s^r \text{ for infinitely many } t|S_0=s^r)=1$. For finite state space Markov chains \citep{markovchain}, $s^r$ must belong to some {\it closed set $\mathcal{C}$ of communicating states} that does not include the absorbing state $s^g$. Also, closure implies the chain cannot exit $\mathcal{C}$ once it enters it. Thus, the policy $\mu$ cannot satisfy the definition of a proper policy (when starting in $\mathcal{C}$).

{\it (Reverse implication.)} Conversely, suppose every state $s \in \mathcal{S} \setminus \{s^g\}$ is transient (i.e., non-recurrent) under any deterministic policy $\mu$, that is $p^{\mu}(S_t=s \text{ for infinitely many } t|S_0=s)=0$ \citep{markovchain}.
Consequently, from any starting state, the chain will eventually reach the (absorbing) goal state with probability $1$ \citep{markovchain}, implying that all policies are proper.
\end{proof}

For each (time index) $t \in \mathbb{Z}_{\geq 1}$, let $\tilde{\mu}_t:\mathcal{S} \rightarrow \mathcal{A}$ be a decision rule. Define $\tilde{\mu}: \mathbb{Z}_{\geq 1} \times \mathcal{S} \rightarrow \mathcal{A}$ to be $\tilde{\mu}(t,s)=\tilde{\mu}_t(s)$. Let $\widetilde{M}$ denote the set of all such {\it time-dependent}, or non-stationary, policies $\tilde{\mu}$; and for any $\tau \in \mathbb{Z}_{\geq 1}$, let
$$
p^{\tilde{\mu}}(s_{1:\tau}, a_{1:\tau}|S_0=s_0, A_0=a_0)= \prod_{t=1}^{\tau}[p^S(s_t|s_{t-1},a_{t-1}) \mathbbm{1}(a_t \in \tilde{\mu}(t,s_t))].$$
Various marginal and conditional probabilities of $p^{\tilde{\mu}}$ are also needed in the proof.

Let the set of states $s \in \mathcal{S}$ that can lead to $s^r$ be denoted $\mathcal{C}^r$, that is $s\in \mathcal{C}^r $ if and only if there exists $a \in \mathcal{A}_s$ and $\tilde{\mu} \in \widetilde{M}$ such that 
$p^{\tilde{\mu}}(S_t=s^r \text{ for some } t | S_0=s, A_0=a) >0$.

Recall the definition of $u \in [0,1]^{d_{\Theta}}$ in Equation \ref{eqn:maxprob_u}: $$u_i=\max_{\tilde{\mu} \in \widetilde{M}} p^{\tilde{\mu}}(S_t=s^r \text{ for some } t > 0 |(S_0,A_0) = \nu^{-1}(i)).$$
Extend the dimension of $u$ by defining $u_{d_\Theta+1} \equiv 0$. Finally, let
$$\mathcal{O} = \Big\{\theta \in \Theta \,\Big| \,\argmax_{a^\prime \in \mathcal{A}_s} \theta_{\nu(s,a^\prime)} \cap \argmax_{a^\prime \in \mathcal{A}_s} u_{\nu(s,a^\prime)} \neq \emptyset, \, s \in \mathcal{C}^r\Big\}.$$ For $\theta \in \mathcal{O}$, a policy that acts greedily according to $\theta$ is  $\mu^\theta(s) \in \argmax_{a^\prime \in \mathcal{A}_s}\theta_{\nu(s,a^\prime)}$; note that there is more than one greedy policy corresponding to $\theta$ if the maximising actions are not unique. Similarly, a policy that acts greedily according to $u$ is $\mu^{u}(s) \in \argmax_{a^\prime \in \mathcal{A}_s}u_{\nu(s,a^\prime)}$. The definition of $\mathcal{O}$ implies that, for any $\theta$, at least one of its greedy policies  $\mu^\theta$ is identical to some greedy policy $\mu^u$ of $u$.

For any $\theta \in \mathcal{O}$ and $c>0$, $\theta + cu$ satisfies the requirement to belong to $\mathcal{O}$.

\begin{lemma}
For any $\theta \in \mathcal{O}$, $s \in \mathcal{S}$ and $c\geq 0$,
\begin{equation}
\max_{a \in \mathcal{A}_s} (\theta + cu)_{\nu(s,a)} = \max_{a \in \mathcal{A}_s} \theta_{\nu(s,a)} + c \max_{a \in \mathcal{A}_s} u_{\nu(s,a)} \label{eqn:lem:step1}.
\end{equation}
\label{lem:theta_plus_c_u}
\end{lemma}
\begin{proof}
{\it (Case 1.)} For $s \in \mathcal{C}^r$, let $a \in \argmax_{a^\prime \in \mathcal{A}_s} \theta_{\nu(s,a^\prime)} \cap \argmax_{a^\prime \in \mathcal{A}_s} u_{\nu(s,a^\prime)}$.

For any other $a^\prime \in \mathcal{A}_s$, where $a^\prime \neq a$, we have 
$$u_{\nu(s,a)} - u_{\nu(s,a^\prime)} \geq 0 \geq \theta_{\nu(s,a^\prime)} - \theta_{\nu(s,a)} \implies \,\, c( u_{\nu(s,a)} - u_{\nu(s,a^\prime)}) \geq \theta_{\nu(s,a^\prime)} - \theta_{\nu(s,a)},
$$ for  $c \geq 0$. This implies 
$$(\theta+cu)_{\nu(s,a)} - (\theta + cu)_{\nu(s,a^\prime)}  \geq 0,$$
and therefore, $a$ maximises $(\theta + cu)_{\nu(s,a)}$ and 
$$\max_{a^\prime \in \mathcal{A}_s} (\theta + cu)_{\nu(s,a^\prime)} = (\theta + cu)_{\nu(s,a)}  = \max_{a^\prime \in \mathcal{A}_s} \theta_{\nu(s,a^\prime)} + c \max_{a^\prime \in \mathcal{A}_s} u_{\nu(s,a^\prime)}.$$
Hence, Equation \ref{eqn:lem:step1} holds.

{\it (Case 2.)} For $s \notin \mathcal{C}^r$, by definition, we have $u_{\nu(s,a)} = 0$  for all $a \in \mathcal{A}_s$. Hence, 
Equation \ref{eqn:lem:step1} also holds.
\end{proof}

\begin{proof}[of Lemma \ref{lem:loopthm}]
We verify the given choices for $u$ and the subset of $\theta$ lead to the stated invariance of the likelihood $\hat{L}_{\epsilon}$. Note that, while it is possible to further tailor these definitions to the specific data set, the definitions presented below are applicable to all possible data sets $\mathcal{D}_\tau$ for simplicity.\newline

The likelihood function is
$$ \hat{L}_\epsilon(\theta|\mathcal{D}_\tau) = \prod_{(s,a) \in \mathcal{D}_\tau^{\mathcal{S},\mathcal{A}}} \mathcal{N}\Big(\bar{r}_{s,a};\theta_{\nu(s,a)} - \sum_{s^\prime \in \mathcal{S}} p^S(s^\prime|s,a) \max_{a^\prime \in \mathcal{A}_{s^\prime}} \theta_{\nu(s^\prime,a^\prime)}, \epsilon^2\Big).$$
Below, we show the following result: for any $c \geq 0$, $\theta \in \mathcal{O}$, 
\begin{equation}
(\theta + cu)_{\nu(s,a)} - \sum_{s^\prime \in \mathcal{S}} p^S(s^\prime|s,a) \max_{a^\prime \in \mathcal{A}_{s^\prime}} (\theta + cu)_{\nu(s^\prime, a^\prime)} = \theta_{\nu(s,a)} - \sum_{s^\prime \in \mathcal{S}} p^S(s^\prime|s,a) \max_{a^\prime \in \mathcal{A}_{s^\prime}} \theta_{\nu(s^\prime, a^\prime)}.
\label{eqn:lemma:loopthm_alg_2}
\end{equation}
Therefore, 
\begin{align*}
    \hat{L}_\epsilon(\theta|\mathcal{D}_\tau) =& \prod_{(s,a) \in \mathcal{D}_\tau^{\mathcal{S},\mathcal{A}}} \mathcal{N}\Big(\bar{r}_{s,a};(\theta + cu)_{\nu(s,a)} - \sum_{s^\prime \in \mathcal{S}} p^S(s^\prime|s,a) \max_{a^\prime \in \mathcal{A}_{s^\prime}} (\theta + cu)_{\nu(s^\prime, a^\prime)}, \epsilon^2 \Big) \\
    =& \hat{L}_\epsilon(\theta+cu|\mathcal{D}_\tau).
\end{align*}
\newline By Lemma \ref{lem:theta_plus_c_u}, showing Equation \ref{eqn:lemma:loopthm_alg_2} is equivalent to showing 
$$u_{\nu(s,a)} = \sum_{s^\prime \in \mathcal{S}} p^S(s^\prime|s,a) \max_{a^\prime \in \mathcal{A}_{s^\prime}} u_{\nu(s^\prime, a^\prime)}$$ for all $s \in \mathcal{S}, a \in \mathcal{A}_s$. To show this new identity, note that 
\begin{align*}
u_{\nu(s,a)} =& \max_{\tilde{\mu} \in \widetilde{M}} p^{\tilde{\mu}} (S_t=s^r \text{ for some } t \geq 1 | S_0=s, A_0=a) \\
=& \max_{\tilde{\mu} \in \widetilde{M}} \sum_{s^\prime \in \mathcal{S}} p^{\tilde{\mu}}(S_t=s^r \text{ for some } t \geq 1 | S_1=s^\prime, A_1 = \tilde{\mu}_1(s^\prime)) p^S(s^\prime|s,a) \\
=& \sum_{s^\prime \in \mathcal{S}} \max_{a^\prime \in \mathcal{A}_{s^\prime}} \max_{\substack{\tilde{\mu} \in \widetilde{M}\\ \tilde{\mu}_1(s^\prime) = a^\prime}} p^{\tilde{\mu}}(S_t=s^r \text{ for some } t \geq 1 | S_1=s^\prime, A_1 = a^\prime) p^S(s^\prime|s,a) \\
=& \sum_{s^\prime \in \mathcal{S}} \max_{a^\prime \in \mathcal{A}_{s^\prime}} \max_{\substack{\tilde{\mu} \in \widetilde{M}\\ \tilde{\mu}_1(s^\prime) = a^\prime}} \big(\mathbbm{1}(s^r = s^\prime) + (1-\mathbbm{1}(s^r = s^\prime))\\
&\quad \times p^{\tilde{\mu}}(S_t=s^r \text{ for some } t \geq 2 | S_1=s^\prime, A_1 =a^\prime)\big) p^S(s^\prime|s,a) \\
=& \sum_{s^\prime \in \mathcal{S}} \Big(\mathbbm{1}(s^r = s^\prime) + (1-\mathbbm{1}(s^r = s^\prime)) \max_{a^\prime \in \mathcal{A}_{s^\prime}} u_{\nu(s^\prime, a^\prime)}\Big) p^S(s^\prime|s,a).
\end{align*}
Since $\mathbbm{1}(s^r = s^\prime) \max\limits_{a^\prime \in \mathcal{A}_{s^\prime}} u_{\nu(s^\prime, a^\prime)} = \begin{cases} 1& \text{if } s^\prime = s^r \\ 0 & \text{otherwise} \end{cases}$,  we have the result.
\end{proof}

\subsubsection{Proof of Lemma \ref{lem:propernou}}\label{apd:lem:propernou}

\begin{proof}
    Recall that 
    $$\hat{L}_\epsilon(\theta|\mathcal{D}^{\mathrm{full}}) = \prod_{(s,a,\bar{r}_{s,a}) \in \mathcal{D}_\tau^{\mathrm{full}}} \mathcal{N}\Big(\bar{r}_{s,a};\theta_{\nu(s,a)} - \sum_{s^\prime \in \mathcal{S}} p^S(s^\prime|s,a) \max_{a^\prime \in \mathcal{A}_{s^\prime}} \theta_{\nu(s^\prime,a^\prime)}, \epsilon^2\Big).$$

    For any $s \in \mathcal{S}$, let 
    $$\mu^\theta(s) \in \Big\{a \in \mathcal{A}_s \, \Big| \, a \in \argmax_{a^\prime \in \mathcal{A}_{s}}Q_\theta(s,a^\prime) \text{ and } u_{\nu(s,a)} \geq u_{\nu(s,\bar{a})} \, \forall \, \bar{a} \in \argmax_{a^\prime \in \mathcal{A}_s} Q_\theta(s,a^\prime)\Big\}.$$
    Therefore, $\mu^\theta$ is a greedy policy induced by $\theta$.

    Let $u_{\nu(s^g,a^g)}:= 0$. Suppose $\hat{L}_\epsilon(\theta |\mathcal{D}^{\mathrm{full}}) = \hat{L}_\epsilon(\theta+cu|\mathcal{D}^{\mathrm{full}})$ for all $c>0$.  Then, for any $s \in \mathcal{S} \setminus \{s^g\}$ and $a \in \mathcal{A}_s$, the invariance implies that

    \begin{align}
    Q_\theta(s,a) - \sum_{s^\prime \in \mathcal{S} \setminus \{s^g\}} &p^S(s^\prime|s,a) \max_{a^\prime \in \mathcal{A}_{s^\prime}} Q_\theta(s^\prime,a^\prime) \nonumber \\
    =& Q_\theta(s,a) + cu_{\nu(s,a)} -\sum_{s^\prime \in \mathcal{S} \setminus \{s^g\}} p^S(s^\prime|s,a) \max_{a^\prime \in \mathcal{A}_{s^\prime}} (Q_\theta(s^\prime,a^\prime)+cu_{\nu(s^\prime,a^\prime)}). \label{eqn:lem_llhinv}
\end{align}

Pick a sufficiently small $c>0$, such that for any $s \in \mathcal{S} \setminus \{ s^g\}$, $\max_{a \in \mathcal{A}_{s}} (Q_\theta(s,a)+cu_{\nu(s,a)}) = Q_\theta(s,\mu^\theta(s)) + cu_{\nu(s,\mu^\theta(s))}$. Then, rearranging Equation \ref{eqn:lem_llhinv} yields
\begin{equation}
u_{\nu(s,a)} = \sum_{s^\prime \in \mathcal{S} \setminus \{s^g\}} p^S(s^\prime|s,a) u_{\nu(s^\prime,\mu^\theta(s^\prime))}, \quad a \in \mathcal{A}_s, \label{eqn:lem:propernou_eq1}
\end{equation}
and furthermore,
\begin{equation}
u_{\nu(s,\mu^\theta(s))} = \sum_{s^\prime \in \mathcal{S} \setminus \{s^g\}} p^S(s^\prime|s,\mu^\theta(s)) u_{\nu(s^\prime,\mu^\theta(s^\prime))}. \label{eqn:lem:propernou_eq2}
\end{equation}

Let $P \in \mathbb{R}^{|\mathcal{S}| \times |\mathcal{S}|}$ and $\tilde{u} \in \mathbb{R}^{|\mathcal{S}|}$ such that, for a state space labelled as $\mathcal{S}=\{s^1,\dots,s^K\}$ with $s^K=s^g$ and $K=|\mathcal{S}|$, $P_{ij}:=p^S(s^j|s^i,\mu^\theta(s^i))$ and $\tilde{u}_i:=u_{\nu(s^i,\mu^\theta(s^i))}$ for $i,j \in \{1,\dots,K\}$. By Equation \ref{eqn:lem:propernou_eq2}, we have
$$\tilde{u} = P \tilde{u}=P^n\tilde{u},\quad n \in \mathbb{Z}_{\geq 0}.$$
Since $\mu^\theta$ is proper, $\lim_{n \rightarrow \infty}P^n_{ij}=\mathbbm{1}(j=K)$. Thus, $u_{\nu(s,\mu^\theta(s))}=0$ for all $s \in \mathcal{S} \setminus \{s^g\}$. This, together with Equation \ref{eqn:lem:propernou_eq1}, implies that $u_{\nu(s,a)} = 0$ for all $s \in \mathcal{S} \setminus \{s^g\}$, $a \in \mathcal{A}_s$.
\end{proof}

 \subsubsection{Theorem \ref{thm:unboundedlh}}\label{apd:unboundedlhproof}
 
The idea of the proof of $(i)$ is to show that there exists a hypercube in $\mathcal{O}^\phi$, and we can translate the hypercube in the direction of $u$ so that the translation remains in $\mathcal{O}^\phi$. Subsequently, we can construct an infinite number of disjoint hypercubes so that the approximated likelihood function shares the same positive lower bound within each of them. On the other hand, to prove $(ii)$, the idea is to rewrite the integral of the approximated likelihood as a Gaussian integral with respect to $\theta$ when $\theta \in \mathcal{O}^\mu$, and show that the covariance matrix of the Gaussian integral is positive definite if $\mu$ is proper. Note that the same proof can be tailored to show that if $\mu$ is improper, the approximated likelihood is unbounded. However, in $(i)$, we use a slightly different but more general approach that works for a broader class of approximation kernels in the construction of the approximated likelihood.

\begin{proof}[of Theorem \ref{thm:unboundedlh} $(i)$]    
Firstly, pick $\bar{\theta}\in \mathcal{O}^\phi$ such that $\bar{\theta}_{\nu(s,a)} \neq \bar{\theta}_{\nu(s,a^\prime)}$ $\forall s \in \mathcal{S}$, $\forall a,a^\prime \in \mathcal{A}_s$ where $a \neq a^\prime$. This is possible because $\Theta = \mathbb{R}^{d_\Theta}$.

Next, let $$2c:=\min\limits_{\substack{s \in \mathcal{S}\\ a,a^\prime \in \mathcal{A}_s\\ a \neq a^\prime}} \bar{\theta}_{\nu(s,a)} - \bar{\theta}_{\nu(s,a^\prime)} > 0$$
and define the hypercube
$$B_c^{\bar{\theta}} := \{\theta \in \Theta \mid \bar{\theta}_{\nu(s,a)} - c \leq \theta_{\nu(s,a)} \leq \bar{\theta}_{\nu(s,a)} \,\, \forall s\in \mathcal{S}, a \in \mathcal{A}_s\}.$$
    
We now show that $B_c^{\bar{\theta}+ku} \subseteq \mathcal{O}^\phi$ for any $k \geq 0$. Let $\theta \in B_c^{\bar{\theta}+ku}$ for some $k \geq 0$. Then, for any $s \in \mathcal{S}$, $a \in \mathcal{A}_s$,
        $$\bar{\theta}_{\nu(s,a)} + ku_{\nu(s,a)} - c  \leq  \theta_{\nu(s,a)} \leq  \bar{\theta}_{\nu(s,a)} + ku_{\nu(s,a)}.$$
Hence, by the definition of $B_c^{\bar{\theta}+ku}$ and $\phi$, 
$$\theta_{\nu(s,\phi(s))} - \theta_{\nu(s,a)} \geq (\bar{\theta}_{\nu(s,\phi(s))} - \bar{\theta}_{\nu(s,a)}) + k (u_{\nu(s,\phi(s))} - u_{\nu(s,a)}) - c \geq 2c + 0 - c > 0.$$
Thus, $\theta \in \mathcal{O}^\phi$.

Next, as $B_c^{\bar{\theta}+ku}$ is compact and $\hat{L}_\epsilon(\cdot|\mathcal{D}_\tau)$ is continuous, it follows from the extreme value theorem that $\hat{L}_\epsilon$ is bounded within  $B_c^{\bar{\theta}+ku}$ below from $0$ for any $k \geq 0$. Let $\ell^*=\inf_{\theta \in B_c^{\bar{\theta}}} \hat{L}_\epsilon(\theta|\mathcal{D}_\tau) > 0$. 

We now show two properties of $B_c^{\bar{\theta}+ku}$: (1) $\forall j,j^\prime \in \mathbb{Z}_{\geq 0}$ such that $j\neq j^\prime$, $B_c^{\bar{\theta}+(2jc)u} \cap B_c^{\bar{\theta}+(2j^\prime c)u} = \emptyset$; (2) $\min_{\theta \in B_c^{\bar{\theta} + ku}} \hat{L}_\epsilon(\theta|\mathcal{D}_\tau) = \ell^*$ $\forall k \geq 0$.

To show (1), if $\theta \in B_c^{\bar{\theta} + 2jcu}$ and $\theta^\prime \in B_c^{\bar{\theta} + 2j^\prime c u}$  and w.l.o.g. $j < j^\prime$, since $u_{\nu(s^r,\phi(s^r))} = 1$, we have
\begin{align*}
\theta_{\nu(s^r,\phi(s^r))} &\leq \bar{\theta}_{\nu(s^r,\phi(s^r))} + 2jcu_{\nu(s^r,\phi(s^r))} \\
&=  \bar{\theta}_{\nu(s^r,\phi(s^r))}+2jc  \\
&< \bar{\theta}_{\nu(s^r,\phi(s^r))} +(2(j+1))cu_{\nu(s^r,\phi(s^r))} - c \\
&\leq \theta^\prime_{\nu(s^r,\phi(s^r))}.
\end{align*}
Therefore $ B_c^{\bar{\theta} + 2jcu} \cap  B_c^{\bar{\theta} + 2j^\prime cu}=\emptyset$.

To show (2), if $\theta \in B_c^{\bar{\theta}}  \subseteq \mathcal{O}^\phi$, then $\theta + ku \in B_c^{\bar{\theta}+ku}$ and $\hat{L}_\epsilon(\theta + ku|\mathcal{D}_\tau) = \hat{L}_\epsilon(\theta|\mathcal{D}_\tau)$ by Lemma \ref{lem:loopthm}, implying that $\min_{\theta \in B_c^{\bar{\theta}+ku}} \hat{L}_\epsilon(\theta|\mathcal{D}_\tau) \leq \min_{\theta \in B^{\bar{\theta}}} \hat{L}_\epsilon(\theta|\mathcal{D}_\tau)$. Conversely, if $\theta^\prime \in B_c^{\bar{\theta}+ku}$, then $\theta^\prime - ku \in B^{\bar{\theta}}_c \subseteq \mathcal{O}^\phi$ and $\hat{L}_\epsilon(\theta^\prime - ku|\mathcal{D}_\tau) = \hat{L}_\epsilon(\theta^\prime|\mathcal{D}_\tau)$, concluding that $\min_{\theta \in B_c^{\bar{\theta}+ku}} \hat{L}_\epsilon(\theta|\mathcal{D}_\tau) \geq \min_{\theta \in B^{\bar{\theta}}} \hat{L}_\epsilon(\theta|\mathcal{D}_\tau)$. Therefore $\min_{\theta \in B_c^{\bar{\theta} + ku}} \hat{L}_\epsilon(\theta|\mathcal{D}_\tau) = \ell^*$.

Finally, it is clear that $B_c^{\bar{\theta}+2ju} \subseteq \mathcal{O}^\phi$ for any $j \in \mathbb{Z}_{\geq 0}$. Therefore, for any $J \in \mathbb{Z}_{\geq 0}$,
$$\int_{\mathcal{O}^\phi} \hat{L}_\epsilon(\theta|\mathcal{D}_\tau)\mathrm{d} \theta \geq \int_{\bigcup_{j=0}^J B_c^{\bar{\theta} + 2ju}} \hat{L}_\epsilon(\theta|\mathcal{D}_\tau) \mathrm{d} \theta \geq J\ell^*$$
by property (1) and (2). Thus, the RHS diverges as $J \rightarrow \infty$.
\end{proof}

\begin{proof}[of Theorem \ref{thm:unboundedlh} $(ii)$]
Given that $\theta \in \mathcal{O}^\mu$, we rewrite the likelihood function in ``standard'' Gaussian form as follows:
\begin{align*}
    &\hat{L}_\epsilon(\theta|\mathcal{D}^{\mathrm{full}}) \\
    &= \prod_{(s,a) \in \mathcal{D}_\tau^{\mathcal{S},\mathcal{A}}} \mathcal{N}\Bigg(\bar{r}_{s,a};\theta_{\nu(s,a)}  - \sum_{s^\prime \in \mathcal{S} \setminus \{s^g\}} p^S(s^\prime|s,a) \max_{a^\prime \in \mathcal{A}_{s^\prime}} \theta_{\nu(s^\prime, a^\prime)}, \epsilon^2\Bigg)\\
    &= \prod_{(s,a) \in \mathcal{D}_\tau^{\mathcal{S},\mathcal{A}}} \frac{1}{\sqrt{2\pi\epsilon^2}}\exp\Bigg(-\frac{1}{2\epsilon^2}\Bigg(\bar{r}_{s,a} - \Big(\theta_{\nu(s,a)} - \sum_{s^\prime \in \mathcal{S} \setminus \{s^g\}} p^S(s^\prime|s,a) \max_{a^\prime \in \mathcal{A}_{s^\prime}} \theta_{\nu(s^\prime, a^\prime)}\Big)\Bigg)^2\Bigg)\\
    &= \beta  \exp\Bigg(-\frac{1}{2\epsilon^2} (\theta-\theta_0)^\top A(\theta-\theta_0)\Bigg) 
\end{align*}
for some constant $\theta_0 \in \mathbb{R}^{d_\Theta}$, $\beta>0$, and $A=\sum_{j=1}^{d_\Theta} b_j b_j^\top \in \mathbb{R}^{d_\Theta \times d_\Theta}$. Here, $b_j \in \mathbb{R}^{d_\Theta}$ and for any $i \in \{1,\dots,d_\Theta\}$, satisfying
$$(b_j)_i = \mathbbm{1}(j=i) - \sum_{s^\prime \in \mathcal{S}\setminus \{s^g\}} \mathbbm{1} (\nu(s^\prime, \mu(s^\prime)) = i) p^S(s^\prime|\nu^{-1}(j)).$$

Next, we show that the matrix $A$ is positive definite. Define $w_i \in \mathbb{R}^{d_\Theta}$ such that $(w_{i})_j=(b_j)_i$ for all $i,j \in \{1,\dots,d_\Theta\}$. Let $\alpha \in \mathbb{R}^{d_\Theta}$. Then,
\begin{align*}
    \sum_{i=1}^{d_\Theta}\alpha_i w_i = 0 &\iff \sum_{i=1}^{d_\Theta} \alpha_i (b_{j})_i = 0 \quad \forall \, j\in \{1,\dots,d_{\Theta}\}\\
    &\iff \sum_{i=1}^{d_\Theta} \alpha_i \Bigg(\mathbbm{1}(j=i) - \sum_{s^\prime \in \mathcal{S} \setminus \{s^g\}} \mathbbm{1} (\nu(s^\prime, \mu(s^\prime)) = i) p^S(s^\prime|\nu^{-1}(j))\Bigg) = 0 \quad \forall \, j\\
    &\iff \alpha_j -\sum_{s^\prime \in \mathcal{S} \setminus \{s^g\}} \alpha_{\nu(s^\prime,\mu(s^\prime))} p^S(s^\prime|\nu^{-1}(j))=0 \quad \forall \, j.
\end{align*}
By setting $Q(s,a):= \alpha_{\nu(s,a)}$ in Lemma \ref{lem:bellmanfirstpart}, it follows that $\alpha=0$  is the unique solution to $\sum_{i=1}^{d_\Theta}\alpha_i w_i = 0$. Therefore, $\{w_i\}_{i=1}^{d_\Theta}$ is linearly independent, and so is  $\{b_j\}_{j=1}^{d_\Theta}$  (to see this, a $d_\Theta \times d_\Theta$ matrix with row vectors $\{w_i\}_{i=1}^{d_\Theta}$ is invertible if and only if its column vectors span $\mathbb{R}^{d_\Theta}$.). Thus, for any $x \in \mathbb{R}^{d_\Theta}$, $x^\top Ax=\sum_{j=1}^{d_\Theta} x^\top b_j b_j^\top x = \sum_{j=1}^{d_\Theta} (b_j^\top x)^2 \geq 0$. In addition, because $\{b_{j}\}_{j=1}^{d_\Theta}$ is linearly independent, if $x \neq 0$, $\exists \, j \in \{1,\dots,d_\Theta\}$ such that $b_j^\top x \neq 0$, implying that $x^\top Ax > 0$. Thus, $A$ is positive definite.

Therefore,
$$\int_{\mathcal{O}^\mu} \hat{L}_\epsilon(\theta|\mathcal{D}^{\mathrm{full}})\mathrm{d} \theta \leq \beta \int_{\mathbb{R}^{d_\Theta}} \exp\Bigg(-\frac{1}{2\epsilon^2} (\theta-\theta_0)^\top A (\theta-\theta_0)\Bigg)\mathrm{d} \theta = \beta \sqrt{ \frac{(2\pi)^{d_\Theta}}{\det{A}}} < \infty.$$
This completes the proof.

(To remark, the proof can be straightforwardly extended to any data sets $\mathcal{D}_\tau \supseteq \mathcal{D}^{\mathrm{full}}$, by noting that the resulting matrix $A$ can be decomposed as a sum of the positive definite $\sum_{j=1}^{d_\Theta} b_j b_j^\top$ and a positive semi-definite matrix.)
\end{proof}

\section{Further details and supplementary numerical examples}\label{apd:sec:numerical}
\subsection{Details for the MDP environments and computation}\,\\

\textbf{\texttt{deep sea} and its variants:} Following the convention in \citet{bsuite, bootstraposband}, to prevent any systematic bias in the algorithms, the two actions of each state are randomly assigned \textit{left} and \textit{right} at the beginning of each run and fixed for all episodes. In all our experiments, the rewards are deterministic and are based on the corresponding states and the directions \textit{left} and \textit{right} selected, and thus are independent of the randomised action mapping. For numerical studies with stochastic state transitions (see Appendix \ref{apd:sec:additionalnumerial}), following \citet{bsuite}, the reward depends only on the current state and the selected action and is independent of the next state. 

For \texttt{deep sea swirl}, the set-up can potentially be generalised to a depth grid as follows: for every third row beginning from the second row from the top, redirect two (or multiples of two) of the connections to connect to their neighbouring states on the same row instead of the states below; in addition, augment the connections on the next corresponding rows to ensure that each state admits two actions (for convenient implementation) and that all states remain connected.

For any admissible deterministic stationary policy $\mu$ of a given MDP, let
$$\mathcal{E}^\mu:= \{\mu^\prime :\mathcal{S} \rightarrow \mathcal{A} \mid \mu^\prime \text{ is admissible and deterministic stationary and } p^\mu \equiv p^{\mu^\prime}\},$$ be the collection of policies, i.e., an equivalence class, that induces the same joint probability distribution over the trajectories generated during an episode of the MDP as $\mu$. For MDPs with deterministic transitions and a degenerate initial state distribution, it is sufficient to compute the posterior probability only over the equivalence classes because all policies within each class admit identical routes. Thus, in Figure \ref{fig:deepsea_policy_cvg} and similar figures in Appendix \ref{apd:sec:additionalnumerial} that plot the path optimal probabilities, each line represents the path optimal probability of selecting any policy within a particular class.

\textbf{\textit{Bayes-BR}:} 
We evaluate the posterior distribution characterising the optimality of the equivalence classes of (deterministic and stationary) policies using Equation \ref{eq:ts_mdp}. To do this, for each class, we compute $\hat{p}_\epsilon(\theta \in E|\mathcal{D}_{\tau})$ in Equation \ref{eqn:gaussianintegralformula} by setting $E:=\bigcup_{\mu^\prime \in \mathcal{E}^\mu}\{\theta \in \Theta \mid \theta_{\nu(s,\mu^\prime(s))}\geq \theta_{\nu(s,a)}, s\in \mathcal{S}, a \in \mathcal{A}_s\}$, and iterate over all classes. In practice, this integral can be simplified further by applying a linear transformation to the variable $\theta$ so that the inequality constraints become rectangular boundaries. Similarly, to find the marginal optimal probability of an action $a \in \mathcal{A}_s$ of state $s \in \mathcal{S} \setminus \{s^g\}$, set $E:=\{\theta \in \Theta \mid \theta_{\nu(s,a)}\geq \theta_{\nu(s,a^\prime)}, s\in \mathcal{S}, a^\prime \in \mathcal{A}_s\}$, followed by a linear transformation of the parameter. The denominator can either be found by self-normalisation (by computing the entire probability mass function) or by computing the Gaussian integral over every $E^{\ell}$ and $\ell \in \mathcal{L}^{\mathcal{D}_\tau}$ to obtain the normalising constant directly.

\begin{sloppypar}
To evaluate the Gaussian integrals, for any examples with depth-$3$ or depth-$4$, the conventional numerical estimation of Gaussian cumulative distribution functions is used---\texttt{scipy.stats.multivariate\_normal.cdf} (version 1.16.3 in Python)---with default accuracy settings. For experiments with depth-$5$, a Sobol sequence of size $2^{15}$ is used to construct a quasi Monte-Carlo estimate \citep{qmc_caflisch_1998}, following the implementation of \texttt{scipy.stats.qmc.MultivariateNormalQMC} version 1.16.3 in Python. This allows for a more efficient estimation of the Gaussian integrals \citep{scipy}.
\end{sloppypar}

Note that the number of policy classes increases exponentially with the depth of \texttt{deep sea} in deterministic transition settings, and with the number of states in stochastic transition settings. This combinatorial complexity is inherited by \textit{Bayes-BR} due to its reliance on method $(A)$ in Algorithm \ref{alg:onlinedegen1}. It is therefore of particular interest to investigate the application of more scalable Monte Carlo methodologies, such as HMC.

\textbf{\textit{Bayes-BR-HMC}:} 
In each episode, $2000$ iterations of Hamiltonian Monte Carlo are performed; see, e.g., \citet{hmcneal} for the algorithm. The hyper-parameter tuning scheme proceeds as follows: a target acceptance probability is set, for which we use $0.7$ throughout, and the initial step-size is manually calibrated to ensure convergence toward this target during the initial episode; see Table \ref{tab:hmc_stepsize}.

\begin{table}[h]
    \centering
    \caption{HMC step-size used for each tolerance $\epsilon$.}
    \label{tab:hmc_stepsize}
    \begin{tabular}{|c|c|c|c|c|c|c|}
        \hline
         $\epsilon$ &  $0.01$ & $0.1$ & $0.2$ & $0.5$ & $1.0$ & $5.0$\\
        \hline
        step-size & $0.0001$ & $0.005$ & $0.01$ & $0.03$ & $0.05$ & $0.1$   \\
        \hline
    \end{tabular}
\end{table}

Each HMC run begins with a warm-up phase consisting of $1000$ iterations, split into $10$ windows of $100$ iterations each. Each iteration is executed with $10$ leapfrog steps. The average acceptance probability is computed for each warm-up window. If the average acceptance probability exceeds the target, the step-size is increased by a factor of $1.3$ for the next window; conversely, if it is below the target, the step-size is decreased by a factor of $0.7$ for the next window. During the warm-up phase of each episode, once the acceptance probability crosses the target threshold within a warm-up window, the decrease factor is refined and set to $0.8$ and the increase factor to $1.1$, respectively, until the warm-up phase ends.

Following the warm-up, the last step-size that yields an acceptance probability above the target is fixed for the subsequent $1000$ iterations, of which the final sample is used to construct the policy. This step-size is carried forward to initialise the HMC sampler for the posterior of the subsequent episode. Additionally, a diagonal preconditioning matrix is estimated by computing the inverse empirical variance for each dimension of the last $500$ samples, to be used in the HMC sampler for the subsequent episode.

\textbf{\textit{Bayes-TD}-based methods:} As each likelihood of the sequence associated with \textit{Bayes-TD} is simply a linear Gaussian distribution with respect to $\theta$, the posterior is also a Gaussian, where its mean and covariance can be computed via standard Gaussian conjugacy results (see, e.g., Appendix \ref{apd:proofpsrl}). This applies similarly to \textit{Bayes-TD-Max}, where, at episode $i$, the posterior mean $(m^{i-1})$ over $\max_{a^\prime \in \mathcal{A}_{s^\prime}} \theta_{\nu(s^\prime,a^\prime)}$ is approximated by Monte Carlo samples from the posterior obtained in the previous episode. $100000$ and $10000$ Monte Carlo samples are used in the deterministic transition setting and the stochastic transition setting (see Appendix \ref{apd:sec:additionalnumerial}), respectively. For \textit{Bayes-TD-En}, $100$ and $10$ mixtures are used in the deterministic transition setting and the stochastic transition setting, respectively. Through standard conjugacy results, \textit{Bayes-TD-En} has a mixture of Gaussian posterior distribution due to its mixture of Gaussian likelihood, where the mixing weights are proportional to the density of the marginal (Gaussian) distribution of the temporal difference target, i.e., the sum of the reward and the estimated maximum value term.

\subsection{Additional numerical studies}\label{apd:sec:additionalnumerial}

\textbf{More on \textit{HMC}.} We re-run \textit{HMC} for the deterministic depth-$5$ \texttt{deep sea} example for the $\epsilon=0.01$ and $\epsilon =0.1$ cases with the following settings: During a warm-up phase of $10$ windows, the step-size for the next window is reduced by a factor of $0.5$ (instead of $0.7$) before the $70\%$ target acceptance rate is crossed, and by $0.7$ (instead of $0.8$) after it has been crossed, whenever the acceptance probability falls below $70\%$ in a warm-up window. This allows the step-sizes between consecutive episodes to vary more in response to the potentially more abrupt changes in the posterior landscape for these two tolerance values. The results are illustrated in Figure \ref{fig:deepsea_hmc_add}, which, in comparison to Figures \ref{fig:deepsea_policy_cvg} and \ref{fig:deepsea_hmc_acc}, shows no improvements in cumulative regrets despite a higher average acceptance probability for the $\epsilon=0.01$ case.

\textbf{More on improper policies.} We evaluate the performance of the algorithms on the depth-$5$ \texttt{deep sea swirl} environment with deterministic transitions, where improper policies exist. Figure \ref{fig:deepseaswirl_s10_policy_cvg} demonstrates that at a prior variance of $\sigma^2=10^2$ and a tolerance $\epsilon \geq 0.1$, the optimal trajectory has a near-zero probability of being selected. Instead, our more detailed analysis shows that the posterior mass concentrates heavily on the two routes that trap the agent in an infinite loop within the second row. This empirical observation precisely aligns with the theoretical discussions in Section \ref{sec:unidentifiability_improper}. Since the reward for travelling horizontally within this loop is merely $-0.09$ (compared to $-0.03$ and $0.03$ for travelling right-down and left-down, respectively), a high prior variance coupled with a high tolerance allows the improper policy region of the likelihood to dominate (due to the likelihood invariance). Consequently, the posterior places a disproportionately large mass on this region of the parameter space, causing exploration to become trapped within these loops. In contrast, Figure \ref{fig:deepseaswirl_s1_policy_cvg} shows that reducing the prior variance to $\sigma^2=1$ mitigates this pathology for small $\epsilon$. Under a tighter prior, the optimal and near-optimal policies are selected with high probability, while the posterior probabilities of the improper policies diminish comparatively. This demonstrates that the prior effectively restricts the contribution of the likelihood invariance to the posterior landscape, thus reducing the posterior mass assigned to the improper policies. Note that, as shown in Figure \ref{fig:deepseaswirl_policy_cvg}, due to a narrow margin in expected cumulative rewards between the optimal and near-optimal policies (where the minimum difference is merely $0.06$), a tolerance of $\epsilon=0.01$ is insufficient to uniquely isolate the optimal policy.

Another notable observation is that while \textit{Bayes-TD} exhibited the worst performance among the TD-based variants in the standard \texttt{deep sea} environment, the converse is true in the \texttt{deep sea swirl} when $\sigma^2=10^2$ as illustrated in Figures \ref{fig:deepseaswirl_s10_cumreg} and \ref{fig:deepseaswirl_s10_expl_prog}, despite prematurely halting exploration. However, at $\sigma^2=1$, where the effect of improper policies diminishes, the relative performance of the TD-based methods becomes more comparable to our observations from the standard \texttt{deep sea} environment; see Figures \ref{fig:deepseaswirl_s1_cumreg} and \ref{fig:deepseaswirl_s1_expl_prog}.

=\textbf{More on multi-modality.} We evaluate the algorithms on the depth-$5$ \texttt{deep sea pyramid} with deterministic transitions. Figure \ref{fig:deepseapyramid_policy_cvg} illustrates that under \textit{Bayes-BR}, all eight equivalence classes of optimal policies (optimal trajectories) maintain approximately equal marginal posterior probabilities across all tested values of the tolerance $\epsilon$. However, the cumulative regret comparison in Figure \ref{fig:deepseapyramid_cumreg} presents a nuanced exploration-exploitation trade-off. While \textit{Bayes-BR} incurs a lower overall cumulative regret than \textit{Bayes-TD-En} and \textit{Bayes-TD-Max} for small $\epsilon$, the regret curve for \textit{Bayes-TD} flattens earlier than that of \textit{Bayes-BR} when $\epsilon=0.01$. A closer examination of Figures \ref{fig:deepseapyramid_expl_prog} and \ref{fig:deepseapyramid_rank} reveals that \textit{Bayes-TD} rapidly concentrated its posterior mass on a subset of the optimal trajectory, rather than exhaustively exploring the state-action space. Consequently, exploration halts prematurely, and the \textit{Bayes-TD} algorithm fails to discover the remaining optimal trajectories within the $100$ episodes.

\textbf{Stochastic transitions.} Figure \ref{fig:deepsea_sto_cumreg} confirms that our algorithm, \textit{Bayes-BR}, maintains robust empirical performance within the depth-$4$
\texttt{deep sea} environment with stochastic transitions. Specifically, in this environment, an agent has only a $0.729$ probability of successfully reaching the goal state, which yields a reward of $1$, even when consistently selecting the optimal move-right action.

\begin{figure}[h!]
    \centering
    \BeginAccSupp{ActualText={Two subplots showing the running statistics of the HMC rerun for the depth-5 deep sea environment with an adjusted fine-tuning hyperparameter over 100 episodes.}}
    \begin{subfigure}[b]{0.49\textwidth}
        \centering
        \begin{minipage}[c][4cm][c]{\linewidth}
            \centering
            \includegraphics[width=\linewidth, height=3cm, keepaspectratio]{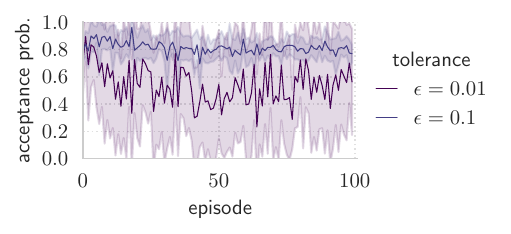}
        \end{minipage}
        \caption{HMC acceptance probability}
        \label{fig:deepsea_hmc_acc_add}
    \end{subfigure}
    \hfill
    \begin{subfigure}[b]{0.49\textwidth}
        \centering
        \begin{minipage}[c][4cm][c]{\linewidth}
            \centering
            \includegraphics[width=\linewidth, height=4cm, keepaspectratio]{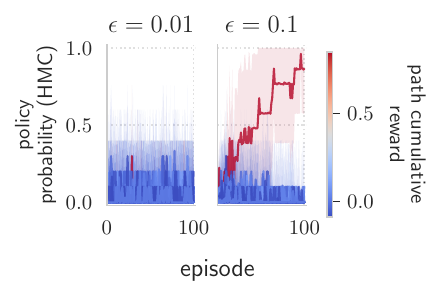}
        \end{minipage}
        \caption{Evolution of the posterior over policies}
        \label{fig:deepsea_convg_hmc_add}
    \end{subfigure}
    \EndAccSupp{}
    \caption{The left figure shows the HMC acceptance rates when the adjusted fine-tuning hyper-parameters are used for the deterministic depth-$5$ \texttt{deep sea} example for $\epsilon=0.01$ and $\epsilon=0.1$. The right plot shows the corresponding evolution of the posterior mass function for the optimal equivalence class of policies. All results are averaged over $10$ independent runs. Shaded areas show one standard deviation computed using these runs.}
    \label{fig:deepsea_hmc_add}
\end{figure}

\begin{figure}[h!]
    \centering
    \BeginAccSupp{ActualText={Two sets (top and bottom) of line graphs over 100 episodes comparing BR and HMC for each of the tolerances 0.01, 0.1, 0.2, 0.5, 1 and 5 in the depth-5 deep sea swirl environment, where the top set is for prior variance 100 and the bottom set is for prior variance 1. Each plot shows the convergence of the posterior probability for the optimal policy as well as the corresponding decay of suboptimal policies.}}
    \begin{subfigure}{1.0\textwidth}
    \centering
    \includegraphics[width=1\linewidth]{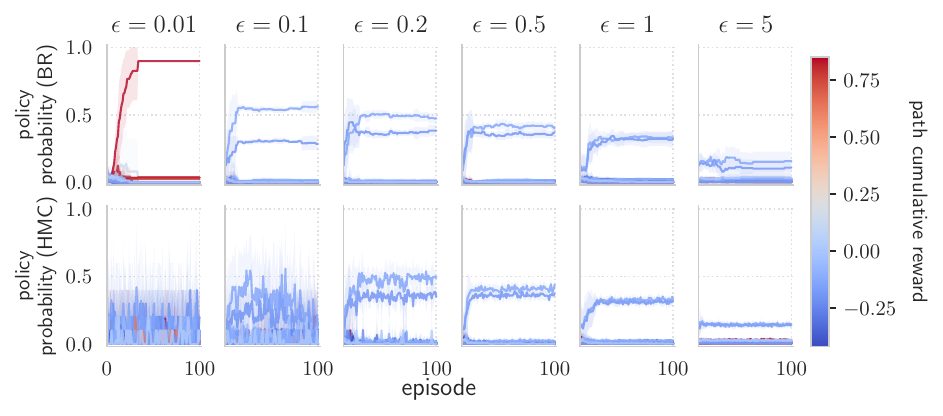}
    \caption{$\sigma^2=10^2$}
    \label{fig:deepseaswirl_s10_policy_cvg}
    \end{subfigure}
    \begin{subfigure}{1.0\textwidth}
    \centering
    \includegraphics[width=1\linewidth]{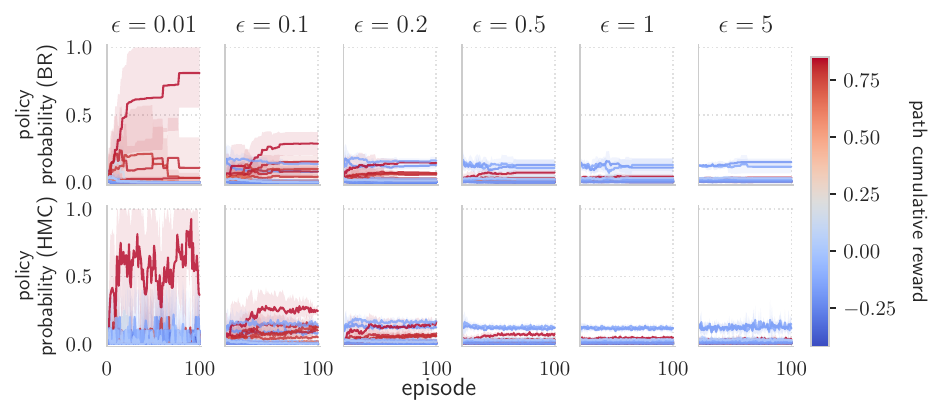}
    \caption{$\sigma^2=1$}
    \label{fig:deepseaswirl_s1_policy_cvg}
    \end{subfigure}
    \EndAccSupp{}
\caption{A comparison between \textit{Bayes-BR} and \textit{Bayes-BR-HMC} on the depth-$5$ {\tt deep sea swirl} environment with different prior variances $\sigma^2$. The evolution of the posterior over the equivalence classes of policies (Equation \ref{eq:ts_mdp}) with episode number is shown for different tolerance values $\epsilon$, after averaging over $10$ independent runs. The colour shows the cumulative reward achieved by each class. Shaded areas show one standard deviation computed with the $10$ independent runs.}
\label{fig:deepseaswirl_policy_cvg}
\end{figure}

\begin{figure}[h!]
    \centering
    \BeginAccSupp{ActualText={Two sets (top and bottom) of cumulative regret curves over 100 episodes in a depth-5 deep sea swirl environment, comparing BR against HMC, TD, TD-En and TD-Max for each of the tolerances 0.01, 0.1, 0.2, 0.5, 1 and 5, where the top set is for prior variance 100 and the bottom set is for prior variance 1.}}
    \begin{subfigure}{1.0\textwidth}
    \centering
    \includegraphics[width=1\linewidth]{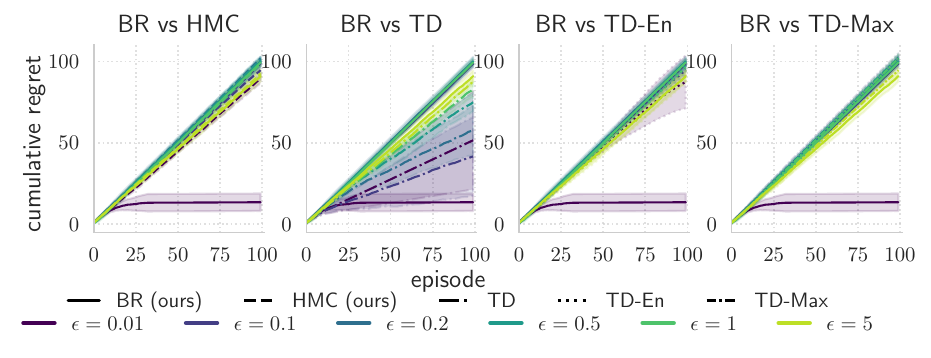}
    \caption{$\sigma^2=10^2$.}
    \label{fig:deepseaswirl_s10_cumreg}
    \end{subfigure}
    \begin{subfigure}{1.0\textwidth}
    \centering
    \includegraphics[width=1\linewidth]{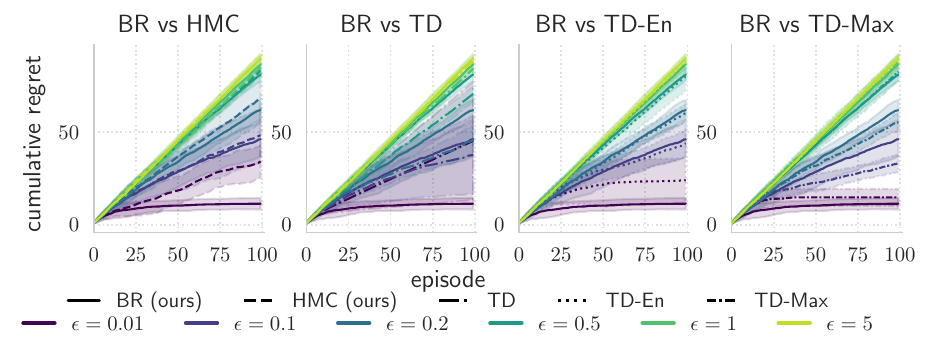}
    \caption{ $\sigma^2=1$}
    \label{fig:deepseaswirl_s1_cumreg}
    \end{subfigure}
    \EndAccSupp{}
\caption{The progress of the regret as it accumulates with each new episode, for the depth-$5$ \texttt{deep sea swirl} problem, for different tolerance values $\epsilon$ and different prior variances $\sigma^2$. Results are averaged over $10$ independent runs. Shaded areas indicate one standard deviation computed with the $10$ independent runs.}
\label{fig:deepseaswirl_cumreg}
\end{figure}

\begin{figure}[h!]
    \centering
    \BeginAccSupp{ActualText={Two sets (top and bottom) of subplots (left and right) for the depth-5 deep sea swirl environment over 100 episodes: the left plots track the exploration progress of BR, HMC, TD, TD-En and TD-Max for tolerances 0.01, 0.1, 0.2, 0.5, 1 and 5; the right plots show the corresponding HMC acceptance probabilities for each tolerance level. The top set is for prior variance 100 and the bottom set is for prior variance 1.}}
    \begin{subfigure}[b]{0.7\textwidth}
        \centering
        \includegraphics[width=\linewidth]{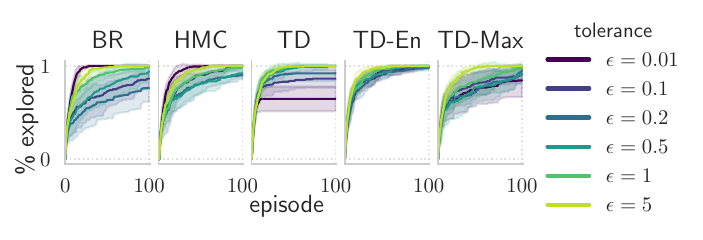}
        \caption{Percentage of state-action pairs visited at least once ($\sigma^2=10^2$)}
        \label{fig:deepseaswirl_s10_expl_prog}
    \end{subfigure}
    \hfill
    \begin{subfigure}[b]{0.29\textwidth}
        \centering
        \includegraphics[width=\linewidth]{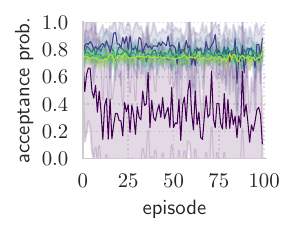}
        \caption{HMC acceptance probability ($\sigma^2=10^2$)}
        \label{fig:deepseaswirl_s10_hmc_acc}
    \end{subfigure}

    \vspace{1.5em}
    
    \begin{subfigure}[b]{0.7\textwidth}
        \centering
        \includegraphics[width=\linewidth]{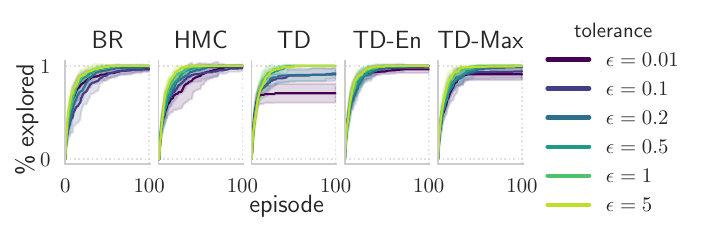}
        \caption{Percentage of state-action pairs visited at least once ($\sigma^2=1$)}
        \label{fig:deepseaswirl_s1_expl_prog}
    \end{subfigure}
    \hfill
    \begin{subfigure}[b]{0.29\textwidth}
        \centering
        \includegraphics[width=\linewidth]{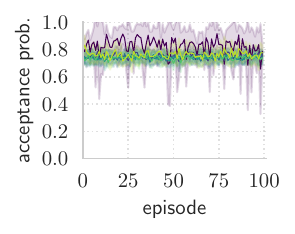}
        \caption{HMC acceptance probability ($\sigma^2=1$)}
        \label{fig:deepseaswirl_s1_hmc_acc}
    \end{subfigure}
    \EndAccSupp{}

    \caption{Results for the depth-$5$ \texttt{deep sea swirl} environment with different prior variances $\sigma^2$. All results are averaged over $10$ independent runs. Shaded areas show one standard deviation computed using these runs.}
    \label{fig:deepseaswirl_cumreg_and_acc}
    
\end{figure}

\begin{figure}[h!]
    \centering
    \BeginAccSupp{ActualText={Line graphs over 100 episodes comparing BR and HMC for each of the tolerances 0.01, 0.1, 0.2, 0.5, 1 and 5 in the depth-5 deep sea pyramid environment. The plots show the convergence of the posterior probability for the optimal policy as well as the corresponding decay of suboptimal policies.}}
    \includegraphics[width=1\linewidth]{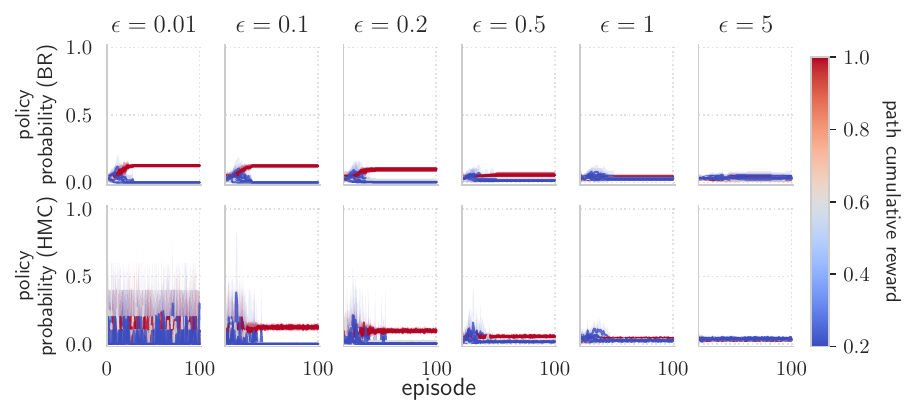}
    \EndAccSupp{}
    \caption{A comparison between \textit{Bayes-BR} and \textit{Bayes-BR-HMC} on the depth-$5$ {\tt deep sea pyramid} environment. The evolution of the posterior over the equivalence classes of policies (Equation \ref{eq:ts_mdp}) with episode number is shown for different tolerance values $\epsilon$, after averaging over $10$ independent runs. The colour shows the cumulative reward achieved by each class. Shaded areas show one standard deviation computed with the $10$ independent runs.}
    \label{fig:deepseapyramid_policy_cvg}
\end{figure}

\begin{figure}[h!]
    \centering
    \BeginAccSupp{ActualText={Cumulative regret curves over 100 episodes in a depth-5 deep sea pyramid environment, comparing BR against HMC, TD, TD-En and TD-Max for each of the tolerances 0.01, 0.1, 0.2, 0.5, 1 and 5.}}
    \includegraphics[width=1\linewidth]{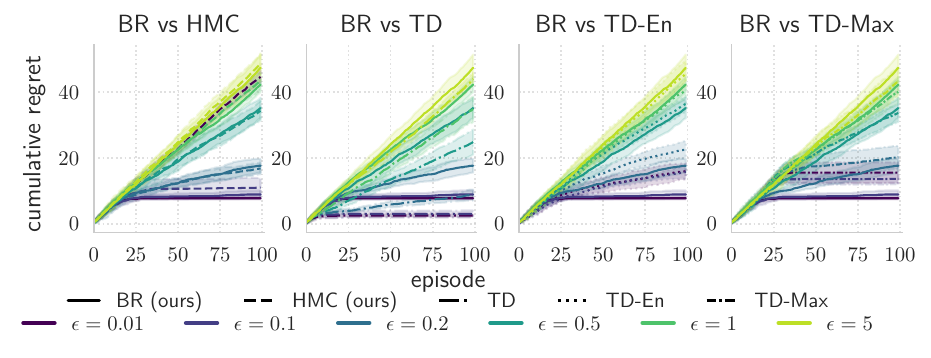}
    \EndAccSupp{}
    \caption{The progress of the regret as it accumulates with each new episode, for the depth-$5$ \texttt{deep sea pyramid} problem, for different tolerance values $\epsilon$. Results are averaged over $10$ independent runs. Shaded areas indicate one standard deviation computed with the $10$ independent runs.}
    \label{fig:deepseapyramid_cumreg}
\end{figure}

\begin{figure}[h!]
    \centering
    \BeginAccSupp{ActualText={Two subplots for the depth-5 deep sea pyramid environment over 100 episodes: the left tracks the exploration progress of BR, HMC, TD, TD-En and TD-Max for tolerances 0.01, 0.1, 0.2, 0.5, 1 and 5; the right plots the corresponding HMC acceptance probabilities for each tolerance level.}}
    \begin{subfigure}[b]{0.7\textwidth}
        \centering
        \includegraphics[width=\linewidth]{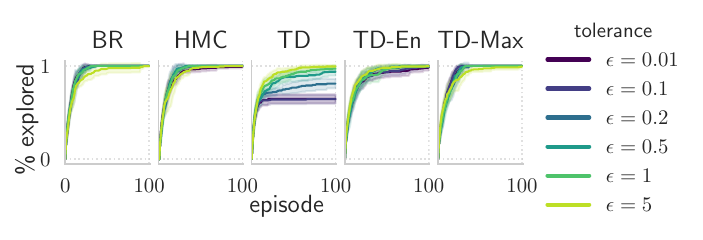}
        \caption{Percentage of the environment's state-action pairs visited at least once.}
        \label{fig:deepseapyramid_expl_prog}
    \end{subfigure}
    \hfill
    \begin{subfigure}[b]{0.29\textwidth}
        \centering
        \includegraphics[width=\linewidth]{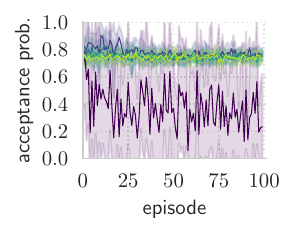}
        \caption{HMC acceptance probability}
        \label{fig:deepseapyramid_hmc_acc}
    \end{subfigure}
    \EndAccSupp{}
    \caption{Results for the depth-$5$ \texttt{deep sea pyramid} environment. All results are averaged over $10$ independent runs. Shaded areas show one standard deviation computed using these runs.}
    \label{fig:deepseapyramid_cumreg_and_acc}
\end{figure}

\begin{figure}[h!]
    \centering
    \BeginAccSupp{ActualText={Line graphs comparing the posterior probability of choosing each of the eight optimal classes of policies at episode 100 between BR, HMC, TD, TD-En and TD-Max for the deep sea pyramid environment, plotted for each of the tolerances 0.01, 0.1, 0.2, 0.5, 1 and 5.}}
    \includegraphics[width=1\linewidth]{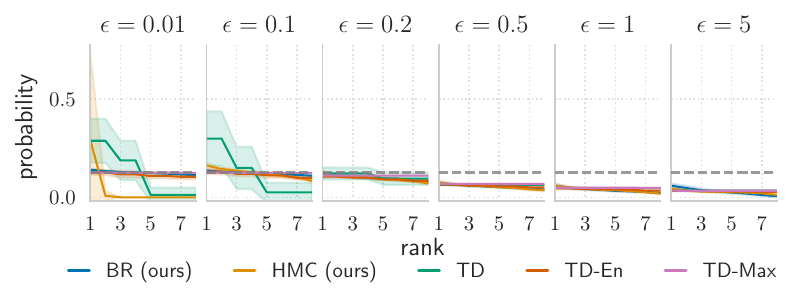}
    \EndAccSupp{}
    \caption{Ranked posterior probabilities of the eight optimal classes of policies for the depth-$5$ \texttt{deep sea pyramid} problem at episode $100$. The dashed line indicates a uniform distribution over the optimal policy classes, which is preferred. Results are averaged over $10$ independent runs. To account for permutation invariance, the probabilities from each run are sorted in descending order before averaging. Shaded areas indicate one standard deviation computed with the $10$ independent runs.}
    \label{fig:deepseapyramid_rank}
\end{figure}

\begin{figure}[h!]
    \centering
    \BeginAccSupp{ActualText={Cumulative regret curves over 100 episodes in a depth-4 stochastic deep sea environment, comparing BR against HMC, TD, TD-En and TD-Max for each of the tolerances 0.01, 0.1, 0.2, 0.5, 1 and 5.}}
    \includegraphics[width=1\linewidth]{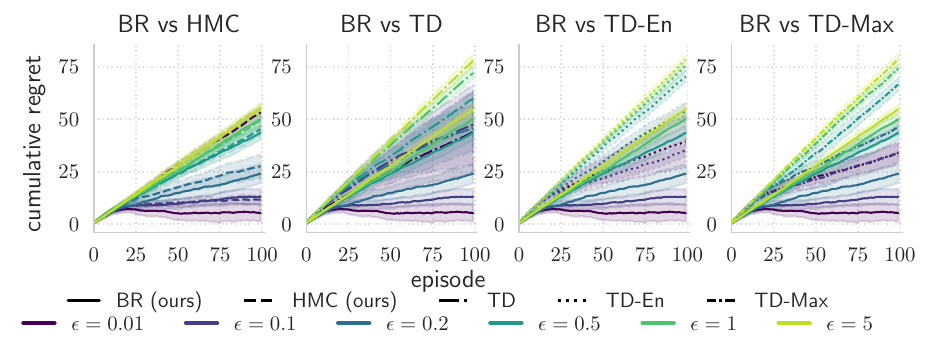}
    \EndAccSupp{}
    \caption{The progress of the regret as it accumulates with each new episode, for the depth-$4$ \texttt{deep sea} problem with stochastic transitions ($0.9$ probability of transitioning according to the selected action at all but the bottom row of states) and deterministic rewards, for different tolerance values $\epsilon$. Results are averaged over $10$ independent runs. Shaded areas indicate one standard deviation computed with the $10$ independent runs.}
    \label{fig:deepsea_sto_cumreg}
\end{figure}